\def\year{2020}\relax
\documentclass[letterpaper]{article} % DO NOT CHANGE THIS
\usepackage{aaai20}  % DO NOT CHANGE THIS
\usepackage{times}  % DO NOT CHANGE THIS
\usepackage{helvet} % DO NOT CHANGE THIS
\usepackage{courier}  % DO NOT CHANGE THIS
\usepackage[hyphens]{url}  % DO NOT CHANGE THIS
\usepackage{graphicx} % DO NOT CHANGE THIS
\usepackage{graphicx}  % DO NOT CHANGE THIS
\usepackage{amsmath,amsfonts,bm}

\def\eqref#1{equation~\ref{#1}}
\def\1{\bm{1}}

\def\rmC{{\mathbf{C}}}

\def\rmI{{\mathbf{I}}}

\def\rmM{{\mathbf{M}}}
\def\rmN{{\mathbf{N}}}

\def\rmR{{\mathbf{R}}}

\DeclareMathAlphabet{\mathsfit}{\encodingdefault}{\sfdefault}{m}{sl}
\SetMathAlphabet{\mathsfit}{bold}{\encodingdefault}{\sfdefault}{bx}{n}

\DeclareMathOperator*{\argmax}{arg\,max}

\usepackage[utf8]{inputenc} % allow utf-8 input
\usepackage[T1]{fontenc}    % use 8-bit T1 fonts
\usepackage{url}            % simple URL typesetting
\usepackage{booktabs}       % professional-quality tables
\usepackage{amsfonts}       % blackboard math symbols
\usepackage{nicefrac}       % compact symbols for 1/2, etc.
\usepackage{microtype}      % microtypography

\usepackage{amsmath}
\usepackage{amssymb}
\usepackage{amsthm}
\usepackage{amscd}
\usepackage{comment}
\usepackage{microtype}
\usepackage{graphicx}
\usepackage{multirow}
\usepackage{subcaption}
\usepackage{float}
\usepackage{array}
\usepackage{booktabs} % for professional tables
\usepackage{xcolor}
\usepackage{color, colortbl}
\usepackage{enumitem}
\newcommand\crule[3][black]{\textcolor{#1}{\rule{#2}{#3}}}
\setlist[itemize]{leftmargin=*}

\newtheorem{theorem}{Theorem}
\newtheorem{corollary}{Corollary}

\newtheorem{lemma}{Lemma}

\newtheorem{definition}{Definition}
\usepackage{algorithm}
\usepackage{algorithmic}
\usepackage{fancyhdr}

\newif\ifsubmit

\submittrue

\ifsubmit
\newcommand{\bo}[1]{}
\newcommand{\yang}[1]{}
\newcommand{\jingkang}[1]{}
\else
\newcommand{\bo}[1]{{\color{blue}[Bo: #1]}}
\newcommand{\yang}[1]{{\color{magenta}[Yang: #1]}}
\newcommand{\jingkang}[1]{{\color{brown}[Jingkang: #1]}}
\fi

\title{Reinforcement Learning with Perturbed Rewards}
\author{
Jingkang Wang \\
University of Toronto \& Vector Institute \\
Toronto, Canada \\
\texttt{wangjk@cs.toronto.edu}
\And
Yang Liu \\
University of California, Santa Cruz \\
California, USA \\
\texttt{yangliu@ucsc.edu}
\And
Bo Li \\
University of Illinois, Urbana–Champaign \\
Illinois, USA \\
\texttt{lbo@illinois.edu}
}

\begin{document}

\maketitle

\begin{abstract}
%Reinforcement learning (RL) has been widely applied in various applications, such as autonomous driving. However, 
% Recent studies have shown the vulnerability of reinforcement learning (RL) models in noisy environments. The sources of noises differ across scenarios.
Recent studies have shown that reinforcement learning (RL) models are vulnerable in various noisy scenarios.
% , where noises from different sources appear.
For instance, the observed reward channel is often subject to noise in practice (e.g., when rewards are collected through sensors), and is therefore not credible. In addition, for applications such as robotics, a deep reinforcement learning (DRL) algorithm can be manipulated to produce arbitrary errors by receiving corrupted rewards.
%
%For instance, a deep reinforcement learning (DRL) algorithm can be manipulated to produce arbitrary errors by adding carefully crafted adversarial perturbations (e.g., to each frame of a game). 
%
%In applications such as robotics, a more common adversarial attack is to physically alter the rewards. % Such attacks raise severe security concerns when deploying DRL in real world. 
%Allowing adversaries to arbitrarily manipulate pixel values is less practical, and it is much easier for adversaries to perform physical attacks by altering the rewards. 
% Allowing adversaries to arbitrarily manipulate pixel values, or assuming arbitrary noise modelmakes this noisy learning problem extremely challenging. 
In this paper, we consider noisy RL problems with \textit{perturbed rewards}, which can be approximated with a confusion matrix.
% generated with a reward confusion matrix. 
% We call the observed rewards as \emph{perturbed rewards}.
%in which rewards can be subject to either inherent noise and/or noise due to adversarial manipulation. 
%
%, which is more practical given the uncertainty in real world. The perturbation in rewards can be due to either adversarial manipulations, or uncertain noises, or both.
 % It is also easy for adversaries to perform physical attacks by altering a small set of final behaviors (reward).
We develop
%an unbiased reward estimator aided 
a robust RL framework that enables agents to learn in noisy environments where only perturbed rewards are observed. Our solution framework builds on existing RL/DRL algorithms and firstly addresses the \textbf{biased} noisy reward setting without any assumptions on the true distribution
% assuming the distribution of the noise
(e.g., zero-mean Gaussian noise as made in previous works).
% Our framework draws upon approaches for supervised learning with noisy data. 
The core ideas of our solution include estimating a reward confusion matrix and defining a set of unbiased surrogate rewards. We prove the convergence and sample complexity of our approach. 
Extensive experiments on different DRL platforms show that trained policies based on our estimated surrogate reward can achieve higher expected rewards, and converge faster than existing baselines. For instance, the state-of-the-art PPO algorithm is able to obtain 84.6\% and 80.8\% improvements on \emph{average score} for five Atari games, with error rates as 10\% and 30\% respectively. %\bo{do we have 1-2 nice number to demonstrate the strength of proposed method to add here?} \jingkang{Yes, we can get some numbers (such as relative improvement) from Table 2 and 3. But I am not sure which one to choose or how to weigh the numbers.}
% For reproducibility, the code is publicly accessible: \url{https://tinyurl.com/y4jeveg7}.
\end{abstract}

\section{Introduction}

%\yang{note to myself: start with a more general RL setting, but mention DRL. mention that this is a specific attach model, sounds restrictive? need to discuss with Bo; }

%Reinforcement learning (RL) has recently witnessed tremendous success in various applications, in part due to developments in deep reinforcement learning (DRL) techniques~\cite{silver2016mastering,mnih2015human}. 
%
Designing a suitable reward function plays a critical role in building reinforcement learning models for real-world applications. 
Ideally, one would want to customize reward functions to achieve application-specific goals~\cite{hadfield2017inverse}.
In practice, however, it is difficult to design a reward function that produces credible rewards in the presence of noise. 
This is because the output from any reward function is subject to multiple kinds of randomness:
\begin{itemize}
    \item \textit{Inherent Noise}. For instance, sensors on a robot will be affected by physical conditions such as temperature and lighting, and therefore will report back noisy observed rewards.
    \item \textit{Application-Specific Noise}. In machine teaching tasks~\cite{Loftin2014}, when an RL agent receives feedback/instructions, different human instructors might provide drastically different feedback that leads to biased rewards for machine.
    % due to their personal styles and capabilities. 
    % This way the RL agent (machine) will obtain reward with bias.
    \item \textit{Adversarial Noise}. 
    % Adversarial perturbation has been widely explored in different learning tasks and shows strong attack power against different machine learning models. For instance, 
    ~\citeauthor{huang2017adversarial} have shown that by adding adversarial perturbation to each frame of the game, they can mislead pre-trained RL policies arbitrarily. % to perform arbitrarily bad. 
\end{itemize}

Assuming an arbitrary noise model makes solving this noisy RL problem extremely challenging. Instead, we focus on a specific noisy reward model which we call \emph{perturbed rewards}, where the observed rewards by RL agents are learnable. The perturbed rewards are generated via a confusion matrix that flips the true reward to another one according to a certain distribution. This is not a very restrictive setting~\cite{DBLP:conf/ijcai/EverittKOL17} to start with, even considering that the noise could be adversarial: 
% Given that arbitrary pixel value manipulation attack in RL is not very practical, adversaries in the real-world have high incentives to inject adversarial perturbation to the reward value by slightly modifying it. 
For instance, adversaries can manipulate sensors via reversing the reward value. 
% \jingkang{a bit confusing}
%This way, we can hope to leverage statistical estimation methods to reason and learn against such perturbations. 

In this paper, we develop an unbiased reward estimator aided robust framework that enables an RL agent to learn in a noisy environment with observing only perturbed rewards. 
%Our solution framework builds on existing reinforcement learning algorithms.
%, including the recently developed DRL ones ($Q$-Learning~\cite{Watkins:1989,Watkins92q-learning}, Cross-Entropy Method (CEM)~\cite{cem}, Deep SARSA~\cite{sarsa}, Deep $Q$-Network (DQN)~\cite{dqn1,mnih2015human,double-dqn}, Dueling DQN (DDQN)~\cite{dueling-dqn}, Deep Deterministic Policy Gradient (DDPG)~\cite{ddpg}, Continuous DQN (NAF)~\cite{naf} and Proximal Policy Optimization (PPO)~\cite{ppo} algorithms).
The main challenge is that the observed rewards are likely to be biased, and in RL or DRL the accumulated errors could amplify the reward estimation error over time. To the best of our knowledge, this is the first work addressing robust RL in the biased rewards setting (existing work need to assume the unbiased noise distribution). We do not require any assumption on the knowledge of true reward distribution or adversarial strategies, other than the fact that the generation of noises follows a reward confusion matrix. 
% We provide a general solution framework when this confusion matrix is known. For the case it is unknown, 
We address the issue of estimating the reward confusion matrices by proposing an efficient and flexible estimation module for settings with deterministic rewards. %We'd like to note that in practice this perturbation noise model can include both random uncertainty and adversarial perturbation injected in the reward values.

\citeauthor{DBLP:conf/ijcai/EverittKOL17} provided preliminary studies for this noisy reward problem and gave some general negative results. The authors proved a \emph{No Free Lunch} theorem, which is, without any assumption about what the reward corruption is, all agents can be misled.
%, and Bayesian RL agents may still suffer near-maximal regret despite strongly simplifying assumptions. 
Our results do not contradict with the results therein, as we consider a stochastic noise generation model (that leads to a set of perturbed rewards).

We analyze the convergence and sample complexity for the policy trained using our proposed method based on surrogate rewards, using $Q$-Learning as an example. We then conduct extensive experiments on OpenAI Gym~\cite{gym} 
% (AirRaid, Alien, Carnival, MsPacman, Pong, Phoenix, Seaquest) 
and show that the proposed reward robust RL method achieves comparable performance with the policy trained using the true rewards. In some cases, our method even achieves higher cumulative reward - this is surprising to us at first, but we conjecture that the inserted noise together with our noise-removal unbiased estimator add another layer of exploration, which proves to be beneficial in some settings. % This merits a future study. 

% Real-world reward is imperfect and can be noisy or corrupted due to various reasons, just like human-annotated labels in supervised learning. In practice, we need RL algorithms to optimize the policy from observed noisy rewards. This problem is hard because the observed rewards may not be stochastic (zero-mean) and the agents are easily misled if the corrupted rewards are not corrected. Besides, in RL background, the accumulated error in sequence or iteration is another critical issue to be considered. \cite{DBLP:conf/ijcai/EverittKOL17} generally study this problem and also give some general negative results. They prove the No Free Lunch theorem, which is, without some assumption about what the reward corruption can look like, all agents are essentially lost. Furthermore, Bayesian RL agents may still suffer near-maximal regret despite strongly simplifying assumptions.

% \begin{enumerate}
% \item Sensor error: robots short-circuits the reward channel with certain probability.
% \item Reward misspecification: agents mis-claim a false reward. 
% \item Adversarial intervention: adversarial agents intentionally intervene with noises.
% \item \jingkang{story: \url{https://www.ijcai.org/proceedings/2017/0656.pdf}}
% \end{enumerate}

Our contributions are summarized as follows: (1) %We propose a reward robust reinforcement learning framework that enables a RL agent to learn in a noisy environment with observing only perturbed rewards. 
We formulate and generalize the idea of defining a simple but effective unbiased estimator for true rewards % using observed and perturbed rewards 
under reinforcement learning setting. The proposed estimator helps guarantee the convergence to the optimal policy even when the RL agents only have noisy observations of the rewards. %We extend the method of estimation to generalized reward models, and provide an efficient algorithm for estimating the reward confusion matrices %\bo{this is not clear} in reinforcement learning. 
(2) We analyze the convergence to the optimal policy and the finite sample complexity of our reward-robust RL methods, using $Q$-Learning as the example. (3) Extensive experiments on OpenAI Gym show that our proposed algorithms perform robustly even at high noise rates. Code is online available: \url{https://github.com/wangjksjtu/rl-perturbed-reward}.%with the proposed surrogate rewards. %In some circumstances, the models converge even faster or obtain better performance compared to those with true rewards.
%\end{itemize}

\subsection{Related Work}
\paragraph{Robust Reinforcement Learning} 
% RL have made impressive progresses in many challenging tasks, from games~\cite{mnih2015human,silver2017mastering} to robot control~\cite{DBLP:journals/ijrr/KoberBP13,DBLP:conf/icml/MnihBMGLHSK16,DBLP:conf/iros/Chatzilygeroudis17}. However, 
It is known that RL algorithms are vulnerable in noisy environments~\cite{rlblogpost}. Recent studies ~\cite{huang2017adversarial,DBLP:journals/corr/KosS17,DBLP:conf/ijcai/LinHLSLS17} show that learned RL policies can be easily misled with small perturbations in observations. The presence of noise is very common in real-world environments, especially in robotics-relevant applications~\cite{deisenroth2011learning,Loftin2014}. Consequently, robust RL algorithms have been widely studied, aiming to train a robust policy that is capable of withstanding perturbed observations~\cite{DBLP:conf/nips/TehBCQKHHP17,DBLP:conf/icml/PintoDSG17,gu2018adversary} or transferring to unseen environments~\cite{DBLP:journals/corr/RajeswaranGLR16,DBLP:journals/corr/abs-1710-11248}. However, these algorithms mainly focus on noisy vision observations, instead of observed rewards. Some early works~\cite{pendrith1997estimator,moreno2006noisy,DBLP:conf/icml/Strens00,Romoff2018} on noisy reward RL rely on the knowledge of \underline{unbiased} noise distribution, which limits their applicability to more general \underline{biased} rewards settings. A couple of recent works \cite{DBLP:journals/mor/LimXM16,DBLP:journals/corr/RoyXP17} have looked into a parallel question of training robust RL algorithms with uncertainty in models.

% \cite{DBLP:journals/neco/MorimotoD05} firstly propose a paradigm that explicitly takes into account input disturbance and modeling errors. 

% \begin{comment}
% To enhance the performance of $Q$-Learning in noisy environments, \cite{DBLP:conf/uai/FoxPT16} design $G$-Learning that introduces penalty term at the beginning of the learning process. \cite{DBLP:conf/nips/TehBCQKHHP17} introduce multi-task mechanism and capture common behaviors across different tasks using a ``distilled" policy. As for RARL, adversary training mechanism~\cite{DBLP:conf/icml/PintoDSG17,gu2018adversary} and detection algorithms~\cite{DBLP:journals/corr/abs-1710-00814} are studied to improve the robustness of policies against adversarial examples. Another branch of works focuses on the transferability in RL. For instance, \cite{DBLP:journals/corr/RajeswaranGLR16} leverage model ensembles to benefit the learning as well as the robustness of policies;~\cite{DBLP:journals/corr/abs-1710-11248} propose adversarial inverse reinforcement learning (AIRL) algorithm for learning transferrable reward functions.
% \end{comment}

% \yang{please shorten}
%~\cite{DBLP:conf/icml/PintoDSG17,gu2018adversary,DBLP:journals/corr/abs-1710-00814}.

% Robustness and sample complexity are major challenges when learning policies with reinforcement learning for real-world tasks

\paragraph{Learning with Noisy Data} Learning appropriately with biased data has received quite a bit of attention in recent machine learning studies \cite{natarajan2013learning,scott2013classification,scott2015rate,sukhbaatar2014learning,van2015learning,menon2015learning}. The idea of this line of works is to define unbiased surrogate loss functions to recover the true loss using the knowledge of the noise. Our work is the first to formally establish this extension both theoretically and empirically. Our quantitative understandings will provide practical insights when implementing reinforcement learning algorithms in noisy environments.

\section{Problem Formulation and Preliminaries}

In this section, we define our problem of learning from perturbed rewards in reinforcement learning. Throughout this paper, we will use \emph{perturbed reward} and \emph{noisy reward} interchangeably, considering that the noise could come from both intentional perturbation and natural randomness.
% each time step of our sequential decision making setting is similar to the ``learning with noisy data" setting in supervised learning~\cite{natarajan2013learning,scott2013classification,scott2015rate,sukhbaatar2014learning}. 
In what follows, we formulate our Markov Decision Process (MDP) and reinforcement learning (RL) problem with perturbed rewards.

\subsection{Reinforcement Learning: The Noise-Free Setting}

Our RL agent interacts with an unknown environment and attempts to maximize the total of its collected reward. The environment is formalized as a Markov Decision Process (MDP), denoting as $\mathcal{M} = \langle \mathcal{S}, \mathcal{A}, \mathcal{R}, \mathcal{P}, \gamma\rangle$. At each time $t$, the agent in state $s_t \in \mathcal{S}$ takes an action $a_t \in \mathcal{A}$, which returns a reward $r(s_t,a_t,s_{t+1}) \in \mathcal R$ (which we will also shorthand as $r_t$) \footnote{We do not restrict the reward to deterministic in general, except for when we need to estimate the noises in the perturbed reward (Section 3.3).}, and leads to the next state $s_{t+1} \in \mathcal{S}$ according to a transition probability kernel $\mathcal{P}$. $\mathcal{P}$ encodes the probability $\mathbb{P}_a(s_{t}, s_{t+1})$, and commonly is unknown to the agent. The agent's goal is to learn the optimal policy, a conditional distribution $\pi(a | s)$ that  maximizes the state's value function. The value function calculates the cumulative reward the agent is expected to receive given it would follow the current policy $\pi$ after observing the current state $s_t$:
$
V^\pi(s) = \mathbb{E}_\pi \left[ \sum_{k=0}^\infty \gamma^k r_{t+k+1} \mid s_t = s   \right], 
$
where $0\leq \gamma\leq 1$ is a discount factor ($\gamma = 1$ indicates an undiscounted MDP setting~\cite{DBLP:conf/icml/Schwartz93,DBLP:journals/ior/Sobel94,Kakade2003OnTS}). Intuitively, the agent evaluates how preferable each state is, given the current policy. From the Bellman Equation, the optimal value function is given by
$
V^\ast(s) = \max_{a \in \mathcal{A}} \sum_{s_{t+1} \in \mathcal{S}} \mathbb P_a(s_t, s_{t+1})\left[ r_t + \gamma V^\ast(s_{t+1}) \right].
$
It is a standard practice for RL algorithms to learn a state-action value function, also called the $Q$-function. $Q$-function denotes the expected cumulative reward if agent chooses $a$ in the current state and follows $\pi$ thereafter:
$
Q^\pi(s,a) =\mathbb{E}_\pi\left[ r(s_t,a_t,s_{t+1}) + \gamma V^\pi(s_{t+1}) \mid s_t = s, a_t = a \right].
$
%\yang{you haven't defined $r(s_t,a_t,s_{t+1})$. pleaee add.}

\subsection{Perturbed Reward in RL}
%The corrupt reward problem in MDP is formulated in~\cite{DBLP:conf/ijcai/EverittKOL17}, which is different from our setup. 
In many practical settings, the RL agent does not observe the reward feedback perfectly. 
We consider the following MDP with perturbed reward, denoting as $\mathcal{\tilde{M}} = \langle \mathcal{S}, \mathcal{A}, \mathcal{R}, C, \mathcal{P}, \gamma\rangle$\footnote{The MDP with perturbed reward can equivalently be defined as a tuple $\mathcal{\tilde{M}} = \langle \mathcal{S}, \mathcal{A}, \mathcal{R}, \mathcal{\tilde{R}}, \mathcal{P}, \gamma\rangle$, with the perturbation function $C$ implicitly defined as the difference
between $\mathcal{R}$ and $\mathcal{\tilde{R}}$.}: instead of observing $r_t \in \mathcal{R}$ at each time $t$ directly (following his action), our RL agent only observes a perturbed version of $r_t$, denoting as $\tilde{r}_t \in \mathcal{\tilde{R}}$. For most of our presentations, we focus on the cases where $\mathcal{R}$, $\mathcal{\tilde{R}}$ are finite sets; but our results generalize to the continuous reward settings with discretization techinques. %and clustering techniques. \yang{Jingkang, can you double check? I remember we talked about clustering but not sure this was implemented by you or not. } % of true rewards and observed potentially perturbed rewards. Different from standard MDP, agent only observes a perturbed version of $r$, denoting as $\tilde{r}$. %Different from~\cite{DBLP:conf/ijcai/EverittKOL17}, 

The generation of $\tilde{r}$ follows a certain function $C: \mathcal{S} \times \mathcal{R} \rightarrow \mathcal{\tilde{R}}$. To let our presentation stay focused, we consider the following state-independent
% \footnote{The case of state-dependent perturbed reward is discussed in Appendix~\ref{sec:state-dependent}}
flipping error rates model: if the rewards are binary (consider $r_+$ and $r_-$), $\tilde{r}(s_t,a_t,s_{t+1})$ ($\tilde{r}_t$) can be characterized by the following noise rate parameters $e_+,e_-$: 
$
e_+ = \mathbb{P}(\tilde{r}(s_t,a_t,s_{t+1}) = r_-| r(s_t,a_t,s_{t+1}) = r_+),
e_- = \mathbb P(\tilde{r}(s_t,a_t,s_{t+1}) = r_+| r(s_t,a_t,s_{t+1}) = r_-)
$
. When the signal levels are beyond binary, suppose there are $M$ outcomes in total, denoting as $[R_0, R_1, \cdots, R_{M - 1}]$. $\tilde{r}_t$ will be generated according to the following confusion matrix $\rmC_{M\times M}$
% \begin{small}
% \[ \rmC_{M\times M} = 
% \begin{bmatrix}
%     c_{0,0}       & c_{0,1} &  \dots & c_{0,M-1} \\
% %    c_{1,0}        & c_{1,1}  & \dots & c_{1, M-1} \\
%     \cdots & \cdots & \cdots & \cdots \\
%     c_{M-1,0}       & c_{M-1,1} & \dots & c_{M-1,M-1}
% \end{bmatrix},
% \]
% \end{small}
where each entry $c_{j,k}$ indicates the flipping probability for generating a perturbed outcome:
$
c_{j,k} = \mathbb P(\tilde{r}_t = R_k|r_t=R_j).
$
Again we'd like to note that we focus on settings with finite reward levels for most of our paper, but we provide discussions later 
%in Section~\ref{sec:method} 
on how to handle continuous rewards. % with discretizations. 

In the paper, we also generalize our solution to the case without knowing the noise rates (i.e., the reward confusion matrices) for settings in which the rewards for each (state, action) pair is deterministic, which is different from the assumption of knowing them as adopted in many supervised learning works \cite{natarajan2013learning}. Instead we will estimate the confusion matrices in our framework. % (Section~\ref{sec:estimate_c}). 
%The estimated reward is called surrogate reward. 
% We let $\rmR$, $\tilde{\rmR}$ and $\hat{\rmR}$ denote true rewards, observed noisy rewards and surrogate rewards estimated by proposed proxy, respectively; $r$, $\tilde{r}$ and $\hat{r}$ denote true reward, noisy reward and surrogate reward. 

\section{Learning with Perturbed Rewards}
In this section, we first introduce an unbiased estimator for binary rewards in our reinforcement learning setting when the error rates are known. This idea is inspired by~\cite{natarajan2013learning}, but we will extend the method to the multi-outcome, as well as the continuous reward settings. % (Section~\ref{sec:method}). %We call such an unbiased estimated reward a ``surrogate reward". Our RL algorithm is simply a meta-algorithm that builds on existing RL algorithms, but with replacing the rewards using surrogate rewards. We will provide theoretical analysis of convergence and sample complexity for our surrogate reward based RL, taking $Q$-Learning as an example (Section~\ref{sec:theoretical_analysis}). Finally, we propose an efficient algorithm to address the estimation of confusion matrices (Section~\ref{sec:estimate_c}).

\subsection{Unbiased Estimator for True Reward}
\label{sec:method}
With the knowledge of noise rates (reward confusion matrices), we are able to establish an unbiased approximation of the true reward in a similar way as done in \cite{natarajan2013learning}. We will call such a constructed unbiased reward as a \emph{surrogate reward}. To give an intuition, we start with replicating the results for binary reward $\mathcal R = \{r_{-},r_{+}\}$ in our RL setting:
% \footnote{In the proof (Appendix), we cover the general binary reward level case.} 
\begin{lemma}
\label{lemma:1}
Let $r$ be bounded. Then, if we define,
\begin{align}
\begin{array}{l}
\label{eq:binary}
\hat{r}(s_t,a_t,s_{t+1}) := 
\begin{cases}
\frac{(1-e_{-})\cdot r_{+}-e_{+}\cdot r_{-}}{1-e_+-e_-} & (\tilde{r}(s_t,a_t,s_{t+1}) = r_+) \\
\frac{(1-e_{+})\cdot r_{-}-e_{-}\cdot r_{+}}{1-e_+-e_-} & (\tilde{r}(s_t,a_t,s_{t+1}) = r_-)
\end{cases}
\end{array}
\end{align}
we have for any $r(s_t,a_t,s_{t+1})$,
$\mathbb E_{\tilde{r}|r}[\hat{r}(s_t,a_t,s_{t+1})] = r(s_t,a_t,s_{t+1}).$
\end{lemma}

% \begin{lemma}
% \label{lemma:1}
% Let $r$ be bounded. Then, if we define,
% \begin{small}
% \begin{align}
% \label{eq:binary}
% \hat{r}(s_t,a_t,s_{t+1}) := \frac{(1-e_{-\tilde{r}})\tilde{r}(s_t,a_t,s_{t+1})-e_{\tilde{r}}\cdot (-\tilde{r}(s_t,a_t,s_{t+1}))}{1-e_+-e_-}
% \end{align}
% \end{small}
% we have for any $r(s_t,a_t,s_{t+1})$,
% $\mathbb E_{\tilde{r}|r}[\hat{r}(s_t,a_t,s_{t+1})] = r(s_t,a_t,s_{t+1}).$
% \end{lemma}

%\begin{align}
%\hat{r}(s_t,a_t,s_{t+1}) := %\frac{(1-e_{-\tilde{r}})\tilde{r}-e_{\tilde{r}}\cdot %(-\tilde{r})}{1-e_+-e_-}
%\end{align}
%The above gives an unbiased estimator:

In the standard supervised learning setting, the above property guarantees convergence - as more training data are collected, the empirical surrogate risk converges to its expectation, which is the same as the expectation of the true risk (due to unbiased estimators). This is also the intuition why we would like to replace the reward terms with surrogate rewards in our RL algorithms.

% I'm interested in how this idea extends to the sequential learning setting, i.e., reinforcement learning. Instead of using $\tilde{r}$ or $r$ (because we don't have access to the true reward), replace with the surrogate reward $\hat{r}$ in any existing RL algorithms. My intuition is that is works too - by induction, we can show that the unbiasedness property holds for each $Q$-function at each step $t$. The convergence, nonetheless, merits more studies. 

%\subsubsection{Extension: Multi-outcome and continuous reward settings}
%\label{sec:extension}
The above idea can be generalized to the multi-outcome setting in a fairly straight-forward way. Define $\hat{\rmR}:= [\hat{r}(\tilde{r} = R_0), \hat{r}(\tilde{r} = R_{1}), ..., \hat{r}(\tilde{r} = R_{M-1})]$, where $\hat{r}(\tilde{r} = R_{k})$ denotes the value of the surrogate reward when the observed reward is $R_{k}$. Let $\rmR= \left[R_0; R_1;\cdots; R_{M-1}\right]$ be the bounded reward matrix with $M$ values. We have the following results:
% \yang{i changed $\hat{R}$}
\begin{lemma}
\label{lemma:2}
Suppose $\rmC_{M \times M}$ is invertible. With defining:
\begin{align}
\hat{\rmR} = \rmC^{-1} \cdot \rmR \label{eq:multivariate}
\end{align}
we have for any $r(s_t,a_t,s_{t+1})$,
$
\mathbb E_{\tilde{r}|r}[\hat{r}(s_t,a_t,s_{t+1})] = r(s_t,a_t,s_{t+1}).
$
\end{lemma}

% \begin{comment}
% The core challenge of this extension is to find surrogate unbiased rewards $\hat{r}$. Writing out the conditions for unbiasedness (s.t. $\mathbb E_{\hat{r}|r}[c \circ \hat{r}] = r.
% $), we need to solve the following set of functions to obtain $\hat{r}$: % $\varphi(\cdot)$:
% $$
% \begin{cases}
% \begin{aligned}
% r_0 &= c_{0,0} \cdot \hat{r}_0 + c_{0,1}  \cdot \hat{r}_1 + \dots +c_{0,M-1}  \cdot \hat{r}_{M-1}\\
% r_1 &= c_{1,0}  \cdot \hat{r}_0 + c_{1,1}  \cdot \hat{r}_1+ \dots +c_{1,M-1}  \cdot \hat{r}_{M-1}\\
% & \cdots \\
% r_{M-1} &= c_{M-1,0}  \cdot \hat{r}_0 + c_{M-1,1}  \cdot \hat{r}_1 + \dots + c_{M-1,M-1}  \cdot \hat{r}_{M-1}
% \end{aligned}
% \end{cases}
% $$

% Denote by $\rmR= \left[r_0; r_1;\cdots; r_{M-1}\right]$, and $\hat{\rmR}= \left[\hat{r}_0; \hat{r}_1; \cdots; \hat{r}_{M-1}\right]$. Then the above equation becomes equivalent with the following system of equation:
% $\rmR  = \rmC \cdot \hat{\rmR}. 
% $ If $\rmC$ is full rank (invertible):
% \begin{align}
% \hat{\rmR} = \rmC^{-1} \cdot \rmR. \label{unbias:multi}
% \end{align}
% When $\rmC$ is non-invertible, but its psudo-inverse exists, then there exist infinitely many solution - but we only need one. It's possible that the above equation doesn't have a solution....
% \end{comment}

\paragraph{Continuous reward} When the reward signal is continuous, 
we discretize it into $M$ intervals, and view each interval as a reward level, with its value approximated by its middle point.
%For instance $r \in [\underline{b},\overline{b}]$, follow the following discretization procedure: The support of reward $[\underline{b},\overline{b}]$ is first discretized into $M$ intervals uniformly:
% \begin{small}
%$\left[\underline{b},\underline{b}+ (\overline{b}-\underline{b})/M\right], \cdots , \left[\underline{b}+ (M-1)(\overline{b}-\underline{b})/M, \overline{b}\right]$, and treat each discretized interval as a reward level, using either its left or right end point. 
With increasing $M$, this quantization error can be made arbitrarily small. Our method is then the same as the solution for the multi-outcome setting, except for replacing rewards with discretized ones. Note that the finer-degree quantization we take, the smaller the quantization error - but we would suffer from learning a bigger reward confusion matrix. This is a trade-off question that can be addressed empirically. 

So far we have assumed knowing the confusion matrices and haven't restricted our solution to any specific setting, but we will address this additional estimation issue 
% in Section \ref{sec:estimate_c} 
focusing on determinisitc reward settings, and present our complete algorithm therein. 

% To estimate $C$: simulate a reward $r$, and determine its quantized reward level; add continuous noise $n$ (now we have $\tilde{r}$) and then determine its quantized reward level; estimate $C$ based on above two quantized signals. 

%\yang{Mention that our method is the same as standard RL, but with surrogate reward. Give it a name. I suggest "surrogate RL".}

%\yang{Put our surrogate RL into a separate algorithmic block.}
\subsection{Convergence and Sample Complexity: $Q$-Learning}

\label{sec:theoretical_analysis}
We now analyze the convergence and sample complexity of our surrogate reward based RL algorithms (with assuming knowing $\rmC$), taking $Q$-Learning as an example. 

\paragraph{Convergence guarantee} First, the convergence guarantee is stated in the following theorem:
\begin{theorem}
\label{thm:convergence}
Given a finite MDP, denoting as $\mathcal{\hat{M}} = \langle \mathcal{S}, \mathcal{A}, \mathcal{\hat{R}}, \mathcal{P}, \gamma\rangle$, the $Q$-learning algorithm with surrogate rewards, given by the update rule,
\begin{small}
\begin{align}
Q_{t+1}(s_{t}, a_{t}) = (1 - \alpha_{t})Q(s_{t}, a_{t}) + \alpha_{t} \left[\hat{r}_{t} + \gamma \max_{b \in \mathcal{A}} Q(s_{t+1}, b) \right],
\label{eq:rule}
\end{align}
\end{small}
converges w.p.1 to the optimal $Q$-function as long as $\sum_{t}\alpha_{t} = \infty$ and $\sum_{t}\alpha^2_{t} < \infty$.
\end{theorem}

Note that the term on the right hand of Eqn.~(\ref{eq:rule}) includes surrogate reward $\hat{r}$ estimated using Eqn.~(\ref{eq:binary}) and Eqn.~(\ref{eq:multivariate}). Theorem~\ref{thm:convergence} states that agents will converge to the optimal policy \textit{w.p.1} when replacing the rewards with surrogate rewards, despite of the noises in the observed rewards. This result is not surprising - though the surrogate rewards introduce larger variance, we are grateful of their unbiasedness, which grants us the convergence. In other words, the addition of the perturbed reward does not affect the convergence guarantees of $Q$-Learning with surrogate rewards.% \yang{let's unify: call the rewards from our estimator as surrogate rewards, instead of proxy rewards?}

\paragraph{Sample complexity} To establish our sample complexity results, we first introduce a \textit{generative model} following previous literature~\cite{DBLP:conf/nips/KearnsS98a,DBLP:conf/colt/KearnsS00,DBLP:conf/ijcai/KearnsMN99}. This is a practical MDP setting to simplify the analysis. 
\begin{definition}
A generative model $G(\mathcal{M})$ for an MDP $\mathcal{M}$ is a sampling model which takes a state-action pair $(s_t, a_t)$ as input, and outputs the corresponding reward $r(s_t, a_t)$ and the next state $s_{t+1}$ randomly with the probability of $\mathbb P_{a}(s_t, s_{t+1})$, i.e., $s_{t+1} \sim \mathbb P(\cdot|s, a)$.
\end{definition}

Exact value iteration is impractical if the agents follow the generative models above exactly~\cite{Kakade2003OnTS}. Consequently, we introduce a \textit{phased Q-Learning} which is similar to the ones presented in \cite{Kakade2003OnTS,DBLP:conf/nips/KearnsS98a} for the convenience of proving our sample complexity results. We briefly outline \textit{phased Q-Learning} as follows - the complete description (Algorithm~\ref{alg:phased_qlearn}) can be found in Appendix~\ref{appendix:proofs}.

%　\yang{math symbols should go in math environment: e.g., $m$ samples instead of m samples; also $Q$-learning instead of Q learning} \jingkang{Thanks for reminding! Maybe I forgot at some places.}
\begin{definition}
Phased Q-Learning algorithm takes $m$ samples per phase by calling generative model $G(\mathcal{M})$. It uses the collected $m$ samples to estimate the transition probability $\mathcal{P}$ and then update the estimated value function per phase. Calling generative model $G(\hat{\mathcal{M}})$ means that surrogate rewards $\hat{r}$ are returned and used to update the value function.
\end{definition}

% \begin{lemma}
% \label{thm:lower_bound}
% (Lower Bound) Let $\mathcal{F}$ be an algorithm that is given only access to a generative model for an MDP $\mathcal{M}$, and inputs $s$, $epsilon$, and $\delta$, Assume the output policy $\pi$ satisfies, with probability greater that 1 - $\delta$, $\left|V_{\pi}(s) - V^{\ast}(s)\right| \leq \epsilon$. There exists an MDP $\mathcal{M}$ and a state $s$, on which $\mathcal{F}$ must make $\Omega \left( \frac{|\mathcal{S}||\mathcal{A}|T}{\epsilon} \log\frac{1}{\delta}\right)$ calls to the generative model $G(\mathcal{M})$.
% \end{lemma}

The sample complexity of \textit{Phased $Q$-Learning} is given as follows:
\begin{theorem}
\label{thm:upper_bound}
(Upper Bound) Let $r \in [0, R_{\max}]$ be bounded reward, $\rmC$ be an invertible reward confusion matrix with $\mathrm{det}(\rmC)$ denoting its determinant. For an appropriate choice of $m$, the Phased $Q$-Learning algorithm calls the generative model $G(\mathcal{\hat{M}})$
$
O\left(\frac{|\mathcal{S}||\mathcal{A}|T}{\epsilon^2(1 - \gamma)^2\mathrm{det}(\rmC)^2}\log\frac{|\mathcal{S}||\mathcal{A}|T}{\delta}\right)
$
times in $T$ epochs, and returns a policy such that for all state $s \in \mathcal{S}$, 
$
\left|V_{\pi}(s) - V^{\ast}(s) \right| \leq \epsilon,~\epsilon>0,$ w.p. $\geq 1 - \delta,~0<\delta<1$.
\end{theorem}
% \yang{we didn't define $\mathrm{det}(\rmC)^2$. the matrix norm? If it's determinant, define determinant and use det(C).} \jingkang{agree! it is determinant.}
Theorem~\ref{thm:upper_bound} states that, to guarantee the convergence to the optimal policy, the number of samples needed is no more than $O(1 / \mathrm{det}(\rmC)^2)$ times of the one needed when the RL agent observes true rewards perfectly. This additional constant is the price we pay for the noise presented in our learning environment. When the noise level is high, we expect to see a much higher $1 / \mathrm{det}(\rmC)^2$; otherwise when we are in a low-noise regime , $Q$-Learning can be very efficient with surrogate reward~\cite{DBLP:conf/colt/KearnsS00}. Note that Theorem~\ref{thm:upper_bound} gives the upper bound in discounted MDP setting; for undiscounted setting ($\gamma = 1$), the upper bound is at the order of $O\left(\frac{|\mathcal{S}||\mathcal{A}|T^3}{\epsilon^2\mathrm{det}(\rmC)^2}\log\frac{|\mathcal{S}||\mathcal{A}|T}{\delta}\right)$.
% Lower bound result is omitted due to the lack of space. The idea of constructing MDP in which learning is difficult and the algorithm must make $O\left(\frac{|\mathcal{S}||\mathcal{A}|T}{\epsilon}\log\frac{1}{\delta}\right)$ calls to $G(\mathcal{\hat{M}})$, is similar to~\cite{Kakade2003OnTS}.
This result is not surprising, as the phased $Q$-Learning helps smooth out the noise in rewards in consecutive steps. We will experimentally test how the bias removal step performs without explicit phases.

% Compared with previous study~\cite{Kakade2003OnTS, DBLP:conf/colt/KearnsS00}, we could conclude that the number of samples needed to guarantee convergence \textit{w.r.t} surrogate reward is no more $O(1 / {|C|^2})$ (constant) times than the one needed with observing the true reward. In consequence, $Q$-Learning is very efficient with surrogate reward~\cite{DBLP:conf/colt/KearnsS00}. 

While the surrogate reward guarantees the unbiasedness, we sacrifice the variance at each of our learning steps, and this in turn delays the convergence (as also evidenced in the sample complexity bound). It can be verified that the variance of surrogate reward is bounded when $\rmC$ is invertible, and it is always higher than the variance of true reward. This is summarized in the following theorem:

\begin{theorem}
\label{thm:variance}
Let $r \in [0, R_{\max}]$ be bounded reward and confusion matrix $\rmC$ is invertible. Then, the variance of surrogate reward $\hat{r}$ is bounded as follows:
$
    \mathbf{Var}(r) \leq \mathbf{Var}(\hat{r}) \leq \frac{M^2}{\mathrm{det}(\rmC)^2} \cdot R_{\max}^2 .
$
\end{theorem}
To give an intuition of the bound, when we have binary reward, the variance for surrogate reward bounds as follows:
$
     \mathbf{Var}(r) \leq \mathbf{Var}(\hat{r}) \leq \frac{4 R_{\max}^2}{(1 - e_{+} - e_{-})^2} .
$
As $e_{-} + e_{+} \rightarrow 1$, the variance becomes unbounded and the proposed estimator is no longer effective, nor will it be well-defined.

\paragraph{Variance reduction} In practice, there is a trade-off question between bias and variance by tuning a linear combination of $\rmR$ and $\hat{\rmR}$, \textit{i.e.}, $\rmR_{proxy} = \eta \rmR + (1 - \eta) \hat{\rmR}$, via choosing an appropriate $\eta \in [0, 1]$. %\yang{$\alpha$ was used in q learning. changing to $\eta$} \jingkang{Agree, Thanks!}
Other variance reduction techniques in RL with noisy environment, for instance \cite{Romoff2018}, can be combined with our proposed bias removal technique too. We test them in the experiment section.

% \yang{add the closed-form for binary cases, so people will see it more intuitively. }

\subsection{Estimation of Confusion Matrices}

\label{sec:estimate_c}
%In Section \ref{sec:method} 
In previous solutions, we have assumed the knowledge of reward confusion matrices, in order to compute the surrogate reward. This knowledge is often not available in practice. Estimating these confusion matrices is challenging without knowing any ground truth reward information; but we would like to remark that efficient algorithms have been developed to estimate the confusion matrices in supervised learning settings~\cite{DBLP:conf/icassp/BekkerG16,sig15,DBLP:journals/corr/abs-1712-04577,DBLP:journals/corr/abs-1802-05300}. The idea in these algorithms is to dynamically refine the error rates based on aggregated rewards. Note this approach is not different from the inference methods in aggregating crowdsourcing labels, as referred in the literature~\cite{dawid1979maximum,karger2011iterative,liu2012variational}. We adapt this idea to our reinforcement learning setting, which is detailed as follows. 

The estimation procedure is only for the case with deterministic reward, but not for stochastic rewards. The reason is that we will use repeated observations to refine an estimated ground truth reward, which will be leveraged to estimate the confusion matrix. With uncertainty in the true reward, it is not possible to distinguish a clean case with true reward $\rmC \cdot \rmR$ from the perturbed reward case with true reward $\rmR$ and added noise by confusion matrix $\rmC$.

\begin{algorithm}[t]
\begin{small}
\caption{Reward Robust RL (sketch)}\label{alg:robust_qlearn_short}
\begin{algorithmic}[1]
\STATE {\bfseries Input:} $\mathcal{\tilde{M}}$, $\tilde{R}(s, a)$, $\eta$
\STATE {\bfseries Output:} $Q(s, a)$, $\pi (s)$
\STATE Initialize value function $Q(s, a)$ arbitrarily.
\WHILE{$Q$ is not converged}
\STATE Initialize state $s \in \mathcal{S}$, observed reward set $\tilde{R}(s, a)$ \STATE Set confusion matrix $\tilde{\rmC}$ as identity matrix $\rmI$
\WHILE{$s$ is not terminal}
\STATE Choose $a$ from $s$ using policy derived from $Q$ %\yang{chose a from s??s is the state no??}
\STATE Take action $a$, observe $s'$ and noisy reward $\tilde{r}$ 
\IF{collecting enough $\tilde{r}$ for all $\mathcal{S} \times \mathcal{A}$ pairs}
\STATE Get predicted true reward $\bar{r}$ using majority voting
\STATE Re-estimate $\tilde{\rmC}$ based on $\tilde{r}$ and $\bar{r}$ (using Eqn.~\ref{eq:estimate_c})
\ENDIF
\STATE Obtain surrogate reward $\dot{r}$ ($\hat{\rmR} = (1 - \eta) \cdot {\rmR} + \eta \cdot \tilde{\rmC}^{-1} {\rmR}$)
\STATE Update $Q$ using surrogate reward 
\STATE  $s \leftarrow s'$
\ENDWHILE
\ENDWHILE
\STATE  {\bfseries return} $Q(s, a)$ and $\pi(s)$
\end{algorithmic}
\end{small}
\end{algorithm}

At each training step, the RL agent collects the noisy reward and the current \textit{state-action} pair. Then, for each pair in $\mathcal{S} \times \mathcal{A}$, the agent predicts the true reward based on accumulated historical observations of reward for the corresponding \textit{state-action} pair via, e.g., averaging (majority voting). Finally, with the predicted true reward and the accuracy (error rate) for each state-action pair, the estimated reward confusion matrices $\tilde{\rmC}$ are given by % \yang{use different notation for estimated C; suggest $\tilde{c}_{i,j}$}
{\small
\begin{align}
    &\bar{r}(s, a) = \argmax_{R_i \in \mathcal{R}}~\#[\tilde{r}(s, a) = R_i], \label{eq:majority_voting} \\
    &\tilde{c}_{i,j} = \frac{\sum_{(s, a) \in \mathcal{S} \times \mathcal{A}} \#\left[\tilde{r}(s, a) = R_j | \bar{r}(s, a) = R_i\right]}{\sum_{(s, a) \in \mathcal{S} \times \mathcal{A}} \#[\bar{r}(s, a) = R_i]},
\label{eq:estimate_c}
\end{align}
}
where in above $\#\left[\cdot \right]$ denotes the number of state-action pair that satisfies the condition $[\cdot]$ in the set of observed rewards $\tilde{R}(s, a)$ (see Algorithm~\ref{alg:robust_qlearn_short} and~\ref{alg:robust_qlearn}); $\bar{r}(s, a)$ and $\tilde{r}(s, a)$ denote predicted true rewards (using majority voting) and observed rewards when the state-action pair is $(s, a)$. We break potential ties in Eqn. (\ref{eq:majority_voting}) equally likely. The above procedure of updating $\tilde{c}_{i,j}$ continues indefinitely as more observation arrives. 
Our final definition of surrogate reward 
% is nothing different from Eqn. (\ref{eq:multivariate}) but with replacing a known reward confusion matrix $\rmC$
replaces a known reward confusion $\rmC$ in Eqn.~(\ref{eq:multivariate}) with our estimated one $\tilde{\rmC}$. %in the generation of $\hat{r}$.
We denote this estimated surrogate reward as $\dot{r}$. %\yang{in algorithm 1, i think you are only using $\hat{r}$. it should really be $\dot{r}$?} \jingkang{agree!}

%On the other hand, when the confusion matrices are time-variant, the model can estimate the noisy rates based on the knowledge in temporal proximity and adjust the surrogate rewards dynamically. 

\begin{comment}
\begin{algorithm}[tb]
\caption{Reward Robust RL (sketch)}
\label{alg:robust_qlearn_short}
\begin{small}
\begin{algorithmic}[]
\STATE {\bfseries Input:} $\mathcal{\tilde{M}}$, $\alpha$, $\beta$, $\tilde{R}(s, a)$
\STATE {\bfseries Output:} $Q(s)$, $\pi (s, t)$
\STATE Initialize value function $Q(s, a)$ arbitrarily.
\WHILE{$Q$ is not converged}
\STATE Initialize state $s \in \mathcal{S}$
\WHILE{$s$ is not terminal}
\STATE Choose $a$ from $s$ using policy derived from $Q$ %\yang{chose a from s??s is the state no??}
\STATE Take action $a$, observe $s'$ and noisy reward $\tilde{r}$ 
\IF{collecting enough $\tilde{r}$ for every $\mathcal{S} \times \mathcal{A}$ pair}
\STATE Get predicted true reward $\bar{r}$ using majority voting
\STATE Estimate confusion matrix $\tilde{\rmC}$ based on $\tilde{r}$ and $\bar{r}$ (Eqn.~\ref{eq:estimate_c})
\ENDIF
\STATE Obtain surrogate reward $\dot{r}$ ($\hat{\rmR} = (1 - \eta) \cdot {\rmR} + \eta \cdot \rmC^{-1} {\rmR}$)
\STATE Update $Q$ using surrogate reward 
\STATE  $s \leftarrow s'$
\ENDWHILE
\ENDWHILE
\STATE  {\bfseries return} $Q(s)$ and $\pi(s)$
\end{algorithmic}
\end{small}
\end{algorithm}
\end{comment}

\begin{figure*}[t]

\makebox[7pt]{\raisebox{39pt}{\rotatebox[origin=c]{90}{\scriptsize{$\omega=0.3$}}}}%
\begin{subfigure}[b]{.19\textwidth}
    \centering
    \includegraphics[width=\textwidth]{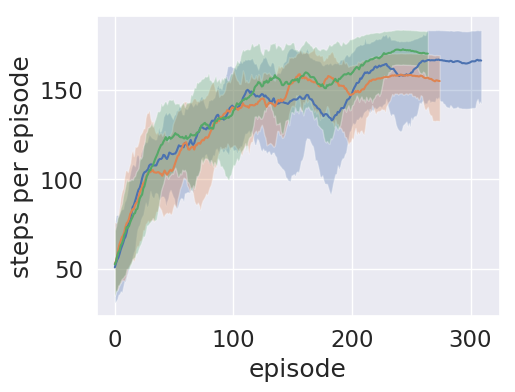}
\end{subfigure}
\begin{subfigure}[b]{.19\textwidth}
    \centering
    \includegraphics[width=\textwidth]{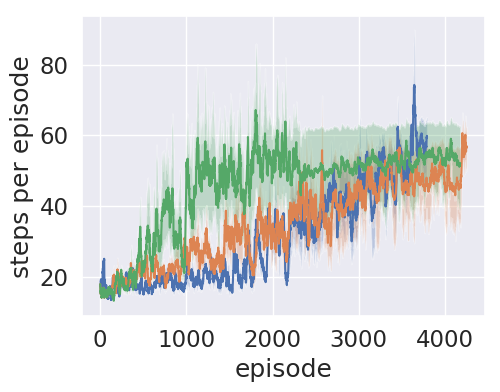}
\end{subfigure}
\begin{subfigure}[b]{.19\textwidth}
    \centering
    \includegraphics[width=\textwidth]{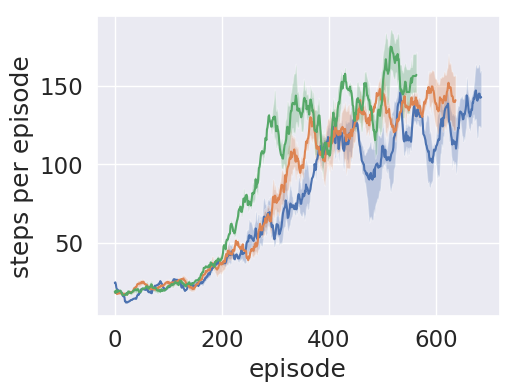}
\end{subfigure}
\begin{subfigure}[b]{.19\textwidth}
    \centering
    \includegraphics[width=\textwidth]{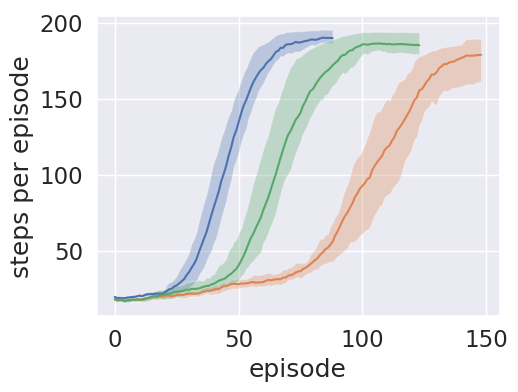}
\end{subfigure}
\begin{subfigure}[b]{.19\textwidth}
    \centering
    \includegraphics[width=\textwidth]{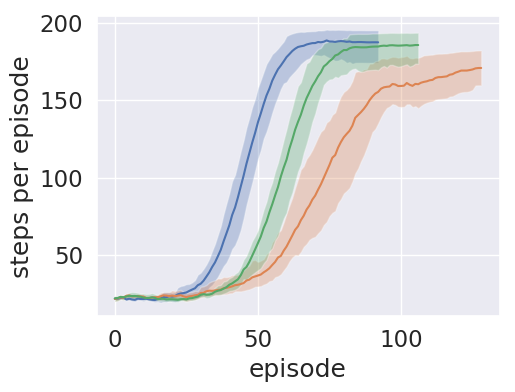}
\end{subfigure}

\makebox[7pt]{\raisebox{55pt}{\rotatebox[origin=c]{90}{\scriptsize{$\omega=0.7$}}}}%
\begin{subfigure}[b]{.19\textwidth}
    \centering
    \includegraphics[width=\textwidth]{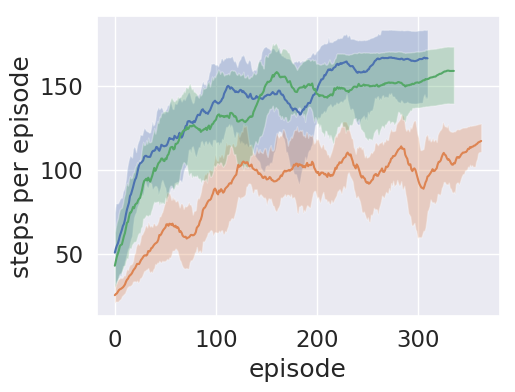}
    \caption{$Q$-Learning}
    \label{fig:cartpole_qlearn}
\end{subfigure}
\begin{subfigure}[b]{.19\textwidth}
    \centering
    \includegraphics[width=\textwidth]{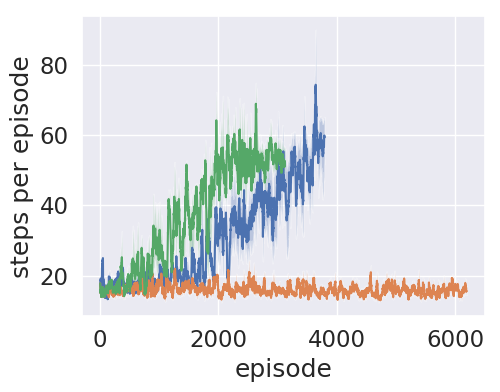}
    \caption{CEM}
    \label{fig:cartpole_cem}
\end{subfigure}
\begin{subfigure}[b]{.19\textwidth}
    \centering
    \includegraphics[width=\textwidth]{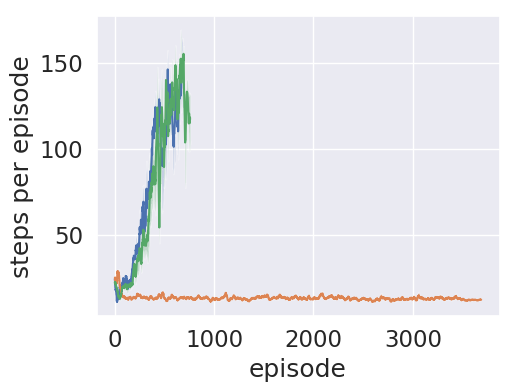}
    \caption{SARSA}
    \label{fig:cartpole_sarsa}
\end{subfigure}
\begin{subfigure}[b]{.19\textwidth}
    \centering
    \includegraphics[width=\textwidth]{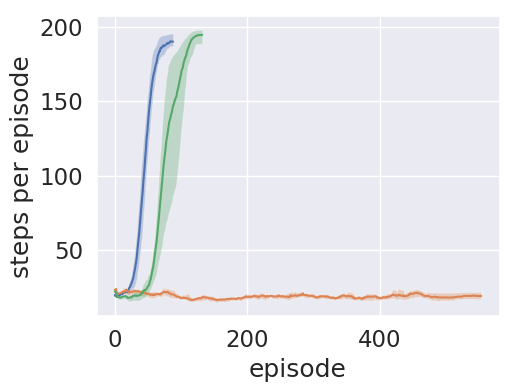}
    \caption{DQN}    
\end{subfigure}
\begin{subfigure}[b]{.19\textwidth}
    \centering
    \includegraphics[width=\textwidth]{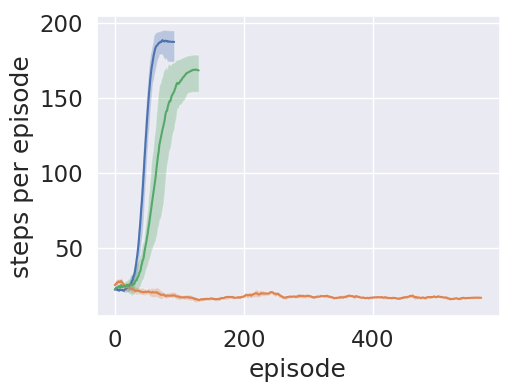}
    \caption{DDQN}
\end{subfigure}
\caption{Learning curves from five RL algorithms on CartPole game with true rewards ($r$)~\crule[blue]{0.30cm}{0.30cm}, noisy rewards ($\tilde{r}$)~\crule[orange]{0.30cm}{0.30cm} and estimated surrogate rewards ($\dot{r}$)~ ($\eta = 1$)~\crule[green]{0.30cm}{0.30cm}. Note that $\rmC$ are unknown to the agents and each experiment is repeated 10 times with different random seeds. We plotted 10\% to 90\% percentile area with its mean highlighted. Full results are in Appendix~\ref{appendix:control_figures} (Figure~\ref{fig:cartpole_est}).}
\label{fig:cartpole}
\end{figure*}

We present (\textit{Reward Robust RL}) in Algorithm~\ref{alg:robust_qlearn_short}\footnote{One complete $Q$-Learning implementation (Algorithm~\ref{alg:robust_qlearn}) is provided in Appendix~\ref{sec:robust}.}. Note that the algorithm is rather generic, and we can plug in any exisitng RL algorithm into our reward robust one, with only changes in replacing the rewards with our estimated surrogate rewards. %we only plug the estimation module in the training loop \yang{dont undertand this sentence}, so it can be applied in other RL algorithms following the same way (See experimental evaluation in Section~\ref{sec:exp_evaluation}). 

% Moreover, Algorithm~\ref{alg:robust_qlearn} takes into account bias-variance trade-off discussed in Section~\ref{sec:robust}.
% \bo{add more analysis to this section?}
% \yang{dont understand the last sentence. what is the bias variance tradeoff thing? what do we say?} \jingkang{remove it, because I put the discussion of variance in earlier section.}

% \yang{bring back a shorter version of Algorithm 2. }

\section{Experimental Results}

In this section, we conduct extensive experiments to evaluate the noisy reward robust RL mechanism with different games, under various noise settings. %In short, quantitative experiments conducted on a large variety of games (classic control \& Atari) on OpenAI Gym show that our surrogate RL method performs consistently well under various different RL algorithms, even at high noise rates. Furthermore, in some cases, surrogate RL even leads to a better or faster convergence compared with using true rewards.
Due to the space limit, more experimental results can be found in Appendix~\ref{appendix:exp_results}. %Finally, we point out the limitedness of this method in reinforcement learning, and some case studies are provided for a better comprehension. 

\subsection{Experimental Setup}

\paragraph{Environments and RL Algorithms}
To fully test the performance under different environments, we evaluate the proposed robust reward RL method on two classic control games (CartPole, Pendulum) and seven Atari 2600 games (AirRaid, Alien, Carnival, MsPacman, Pong, Phoenix, Seaquest), which encompass a large variety of environments, as well as rewards. Specifically, the rewards could be unary (CartPole), binary (most of Atari games), multivariate (Pong) and even continuous (Pendulum). A set of state-of-the-art RL algorithms are experimented with, while training under different amounts of noise (See Table~\ref{tab:alg})\footnote{The detailed settings are accessible in Appendix~\ref{appendix:exp_setup}.}. For each game and algorithm, unless otherwise stated, three policies are trained with different random initialization to decrease the variance.

% including pre-processing of observations, architectures of neural networks
% The experiments are set up in two popular framework, keras-rl~\cite{} and OpenAI baselines~\cite{}. 

% \paragraph{RL Algorithms.}
% A set of state-of-the-art reinforcement learning algorithms are experimented with while training under different amounts of noise, including $Q$-Learning~\cite{Watkins:1989,Watkins92q-learning}, Cross-Entropy Method (CEM)~\cite{cem}, Deep SARSA~\cite{sarsa}, Deep $Q$ Learning (DQN)~\cite{dqn1,mnih2015human,double-dqn}, Dueling DQN (DDQN)~\cite{dueling-dqn}, Deep Deterministic Policy Gradient (DDPG)~\cite{ddpg}, Continuous DQN (NAF)~\cite{naf} and Proximal Policy Optimization (PPO)~\cite{ppo} algorithms. For each game and algorithm, three policies are trained based on different random initialization to decrease the variance in experiments.

\begin{figure*}[t]
\centering
% \begin{subfigure}[b]{.24\textwidth}
%     \centering
%     \includegraphics[width=\textwidth]{pendulum2/Pendulum-v0-reward-0-1-DDPG-anti-iden.png}
% \end{subfigure}
% \begin{subfigure}[b]{.24\textwidth}
%     \centering
%     \includegraphics[width=\textwidth]{pendulum2/Pendulum-v0-reward-0-1-DDPG-norm-one.png}
% \end{subfigure}
% \begin{subfigure}[b]{.24\textwidth}
%     \centering
%     \includegraphics[width=\textwidth]{pendulum2/Pendulum-v0-reward-0-1-DDPG-norm-all.png}
% \end{subfigure}
% \begin{subfigure}[b]{.24\textwidth}
%     \centering
%     \includegraphics[width=\textwidth]{pendulum2/Pendulum-v0-reward-0-1-NAF-norm-all.png}
% \end{subfigure}
\makebox[7pt]{\raisebox{45pt}{\rotatebox[origin=c]{90}{\scriptsize{$\omega=0.3$}}}}%
\begin{subfigure}[b]{.20\textwidth}
    \centering
    \includegraphics[width=\textwidth]{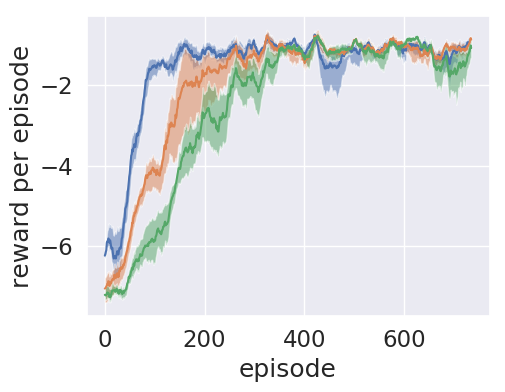}
    % \caption{DDPG (symmetric)}
    % \label{fig:pendulum_a}
\end{subfigure}
\begin{subfigure}[b]{.20\textwidth}
    \centering
    \includegraphics[width=\textwidth]{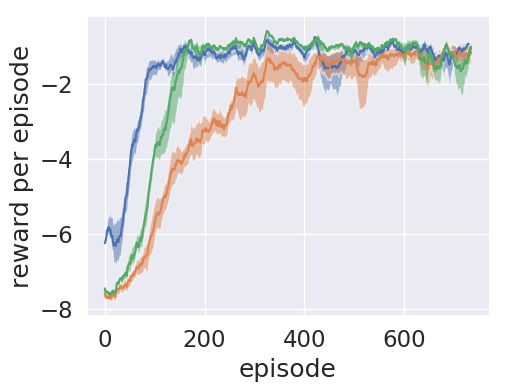}
    % \caption{DDPG (rand-one)}
    % \label{fig:pendulum_b}
\end{subfigure}
\begin{subfigure}[b]{.20\textwidth}
    \centering
    \includegraphics[width=\textwidth]{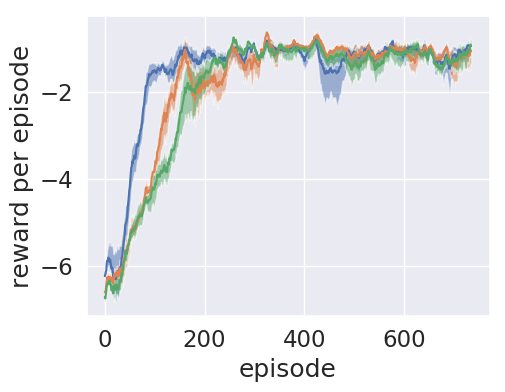}
    % \caption{DDPG (rand-all)}
    % \label{fig:pendulum_c}
\end{subfigure}
\begin{subfigure}[b]{.20\textwidth}
    \centering
    \includegraphics[width=\textwidth]{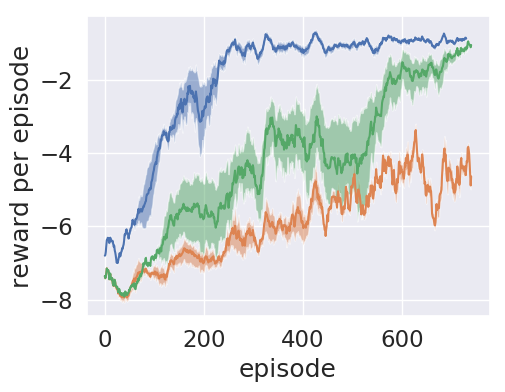}
    % \caption{NAF (rand-all)}
    % \label{fig:pendulum_d}
\end{subfigure}
% \begin{subfigure}[b]{.24\textwidth}
%     \centering
%     \includegraphics[width=\textwidth]{pendulum2/Pendulum-v0-reward-0-5-DDPG-anti-iden.png}
% \end{subfigure}
% \begin{subfigure}[b]{.24\textwidth}
%     \centering
%     \includegraphics[width=\textwidth]{pendulum2/Pendulum-v0-reward-0-5-DDPG-norm-one.png}
% \end{subfigure}
% \begin{subfigure}[b]{.24\textwidth}
%     \centering
%     \includegraphics[width=\textwidth]{pendulum2/Pendulum-v0-reward-0-5-DDPG-norm-all.png}
% \end{subfigure}
% \begin{subfigure}[b]{.24\textwidth}
%     \centering
%     \includegraphics[width=\textwidth]{pendulum2/Pendulum-v0-reward-0-5-NAF-norm-all.png}
% \end{subfigure}
\makebox[7pt]{\raisebox{58pt}{\rotatebox[origin=c]{90}{\scriptsize{$\omega=0.7$}}}}%
\begin{subfigure}[b]{.20\textwidth}
    \centering
    \includegraphics[width=\textwidth]{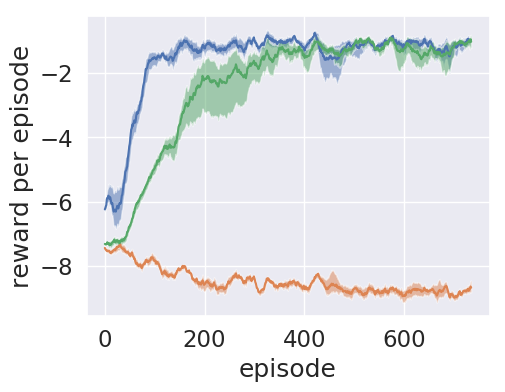}
    \caption{DDPG (symmetric)}
    \label{fig:pendulum_a}
\end{subfigure}
\begin{subfigure}[b]{.20\textwidth}
    \centering
    \includegraphics[width=\textwidth]{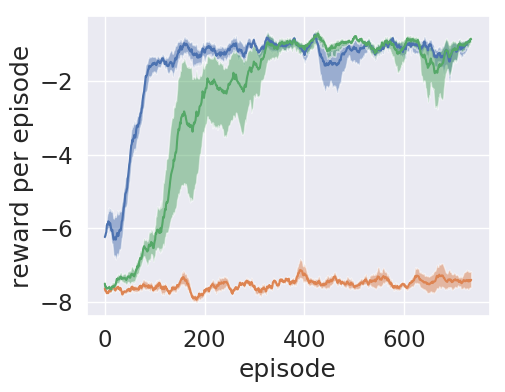}
    \caption{DDPG (rand-one)}
    \label{fig:pendulum_b}
\end{subfigure}
\begin{subfigure}[b]{.20\textwidth}
    \centering
    \includegraphics[width=\textwidth]{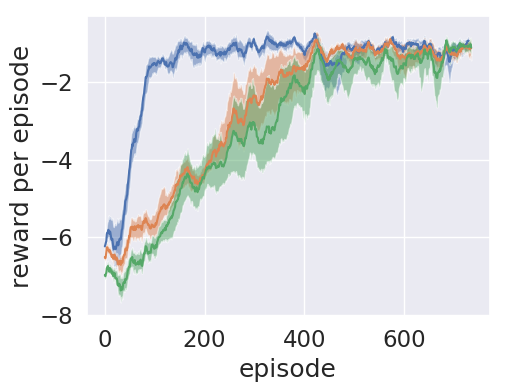}
    \caption{DDPG (rand-all)}
    \label{fig:pendulum_c}
\end{subfigure}
\begin{subfigure}[b]{.20\textwidth}
    \centering
    \includegraphics[width=\textwidth]{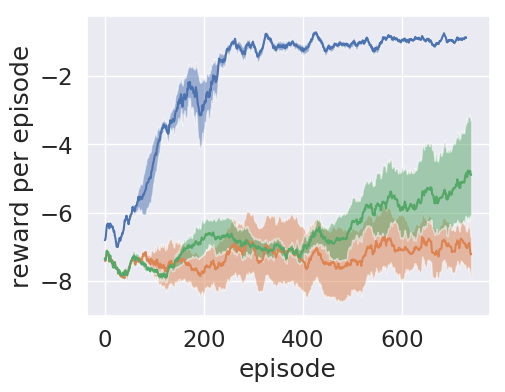}
    \caption{NAF (rand-all)}
    \label{fig:pendulum_d}
\end{subfigure}
\caption{Learning curves from DDPG and NAF on Pendulum game with true rewards ($r$)~\crule[blue]{0.30cm}{0.30cm}, noisy rewards ($\tilde{r}$)~\crule[orange]{0.30cm}{0.30cm} and surrogate rewards ($\hat{r}$)~($\eta = 1$)~\crule[green]{0.30cm}{0.30cm}. Both symmetric and asymmetric noise are conduced in the experiments and each experiment is repeated 3 times with different random seeds. Full results are in Appendix~\ref{appendix:control_figures} (Figure~\ref{fig:pendulum_full}).}
%　From top to the bottom, the noise rates are 0.3 and 0.7. 
\label{fig:pendulum}
\end{figure*}

\paragraph{Reward Post-Processing}
For each game and RL algorithm, we test the performance for learning with true rewards, noisy rewards and surrogate rewards. Both symmetric and asymmetric noise settings with different noise levels are tested. For symmetric noise, the confusion matrices are symmetric. As for asymmetric noise, two types of random noise are tested: 1) \textit{rand-one}, each reward level can only be perturbed into another reward; 2) \textit{rand-all}, each reward could be perturbed to any other reward, via adding a random noise matrix. To measure the amount of noise \textit{w.r.t} confusion matrices, we define the weight of noise $\omega$ in Appendix~\ref{sec:post-processing}. The larger $\omega$ is, the higher the noise rates are.

\subsection{Robustness Evaluation}

\label{sec:exp_evaluation}
% \yang{may need to reduce figures} \jingkang{Bo will give a fancy table to conclude. Others will be moved into appendix.}

\paragraph{CartPole} The goal in \textit{CartPole} is to prevent the pole from falling by controlling the cart's direction and velocity. The reward is $+1$ for every step taken, including the termination step. When the cart or pole deviates too much or the episode length is longer than 200, the episode terminates. Due to the unary reward $\{+1\}$ in CartPole, a corrupted reward $-1$ is added as the unexpected error ($e_{-} = 0$). As a result, the reward space $\mathcal{R}$ is extended to $\{+1, -1\}$. Five algorithms $Q$-Learning~\cite{Watkins92q-learning}, CEM~\cite{cem}, SARSA~\cite{sarsa}, DQN~\cite{double-dqn} and DDQN~\cite{dueling-dqn} are evaluated.

\definecolor{Gray}{gray}{0.92}
\begin{table}[tb]
\centering
\caption{Average scores of various RL algorithms on CartPole and Pendulum with noisy rewards ($\tilde{r}$) and surrogate rewards under known ($\hat{r}$) or estimated ($\dot{r}$) noise rates. Note that the results for last two algorithms DDPG (rand-one) $\&$ NAF (rand-all) are on Pendulum, but the others are on CartPole.}
% $Q$-Learn and DDQN in the table denote $Q$-Learning and Dueling DQN algorithms, respectively. 
%The experiments are repeated three times with different random seeds.
%$Q$-Learn and DDQN in the table denote $Q$-Learning and Dueling DQN algorithms, respectively.
\label{tab:scores-control}
\resizebox{0.48\textwidth}{!}{
\begin{tabular}{@{}c|c|ccccc|cc@{}}
\toprule
\ Noise Rate\   & Reward & $Q$-Learn & CEM & SARSA & DQN & DDQN & DDPG & \ NAF $\ $ \\ \midrule
\multicolumn{1}{c|}{\multirow{3}{*}{$\omega = 0.1$}} & $\tilde{r}$ & 170.0 & 98.1  & 165.2  & 187.2 & \textbf{187.8}  & -1.03  & -4.48          \\ %\cmidrule(l){2-8} 
\multicolumn{1}{c|}{}                     &  $\hat{r}$   & 165.8 & \textbf{108.9}  & \textbf{173.6} & \textbf{200.0} &  181.4 & \textbf{-0.87} & \textbf{-0.89}            \\ 
\multicolumn{1}{c|}{}                     &  \cellcolor{Gray}$\dot{r}$   & \cellcolor{Gray}\textbf{181.9} & \cellcolor{Gray}99.3  & \cellcolor{Gray}171.5 & \cellcolor{Gray}\textbf{200.0} &  \cellcolor{Gray}185.6 & \cellcolor{Gray}-0.90 & \cellcolor{Gray}-1.13           \\ \midrule
\multicolumn{1}{c|}{\multirow{3}{*}{$\omega = 0.3$}} & $\tilde{r}$ & 134.9 & 28.8 & 144.4 & 173.4 & 168.6  & -1.23 & -4.52  \\ %\cmidrule(l){2-8} 
\multicolumn{1}{c|}{}                     & $\hat{r}$ & 149.3 & \textbf{85.9} &  152.4  & 175.3 & 198.7 & \textbf{-1.03} & \textbf{-1.15}  \\ 
\multicolumn{1}{c|}{}                     &  \cellcolor{Gray}$\dot{r}$   & \cellcolor{Gray}\textbf{161.1} & \cellcolor{Gray}\textbf{82.2} & \cellcolor{Gray}\textbf{159.6} & \cellcolor{Gray}\textbf{186.7} & \cellcolor{Gray}\textbf{200.0} & \cellcolor{Gray}-1.05  & \cellcolor{Gray}-1.36              \\ \midrule
\multicolumn{1}{c|}{\multirow{3}{*}{$\omega = 0.7$}} & $\tilde{r}$ & 56.6 & 19.2 & 12.6 & 17.2 & 11.8  & -8.76 & -7.35     \\ %\cmidrule(l){2-8} 
\multicolumn{1}{c|}{}                     & $\hat{r}$ & \textbf{177.6} & \textbf{87.1} & 151.4  & 185.8 & \textbf{195.2} & \textbf{-1.09}  & \textbf{-2.26}  \\
\multicolumn{1}{c|}{}                     &  \cellcolor{Gray}$\dot{r}$  & \cellcolor{Gray}172.1 & \cellcolor{Gray}83.0  & \cellcolor{Gray}\textbf{174.4} & \cellcolor{Gray}\textbf{189.3} & \cellcolor{Gray}191.3 &  \cellcolor{Gray}--  & \cellcolor{Gray}--             \\
 \bottomrule
\end{tabular}
}
\end{table}

\begin{figure*}[t]
\centering
\includegraphics[width=0.85\textwidth, trim={0 0 0 6.3cm},clip]{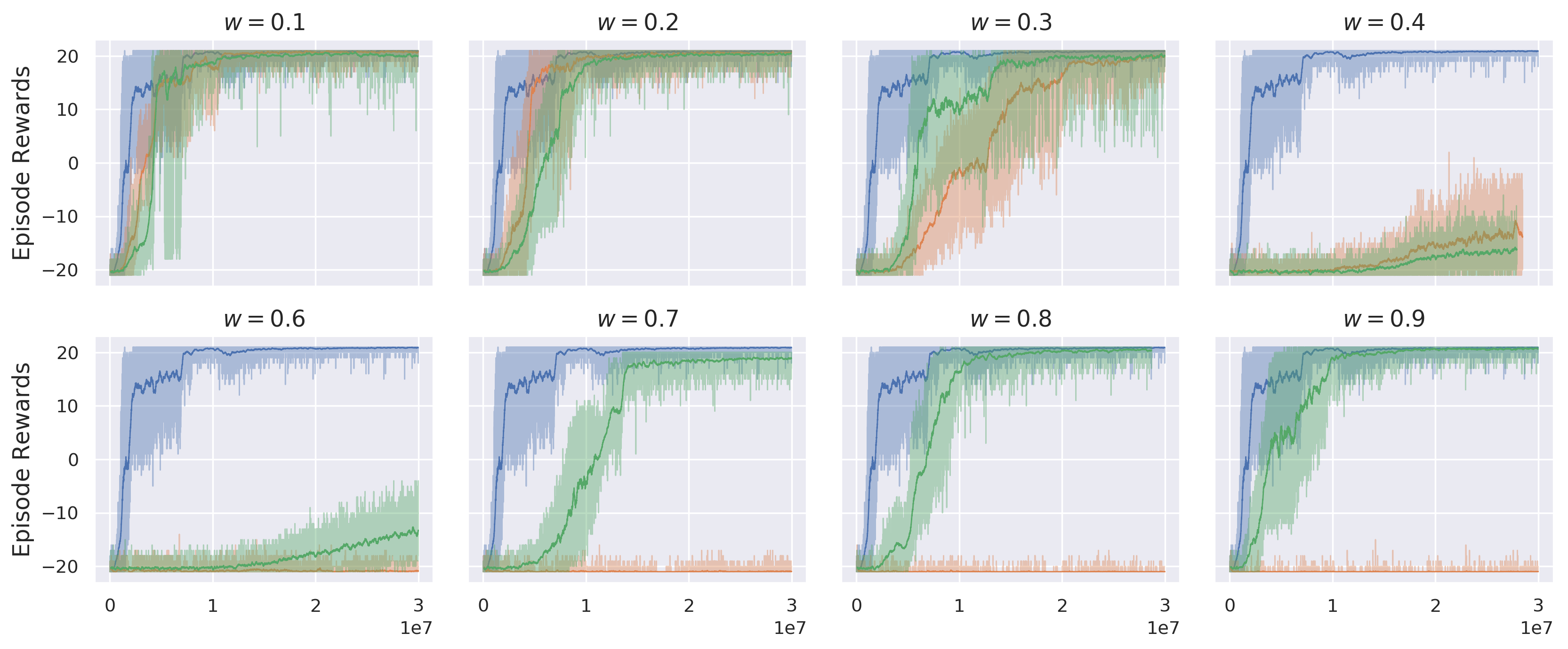}
\caption{Learning curves from PPO on Pong-v4 game with true rewards ($r$)~\crule[blue]{0.30cm}{0.30cm}, noisy rewards ($\tilde{r}$)~\crule[orange]{0.30cm}{0.30cm} and surrogate rewards~($\eta = 1$) ($\hat{r}$)~\crule[green]{0.30cm}{0.30cm}. The noise rate $\omega$ increases from 0.6 to 0.9, with a step of 0.1. Full results are in Appendix~\ref{appendix:atari_results} (Figure~\ref{fig:atari_full}).
}
\label{fig:pong}
% \small\textsuperscript{a=} The footnote-like comment under the caption
\end{figure*}

\begin{table}[t]
\centering
\caption{Average scores of PPO on five selected games with noisy rewards ($\tilde{r}$) and surrogate rewards under known ($\hat{r}$) or estimated ($\dot{r}$) noise rates.}
\label{tab:scores}
\resizebox{0.48\textwidth}{!}{	
\begin{tabular}{@{}cccccccc@{}}
\toprule
\ \ Noise Rate  & Reward & Lift ($\uparrow$) & Alien & Carnival & Phoenix & MsPacman & Seaquest \\ \midrule
\multicolumn{1}{c|}{\multirow{3}{*}{$\omega = 0.1$}} & $\tilde{r}$ & --  & 1835.1 & 1239.3 & 4609.0 & 1709.1 & 849.2 \\ 
\multicolumn{1}{c|}{} & $\hat{r}$ & \textbf{70.4\%$\uparrow$}  & 1737.0 & 3966.8 & \textbf{7586.4} & \textbf{2547.3} & 1610.6 \\ 
\multicolumn{1}{c|}{} & \cellcolor{Gray}$\dot{r}$ & \cellcolor{Gray}\textbf{84.6\%$\uparrow$} & \cellcolor{Gray}\textbf{2844.1} & \cellcolor{Gray}\textbf{5515.0} & \cellcolor{Gray}5668.8 & \cellcolor{Gray}2294.5 & \cellcolor{Gray}\textbf{2333.9} \\ \midrule
\multicolumn{1}{c|}{\multirow{3}{*}{$\omega = 0.3$}} & $\tilde{r}$ & -- & 538.2 & 919.9 & 2600.3 & 1109.6 & 408.7 \\
\multicolumn{1}{c|}{} & $\hat{r}$ & \textbf{119.8\%$\uparrow$}  & \textbf{1668.6} & \textbf{4220.1} & \textbf{4171.6} & 1470.3 & \textbf{727.8} \\ 
\multicolumn{1}{c|}{} & \cellcolor{Gray}$\dot{r}$ & \cellcolor{Gray}\textbf{80.8\%$\uparrow$}  & \cellcolor{Gray}1542.9 & \cellcolor{Gray}4094.3 & \cellcolor{Gray}2589.1 & \cellcolor{Gray}\textbf{1591.2} & \cellcolor{Gray}262.4 \\ \midrule
\multicolumn{1}{c|}{\multirow{3}{*}{$\omega = 0.7$}} & $\tilde{r}$ & -- & 495.2 & 380.3 & 126.5 & 491.6 & 0.0 \\
\multicolumn{1}{c|}{} & $\hat{r}$ & \textbf{757.4\%$\uparrow$}  & \textbf{1805.9} & 4088.9 & \textbf{4970.4} & 1447.8 & \textbf{492.5} \\
\multicolumn{1}{c|}{} & \cellcolor{Gray}$\dot{r}$ & \cellcolor{Gray}\textbf{648.9\%$\uparrow$} & \cellcolor{Gray}1618.0 & \cellcolor{Gray}\textbf{4529.2} & \cellcolor{Gray}2792.1 & \cellcolor{Gray}\textbf{1916.7} & \cellcolor{Gray}328.5 \\
% \multicolumn{1}{c|}{\multirow{3}{*}{$e_{-} = 0.9, e_{+} = 0.7$}} & $\tilde{r}$ & -- &  &  &  &  &  & \\
% \multicolumn{1}{c|}{} & $\hat{r}$ & \textbf{\%$\uparrow$}  &  &  &  &  &  &  \\ 
% \multicolumn{1}{c|}{} & $\dot{r}$ & \textbf{\%$\uparrow$}  &  &  &  &  &  &  \\ \midrule

\iffalse
\multicolumn{1}{c|}{\multirow{3}{*}{$\omega = 0.9$}} & $\tilde{r}$ & -- & 557.8 & 6.3 & 1410.9 & 535.4 & 588.8 \\ 
\multicolumn{1}{c|}{} & $\hat{r}$ & \textbf{508.7\%$\uparrow$}  & \textbf{1958.7} & \textbf{5664.2} & \textbf{6758.7} & \textbf{2515.1} & 1707.2 \\ 
\multicolumn{1}{c|}{} & \cellcolor{Gray}$\dot{r}$ & \cellcolor{Gray}\textbf{450.2\%$\uparrow$} & \cellcolor{Gray}1865.2 & \cellcolor{Gray}5515.0 & \cellcolor{Gray}5388.1 & \cellcolor{Gray}2492.6 & \cellcolor{Gray}\textbf{1788.6} \\ 
\fi
\bottomrule
\end{tabular}
}
\end{table}

\begin{figure*}[t]
\makebox[7pt]{\raisebox{55pt}{\rotatebox[origin=c]{90}{\scriptsize{$\omega=0.9$}}}}%
\begin{subfigure}[b]{.19\textwidth}
    \centering
    \includegraphics[width=\textwidth]{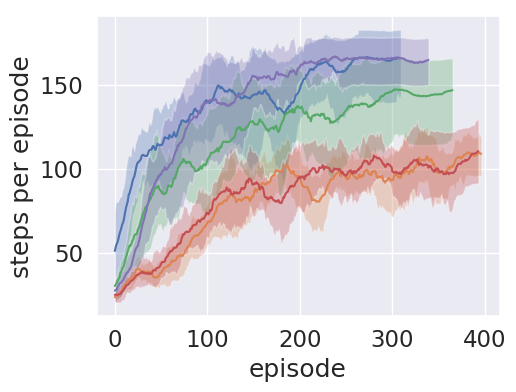}
    \caption{$Q$-Learning}
    \label{fig:cartpole_qlearn_var_part}
\end{subfigure}
\begin{subfigure}[b]{.19\textwidth}
    \centering
    \includegraphics[width=\textwidth]{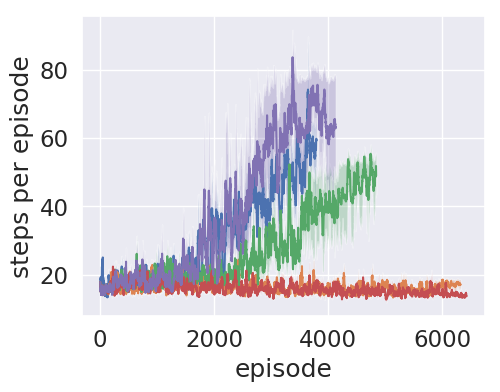}
    \caption{CEM}
    \label{fig:cartpole_cem_var_part}
\end{subfigure}
\begin{subfigure}[b]{.19\textwidth}
    \centering
    \includegraphics[width=\textwidth]{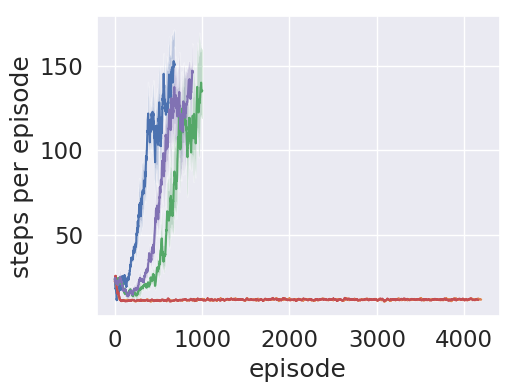}
    \caption{SARSA}
    \label{fig:cartpole_sarsa_var_part}
\end{subfigure}
\begin{subfigure}[b]{.19\textwidth}
    \centering
    \includegraphics[width=\textwidth]{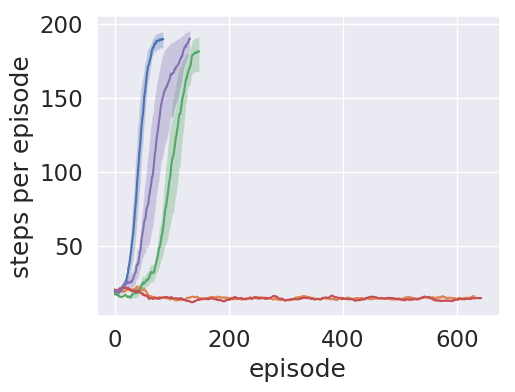}
    \caption{DQN}    
\end{subfigure}
\begin{subfigure}[b]{.19\textwidth}
    \centering
    \includegraphics[width=\textwidth]{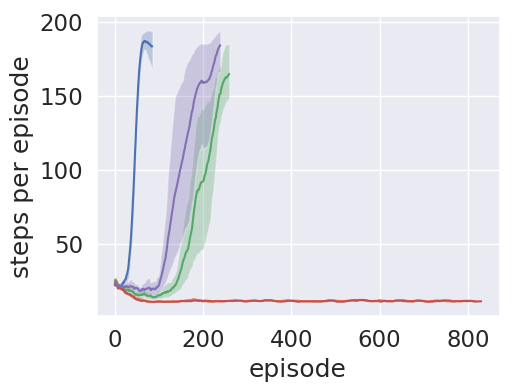}
    \caption{DDQN}
\end{subfigure}
\caption{Learning curves from five \textit{reward robust} RL algorithms on CartPole game with true rewards ($r$)~\crule[blue]{0.30cm}{0.30cm}, noisy rewards ($\tilde{r}$)~\crule[orange]{0.30cm}{0.30cm}, sample-mean noisy rewards~\crule[red]{0.30cm}{0.30cm}, estimated surrogate rewards ($\dot{r}$)~\crule[green]{0.30cm}{0.30cm} and sample-mean estimated surrogate rewards~\crule[blue!40!white]{0.30cm}{0.30cm}. Full results are in Appendix~\ref{appendix:exp_results} (Figure~\ref{fig:cartpole_var}). 
% From top to the bottom, the noise rates are 0.1, 0.3, 0.7 and 0.9. %Here we repeated each experiment 10 times with different random seeds and plotted 10\% to 90\% percentile area with its mean highlighted. 
% Full results are in Appendix~\ref{appendix:control_figures} (Figure~\ref{fig:cartpole_est}).
}
\label{fig:cartpole_var_part}
\end{figure*}

% In the experiments, we fix the training steps as $5e^4$, so the number of episodes is smaller when the models converge faster. \yang{i dont understand the last sentence.} \jingkang{If the models converge faster, they will get better preformance and persist more steps in every episode at early stage.} \jingkang{Maybe rewrite?} % The results ($\omega = 0.5$) are omitted because the confusion matrices are singular. Relevant discussions are given in Section~\ref{sec:limitedness}.

% ~\footnote{The training steps (rather than episodes) are fixed because the huge variance of some RL algorithms. It is clearer to judge the convergence speed if training steps are fixed.}

% \yang{might want to only keep figures for w=0.3 to save space.}

In Figure~\ref{fig:cartpole}, we show that our estimator successfully produces meaningful surrogate rewards that adapt the underlying RL algorithms to the noisy settings, without any assumption of the true distribution of rewards. With the noise rate increasing (from 0.1 to 0.9), the models with noisy rewards converge slower due to larger biases. However, we observe that the models (DQN and DDQN) always converge to the best score 200 with the help of surrogate rewards.

In some circumstances (slight noise - see Figure~\ref{fig:cartpole_cem},~\ref{fig:cartpole_sarsa}), the surrogate rewards even lead to faster convergence. This points out an interesting observation: learning with surrogate reward sometimes even outperforms the case with observing the true reward. We conjecture that the way of adding noise and then removing the bias (or moderate noise) introduces implicit exploration. 
% This implies that for settings even with true reward, we might consider manually adding noise and then remove it in expectation. 
This may also imply why some algorithms with estimated confusion matrices $\tilde{\rmC}$ leads to better results than with known $\rmC$ in some cases (Table~\ref{tab:scores-control}).

\paragraph{Pendulum} The goal in \textit{Pendulum} is to keep a frictionless pendulum standing up. Different from the CartPole setting, the rewards in pendulum are continuous: $r \in (-16.28, 0.0]$. The closer the reward is to zero, the better performance the model achieves. For simplicity, we firstly discretized $(-17, 0]$
% Following our extension,
%(see Section~\ref{sec:method}), 
% the $(-17, 0]$ is firstly discretized 
into 17 intervals: $(-17, -16], (-16, -15], \cdots, (-1, 0]$, with its value approximated using its maximum point. After the quantization step, the surrogate rewards can be estimated using multi-outcome extensions. % presented in Section \ref{sec:method}. 

We experiment two popular algorithms, DDPG~\cite{ddpg} and NAF~\cite{naf} in this game. In Figure~\ref{fig:pendulum}, both algorithms perform well with surrogate rewards under different amounts of noise. In most cases, the biases were corrected in the long-term, even when the amount of noise is extensive (e.g., $\omega = 0.7$). The quantitative scores on CartPole and Pendulum are given in Table~\ref{tab:scores-control}, where the scores are averaged based on the last 30 episodes. 
%The full results ($\omega > 0.5$) can be found in Appendix~\ref{appendix:quatitative_results}, so does Table~\ref{tab:scores}. 
Our reward robust method is able to achieve good scores consistently. 

\paragraph{Atari} We validate our algorithm on seven Atari 2600 games using the state-of-the-art algorithm PPO~\cite{ppo}. The games are chosen to cover a variety of environments. The rewards in the Atari games are clipped into $\{-1, 0, 1\}$. We leave the detailed settings to Appendix~\ref{appendix:exp_setup}.

%For instance, Pong is a one-to-one game but Space Invaders have multiple enemies. In experiments, the input to the network policy is a concatenation of the last four images, converted from RGB to gray scale and resized into $84 \times 84$. Note that the rewards in the Atari games are discrete and clipped into $[-1, 0, 1]$. Except for Pong game, in which $r = -1$ means missing the ball hit by the adversary, the agents in other games attempt to get higher scores in the episode with binary rewards $0$ and $1$.

% \yang{the game details can go to appendix if we run short of space.}

Results for PPO on Pong-v4 in symmetric noise setting are presented in Figure \ref{fig:pong}. More results on other Atari games and noise settings are given in Appendix~\ref{appendix:atari_results}. Similar to previous results, our surrogate estimator performs consistently well and helps PPO converge to the optimal policy. 
% Note that the performance is worse when $\omega$ is close to 0.5 ($1 - e_{+} - e_{-} \rightarrow 0$). This is because in this case, the noise level is too high such that the observations have almost lost all information about the true reward. In higher noise rate $\omega=0.9$, the model converges faster. This is because when $\omega > 0.5$, the estimator will reverse the rewards ($1 - e_{+} - e_{-} < 0$) symmetrically. In consequence, the $\omega = 0.9$ setting is similar to $\omega = 0.1$. 
Table~\ref{tab:scores} shows the average scores of PPO on five selected Atari games with different amounts of noise (symmetric $\&$ asymmetric). In particular, when the noise rates $e_{+} = e_{-} > 0.3$, agents with surrogate rewards obtain significant amounts of improvements in average scores. For the cases with unknown $\rmC$ ($\dot{r}$ in Table~\ref{tab:scores}), due to the large state-space (image-input) in confusion matrix estimation, we embed and consider the adjacent frames within a batch as the same state and set the memory size for states as 1,000. Please refer to Appendix~\ref{sec:discretization}
for details.

% We don't present the results for the case with unknown $\rmC$ although the state-space (image-input) is very large for Atari games, which is difficult to handle with the solution given in Section~\ref{sec:estimate_c}. 
% Besides, when the noise is asymmetric, the games become much more challenging for the misleading rewards.% Nonetheless, the proxy again leads to better scores.  %despite the performance is relatively lower in asymmetric noise.

\subsection{Compatible with Variance Reduction Techniques}
As illustrated in Theorem~\ref{thm:variance}, our surrogate rewards introduce larger variance while conducting unbiased estimation, which are likely to decrease the stability of RL algorithms. Apart from the linear combination idea (a linear trade-off), some variance reduction techniques in statistics (e.g., correlated sampling) can also be applied to our method. Specially, \citeauthor{Romoff2018} proposed to use a reward estimator to compensate for stochastic corrupted-reward signals. It is worthy to notice that their method is designed for variance reduction under zero-mean noises, which is no longer efficacious in more general \textit{perturbed-reward} setting. However, it is potential to integrate their method with our \textit{robust-reward} RL framework because surrogate rewards provide unbiasedness guarantee.

To verify this idea, we repeated the experiments of \textit{Cartpole} %in Section~\ref{sec:exp_evaluation} 
but included variance reduction step for estimated surrogate rewards. Following \citeauthor{Romoff2018}, we adopted sample mean as a simple approximator during the training and set sequence length as $100$. As shown in Figure~\ref{fig:cartpole_var_part}, the models with only variance reduction technique (red lines) suffer from huge regrets, and in general do not converge to the optimal policies. Nevertheless, the variance reduction step helps surrogate rewards (purple lines) to achieve faster convergence or better performance in multiple cases. Similarly, Table~\ref{tab:scores-control-var} in Appendix~\ref{appendix:est_c}
provides quantitative results which show that our surrogate reward benefits from variance reduction techniques (``ours + VRT''), especially when the noise rate is high.

\section{Conclusions}
% \yang{i'll rewrite this tomorrow...}
%Reinforcement learning (RL) has unlocked great success in dealing with problems that invoke interactions with the real world. Nonetheless, 
Improving the robustness of RL in the settings with perturbed and noisy rewards is important given the fact that such noises are common when exploring a real-world scenario, such as sensor errors. 
In addition, in adversarial environments, perturbed reward could be leveraged 
Different robust RL algorithms have been proposed but they either only focus on the noisy observations or need strong assumption on the unbiased noise distribution for observed rewards.
In this paper, we propose the first simple yet effective RL framework for dealing with biased noisy rewards. The convergence guarantee and finite sample complexity of $Q$-Learning (or its variant) with estimated surrogate rewards are provided. To validate the effectiveness of our approach, extensive experiments are conducted on OpenAI Gym, showing that surrogate rewards successfully rescue models from misleading rewards even at high noise rates.
We believe this work will further shed light on exploring robust RL approaches under different noisy rewards observations in real-world environments.

\section*{Acknowledgement}
This work was supported by National Science Foundation award CCF-1910100 and DARPA award ASED-00009970.

\bibliography{noisy_reward,noise_learning}
\bibliographystyle{aaai}

\clearpage
\appendix
\setcounter{secnumdepth}{1} 
\onecolumn
\addcontentsline{toc}{section}{Appendices}
%\section{Convergence and Sample Complexity Proofs}
\section{Proofs}
\label{appendix:proofs}
\begin{proof}[Proof of Lemma \ref{lemma:1}]

For simplicity, we shorthand $\hat{r}(s_t,a_t,s_{t+1}), \tilde{r}(s_t,a_t,s_{t+1}), r(s_t,a_t,s_{t+1})$ as $\hat{r}, \tilde{r}, r$, and let $r_+,r_-, \hat{r}_+,\hat{r}_-$ denote the general reward levels and corresponding surrogate ones: 
\begin{align}
\mathbb{E}_{\tilde{r}|r}(\hat{r})  = \mathbb{P}_{\tilde{r}|r}(\hat{r} = \hat{r}_{-}) \hat{r}_{-} + \mathbb{P}_{\tilde{r}|r}(\hat{r} = \hat{r}_{+}) \hat{r}_{+}. \label{unbias}
\end{align}
When $r=r_+$, from the definition in Lemma~\ref{lemma:1}:
\begin{align*}
\mathbb{P}_{\tilde{r}|r}(\hat{r} = \hat{r}_{-}) = e_+,
~\mathbb{P}_{\tilde{r}|r}(\hat{r} = \hat{r}_{+}) = 1-e_+.
\end{align*}
Taking the definition of surrogate rewards Eqn.~(\ref{eq:binary}) into Eqn.~(\ref{unbias}), we have
\begin{align*}
\mathbb{E}_{\tilde{r}|r}(\hat{r})  &= e_+ \cdot \hat{r}_{-} + (1-e_+)\cdot \hat{r}_{+}\\
&= e_+ \cdot \frac{(1 - e_{+})r_{-7} - e_{-}r_{+}}{1 - e_{-} - e_{+}} + (1-e_+) \cdot \frac{(1 - e_{-})r_{+} - e_{+}r_{-}}{1 - e_{-} - e_{+}} = r_{+}.
\end{align*}
Similarly, when $r = r_{-}$, it also verifies $\mathbb E_{\tilde{r}|r}[\hat{r}(s_t,a_t,s_{t+1})] = r(s_t,a_t,s_{t+1}).$
% $\mathbb E_{\tilde{r}|r}[\hat{r}(s_t,a_t,s_{t+1})] = r(s_t,a_t,s_{t+1}).$
% \begin{align*}
% \mathbb{P}_{\hat{r}|r}(\hat{r} = \hat{r}_{-}) = e_-, \mathbb{P}_{\hat{r}|r}(\hat{r} = \hat{r}_{+}) = 1-e_- \\
% \mathbb{E}_{\tilde{r}|r}(\hat{r})  = e_- \cdot \hat{r}_{-}+ (1-e_-)\cdot \hat{r}_{+} = r_{-}.
% \end{align*}
% In consequence,
% \begin{align*}
% \mathbb E_{\tilde{r}|r}[\hat{r}(s_t,a_t,s_{t+1})] = r(s_t,a_t,s_{t+1}).
% \end{align*}
\end{proof}

% From the law of total probability and the definition of confusion matrices, we have
% \begin{align*}
% \mathbb{E}(\hat{r}) &= \mathbb{P}(\hat{r} = \hat{r}_{-}) \hat{r}_{-} + \mathbb{P}(\hat{r} = \hat{r}_{+}) \hat{r}_{+} = \mathbb{P}(\tilde{r} = r_{-}) \hat{r}_{-} + \mathbb{P}(\tilde{r} = r_{+}) \hat{r}_{+} \\
% &= \left[ \mathbb{P}(r = r_{-}) \mathbb{P}(\tilde{r} = r_{-} | r = r_{-}) + \mathbb{P}(r = r_{+}) \mathbb{P}(\tilde{r} = r_{-} | r = r_{+}) \right] \hat{r}_{-} \\
% & \quad + \left[\mathbb{P}(r = r_{-}) \mathbb{P}(\tilde{r} = r_{+} | r = r_{-}) + \mathbb{P}(r = r_{+}) \mathbb{P}(\tilde{r} = r_{+} | r = r_{+}) \right] \hat{r}_{+} \\
% &= \left[\mathbb{P}(r = r_{-}) (1 - e_{-}) + \mathbb{P}(r = r_{+}) e_{+}\right] \hat{r}_{-} + \left[\mathbb{P}(r = r_{-}) e_{-} + \mathbb{P}(r = r_{+}) (1 - e_{+}) \right] \hat{r}_{-}
% \end{align*}

% Note we have focused on the case with $r_+ = +1, r_- = -1$ in the main text. 
% Taking the following surrogate rewards into above equation:
% $$
% \begin{cases}
% \begin{aligned}
% \hat{r}_{-} &= \frac{(1 - e_{+})r_{-} - e_{-}r_{+}}{1 - e_{-} - e_{+}}\\
% \hat{r}_{+} &= \frac{(1 - e_{-})r_{+} - e_{+}r_{-}}{1 - e_{-} - e_{+}}
% \end{aligned},
% \end{cases}
% $$
% we obtain the unbiasedness property, \textit{i.e.},
% $$
% \mathbb{E}(\hat{r}) = \mathbb{P}(r = r_{-})r_{-} + \mathbb{P}(r = r_{+})r_{+} = \mathbb{E}(r).
% $$

\begin{proof}[Proof of Lemma \ref{lemma:2}]
The idea of constructing unbiased estimator is easily adapted to multi-outcome reward settings via writing out the conditions for the unbiasedness property (s.t. $\mathbb E_{\tilde{r}|r}[\hat{r}] = r.
$). For simplicity, we shorthand $\hat{r}(\tilde{r} = R_i)$ as $\hat{R}_i$ in the following proofs. Similar to Lemma~\ref{lemma:1}, we need to solve the following set of functions to obtain $\hat{r}$: % $\varphi(\cdot)$:

$$
\begin{cases}
\begin{aligned}
R_0 &= c_{0,0} \cdot \hat{R}_0 + c_{0,1}  \cdot \hat{R}_1 + \dots +c_{0,M-1}  \cdot \hat{R}_{M-1}\\
R_1 &= c_{1,0}  \cdot \hat{R}_0 + c_{1,1}  \cdot \hat{R}_1+ \dots +c_{1,M-1}  \cdot \hat{R}_{M-1}\\
& \cdots \\
R_{M-1} &= c_{M-1,0}  \cdot \hat{R}_0 + c_{M-1,1}  \cdot \hat{R}_1 + \dots + c_{M-1,M-1}  \cdot \hat{R}_{M-1}
\end{aligned}
\end{cases}
$$

where $\hat{R}_i$ denotes the value of the surrogate reward when the observed reward is $R_{i}$. Define $\rmR := \left[R_0; R_1; \cdots; R_{M-1}\right]$, and $\hat{\rmR}:= [\hat{R}_0, \hat{R}_1, ..., \hat{R}_{M-1}]$, then the above equations are equivalent to:
$\rmR  = \rmC \cdot \hat{\rmR}. 
$ If the confusion matrix $\rmC$ is invertible, we obtain the surrogate reward:
\begin{align*}
\hat{\rmR} = \rmC^{-1} \cdot \rmR. %\label{eq:multivariate}
\end{align*}
According to above definition, for any true reward level $R_i, i=0, 1, \cdots, M-1$, we have
\begin{align*}
\mathbb E_{\tilde{r}|r=R_i}[\hat{r}] = c_{i,0}  \cdot \hat{R}_0 + c_{i,1}  \cdot \hat{R}_1 + \dots + c_{i,M-1}  \cdot \hat{R}_{M-1}=R_{i}.    
\end{align*}

\end{proof}
% \yang{i changed the notation in the main text; please correct here}

% \yang{missing section title for this proof}

Furthermore, the probabilities for observing surrogate rewards can be written as follows:
\begin{align*}
\hat{\mathbf{P}} &= \left[\hat{p}_1, \hat{p}_2, \cdots, \hat{p}_M \right] = \left[\sum_{j} p_j c_{j,1}, \sum_{j} p_j c_{j,2}, \cdots, \sum_{j} p_j c_{j,M} \right],
\end{align*}
where $\hat{p}_i = \sum_{j} p_j c_{j,i}$, and $\hat{p}_i$, $p_i$ represent the probabilities of occurrence for surrogate reward $\hat{R}_i$ and true reward $R_i$ respectively.

% $p_i$, $\hat{p}_i$ and $\mathbb{P}_a(s_t, s_{t+1})$
\begin{corollary}
Let $\hat{p}_i$ and $p_i$ denote the probabilities of occurrence for surrogate reward $\hat{r}(\tilde{r}=R_i)$ and true reward $R_i$. Then the surrogate reward satisfies,
\label{corollary:1}
\begin{align}
\label{eq:p_reward}
\sum_{s_{t+1} \in \mathcal{S}}\mathbb{P}_a(s_t, s_{t+1}) r(s_t, a, s_{t+1}) = \sum_{j} p_j R_j= \sum_{j} \hat{p}_j \hat{R}_j.
\end{align}
\end{corollary}

% \begin{align*}
% \sum_{s' \in \mathcal{S}}\mathbb{P}_a(s_t, s_{t+1}) = \sum_{j} p_j = \sum_{j} \hat{p}_j = 1,
% \end{align*}

\begin{proof}[Proof of Corollary \ref{corollary:1}]
From Lemma~\ref{lemma:2}, we have,
\begin{align*}
\sum_{s_{t+1} \in \mathcal{S}}\mathbb{P}_a(s_t, s_{t+1}) r(s_t, a, s_{t+1}) = \sum_{s_{t+1} \in \mathcal{S}; R_j \in \mathcal{R}}\mathbb{P}_a(s_t, s_{t+1}, R_j) R_j \\
= \sum_{R_j \in \mathcal{R}} \sum_{s_{t+1} \in \mathcal{S}} \mathbb{P}_a(s_t, s_{t+1}) R_j = \sum_{R_j \in \mathcal{R}} p_j R_j = \sum_{j} p_j R_j.
\end{align*}
Consequently, 
\begin{align*}
\sum_{j} \hat{p}_j \hat{R}_j &= \sum_{j} \sum_{k} p_k c_{k,j} \hat{R}_j = \sum_{k} p_k  \sum_{j} c_{k,j} \hat{R}_j \\
&= \sum_{k} p_k R_k = \sum_{s_{t+1} \in \mathcal{S}}\mathbb{P}_a(s_t, s_{t+1}) r(s_t, a, s_{t+1}).
\end{align*}
\end{proof}

To establish Theorem~\ref{thm:convergence}, we need an auxiliary result (Lemma~\ref{lemma:3}) from stochastic process approximation, which is widely adopted for the convergence proof for $Q$-Learning~\cite{DBLP:conf/nips/JaakkolaJS93,DBLP:journals/ml/Tsitsiklis94}.

\begin{lemma}
The random process $\{\Delta_t\}$ taking values in $\mathbb{R}^{n}$ and defined as 
$$
\Delta_{t+1}(x) = (1 - \alpha_t(x)) \Delta_t(x) + \alpha_t(x) F_t(x)
$$
converges to zero w.p.1 under the following assumptions:
\begin{itemize}
\item $0 \leq \alpha_t \leq 1$, $\sum_t \alpha_t(x) = \infty$ and $\sum_t \alpha_t(x)^{2} < \infty$;
\item $|| \mathbb{E} \left[ F_t(x) | \mathcal{F}_t \right] ||_W \leq \gamma || \Delta_t ||$, with $\gamma < 1$;
\item ${\textbf{var}} \left[ F_t(x) | \mathcal{F}_t \right] \leq C(1 + || \Delta_t||^2_W)$, for $C > 0$.
\end{itemize}
Here $\mathcal{F}_t = \{\Delta_t, \Delta_{t-1}, \cdots, F_{t-1}\, \cdots, \alpha_{t}, \cdots \}$ stands for the past at step $t$, $\alpha_t(x)$ is allowed to depend on the past insofar as the above conditions remain valid. The notation $||\cdot||_W$ refers to some weighted maximum norm.
\label{lemma:3}
\end{lemma}
\begin{proof}[Proof of Lemma \ref{lemma:3}]
See previous literature~\cite{DBLP:conf/nips/JaakkolaJS93,DBLP:journals/ml/Tsitsiklis94}.
\end{proof}

%\begin{theorem}
%Given a finite MDP, denoting as $\mathcal{M} = \langle \mathcal{S}, \mathcal{A}, \mathcal{\hat{R}}, \mathcal{P}, \gamma\rangle$, the $Q$-learning algorithm with surrogate rewards, given by the update rule,
%\begin{equation*}
%Q'(s, a) = (1 - \alpha)Q(s, a) + \alpha \left[\hat{r} + \gamma \max_{b \in \mathcal{A}} Q(s', b) \right],
% \label{eq:rule}
%\end{equation*}
%converges w.p.1 to the optimal $Q$-function as long as $\sum_{t}\alpha = \infty$ and %$\sum_{t}\alpha^2 < \infty$.
%\end{theorem}

\begin{proof}[Proof of Theorem \ref{thm:convergence}]
For simplicity, we abbreviate $s_{t}$, $s_{t+1}$, $Q_{t}$, $Q_{t+1}$, $r_t$, $\hat{r}_{t}$ and $\alpha_t$ as $s$, $s'$, $Q$, $Q'$, $r$, $\hat{r}$, and $\alpha$, respectively. 

Subtracting from both sides the quantity $Q^\ast(s, a)$ in~Eqn.~(\ref{eq:rule}):
\begin{align*}
Q'(s, a) - Q^\ast(s, a) =& (1 - \alpha) \left(Q(s, a) - Q^\ast(s, a)\right) + \alpha \left[\hat{r} + \gamma \max_{b \in \mathcal{A}} Q(s', b) - Q^\ast(s, a)\right].
\end{align*}
Let $\Delta_{t}(s, a) = Q(s, a) - Q^\ast(s, a)$ and $F_t(s, a) = \hat{r} + \gamma \max_{b \in \mathcal{A}} Q(s', b) - Q^\ast(s, a)$.
\begin{equation*}
\Delta_{t+1}(s', a) = (1 - \alpha) \Delta_{t}(s, a) + \alpha F_t(s, a) .
\end{equation*}
In consequence,

\begin{align*}
\mathbb{E} \left[ F_t(x) | \mathcal{F}_t \right] &= \sum_{s' \in \mathcal{S}; \hat{r} \in \mathcal{R}} \mathbb{P}_a(s, s', \hat{r})\left[\hat{r} + \gamma \max_{b \in \mathcal{A}} Q(s', b) \right] - Q^\ast(s, a)\\
&= \sum_{s' \in \mathcal{S}; \hat{r} \in \mathcal{R}} \mathbb{P}_a(s, s', \hat{r}) \hat{r} + \sum_{s' \in \mathcal{S}} \mathbb{P}_a(s, s') \left[ \gamma \max_{b \in \mathcal{A}} Q(s', b) - r - \gamma \max_{b \in \mathcal{A}}Q^\ast(s', b) \right] \\
&= \sum_{s' \in \mathcal{S}; \hat{r} \in \mathcal{R}} \mathbb{P}_a(s, s', \hat{r}) \hat{r} - \sum_{s' \in \mathcal{S}} \mathbb{P}_a(s, s') r + \sum_{s' \in \mathcal{S}} \mathbb{P}_a(s, s') \gamma \left[\max_{b \in \mathcal{A}} Q(s', b) - \max_{b \in \mathcal{A}}Q^\ast(s', b) \right] \\
&= \sum_{j} \hat{p}_j \hat{r}_j - \sum_{s' \in \mathcal{S}} \mathbb{P}_a(s, s') r + \sum_{s' \in \mathcal{S}} \mathbb{P}_a(s, s') \gamma \left[  \max_{b \in \mathcal{A}} Q(s', b) - \max_{b \in \mathcal{A}}Q^\ast(s', b)\right] \\
&= \sum_{s' \in \mathcal{S}} \mathbb{P}_a(s, s') \gamma \left[ \max_{b \in \mathcal{A}} Q(s', b) - \max_{b \in \mathcal{A}}Q^\ast(s', b)\right] \\
&\leq \gamma \sum_{s' \in \mathcal{S}} \mathbb{P}_a(s, s') \max_{b \in \mathcal{A}, s' \in \mathcal{S}} \left\lvert Q(s', b) - Q^\ast(s', b) \right\rvert \\
&= \gamma \sum_{s' \in \mathcal{S}} \mathbb{P}_a(s, s') || Q - Q^\ast ||_{\infty} = \gamma || Q - Q^\ast ||_{\infty} = \gamma || \Delta_t ||_{\infty}.
\end{align*}

Finally, 
\begin{align*}
{\mathbf{Var}} \left[ F_t(x) | \mathcal{F}_t \right] 
&= \mathbb{E} \Bigg[\bigg(\hat{r} + \gamma \max_{b \in \mathcal{A}} Q(s', b) - \sum_{s' \in \mathcal{S}; \hat{r} \in \mathcal{R}} \mathbb{P}'(s, s', \hat{r})\left[\hat{r} + \gamma \max_{b \in \mathcal{A}} Q(s', b) \right] \bigg)^2\Bigg] \\
&= {\mathbf{Var}} \left[\hat{r} + \gamma \max_{b \in \mathcal{A}} Q(s', b) | \mathcal{F}_t \right]\\.
\end{align*}
Because $\hat{r}$ is bounded, it can be clearly verified that
\begin{equation*}
{\mathbf{Var}} \left[ F_t(x) | \mathcal{F}_t \right] \leq C(1 + || \Delta_t||^2_W)
\end{equation*}
for some constant $C$.
Then, due to the Lemma~\ref{lemma:3}, $\Delta_t$ converges to zero w.p.1, \textit{i.e.}, $Q'(s,a)$ converges to $Q^\ast(s, a)$.
\end{proof}

The procedure of \textit{Phased $Q$-Learning} is described as Algorithm~\ref{alg:phased_qlearn}:

\begin{algorithm}
\caption{Phased $Q$-Learning}\label{alg:phased_qlearn}
\begin{algorithmic}[]
\STATE {\bfseries Input:} $G(\mathcal{M})$: generative model of $\mathcal{M} = (\mathcal{S}, \mathcal{A}, \mathcal{R}, \mathcal{P}, \gamma)$, $T$: number of iterations.
\STATE {\bfseries Output:} $\hat{V}(s)$: value function, $\hat \pi (s, t)$: policy function.
\STATE Set $\hat{V}_T(s) = 0$ 
\FOR{$t = T - 1, \cdots, 0$}
\STATE Calling $G(\mathcal{M})$ $m$ times for each state-action pair.
    $$
    \hat{\mathbb{P}}_a (s_t, s_{t+1}) = \frac{\# [(s_{t}, a_{t}) \to s_{t+1}]}{m}
    $$
    \STATE Set 
    \begin{align*}
    \hat{V}(s_t) &= \max_{a \in \mathcal{A}}\sum_{s_{t+1} \in \mathcal{S}} \hat{\mathbb{P}}_a(s_t, s_{t+1})\left[ r_t + \gamma \hat{V}(s_{t+1})\right]\\
    \hat \pi (s, t) &= \argmax_{a \in \mathcal{A}} \hat{V}(s_t)
    \end{align*}
\ENDFOR
\STATE {\bfseries return} $\hat{V}(s)$ and $\hat \pi (s, t)$  
\end{algorithmic}
\end{algorithm}

Note that $\hat{\mathbb{P}}$ here is the estimated transition probability, which is different from $\mathbb{P}$ in Eqn.~(\ref{eq:p_reward}).

% \begin{proof}[Proof of Lemma \ref{thm:lower_bound}]
% See previous literature~\cite{Kakade2003OnTS}.
% \end{proof}

To obtain the sample complexity results, the range of our surrogate reward needs to be known. Assuming reward $r$ is bounded in $[0, R_{\max}]$, Lemma~\ref{lemma:upper_bound_reward} below states that the surrogate reward is also bounded, when the confusion matrices are invertible:
\begin{lemma}
\label{lemma:upper_bound_reward}
Let $r \in [0, R_{\max}]$ be bounded, where $R_{\max}$ is a constant; suppose $\rmC_{M\times M}$, the confusion matrix, is invertible with its determinant denoting as $\mathrm{det}(\rmC)$. Then the surrogate reward satisfies
\begin{align}
    0 \leq \left| \hat{r} \right| \leq \frac{M}{\mathrm{det}(\rmC)} R_{\max}.
\end{align}
\end{lemma}
% \yang{move this lemma to appendix; you only need this lemma in proving Theorem 2, but you dont need to introduce it.} \jingkang{agree!}

\begin{proof}[Proof of Lemma \ref{lemma:upper_bound_reward}]
From Eqn.~(\ref{eq:multivariate}), we have,
\begin{align*}
    \hat{\rmR} = \rmC^{-1} \cdot \rmR = \frac{\mathrm{adj}(\rmC)}{\mathrm{det}( \rmC )} \cdot \rmR,
\end{align*}
where $\mathrm{adj}(\rmC)$ is the adjugate matrix of $\rmC$; $\mathrm{det}(\rmC)$ is the determinant of $\rmC$.
It is known from linear algebra that,
\begin{align*}
    \mathrm{adj}(\rmC)_{ij} = (-1)^{i+j} \cdot \rmM_{ji},
\end{align*}
where $\rmM_{ji}$ is the determinant of the $(M - 1) \times (M - 1)$ matrix that results from deleting row $j$ and column $i$ of $\rmC$.
Therefore, $\rmM_{ji}$ is also bounded:
\begin{align*}
\rmM_{ji} \leq \sum_{\sigma \in S_{n}} \left(\left|\mathrm{sgn}(\sigma)\right| \prod\limits_{m = 1} c'_{m, \sigma_n} \right) \leq \prod
\limits_{m=0}^{M-1}\left(\sum\limits_{n=0}^{M-1} c_{m, n}\right) = 1^M = 1,
\end{align*}
where the sum is computed over all permutations $\sigma$ of the set $\{0, 1, \cdots , M-2\}$; $c'$ is the element of $\rmM_{ji}$; $\mathrm{sgn}(\sigma)$ returns a value that is $+1$ whenever the reordering given by $\sigma$ can be achieved by successively interchanging two entries an even number of times, and $-1$ whenever it can not.

Consequently,
\begin{align*}
    \left|\hat{R}_i\right|= \frac{\sum_j \left|\mathrm{adj}(\rmC)_{ij}\right| \cdot \left|R_j\right|}{\mathrm{det}(\rmC)}  \leq \frac{M}{\mathrm{det}(\rmC)} \cdot R_{\max}.
\end{align*}

% \begin{align*}
%     \rmM_{(j+1)(i+1)} = \begin{vmatrix}
%         c_{0,0}   & \cdots & c_{0,i-1} &  c_{0,i+1} & \cdots & c_{0,M-1} \\
%         \cdots & \cdots & \cdots & \cdots & \cdots & \cdots\\
%         c_{j-1,0}   & \cdots & c_{j-1,i-1} &  c_{j-1,i+1} & \cdots & c_{j-1,M-1} \\
%         c_{j+1,0}   & \cdots & c_{j+1,i-1} &  c_{j+1,i+1} & \cdots & c_{j+1,M-1} \\
%         \cdots & \cdots & \cdots & \cdots & \cdots & \cdots\\
%         c_{M-1,0}   & \cdots & c_{M-1,i-1} &  c_{M-1,i+1} & \cdots & c_{M-1,M-1}
%     \end{vmatrix} = 
% \end{align*}
\end{proof}

\begin{proof}[Proof of Theorem \ref{thm:upper_bound}]
From Hoeffding's inequality, we obtain:
\begin{align*}
    P\left(\left| \sum_{s_{t+1} \in \mathcal{S}} \mathbb{P}_a(s_t, s_{t+1}) V_{t+1}^\ast(s_{t+1}) - \sum_{s_{t+1} \in \mathcal{S}} \hat{\mathbb{P}}_a(s_t, s_{t+1}) V_{t+1}^\ast(s_{t+1}) \right| \geq \epsilon \right) \leq 2\exp{\left(\frac{-2 m\epsilon^2 (1 - \gamma)^2}{R_{\max}^2}\right)},
\end{align*}
because $V_{t}(s_t)$ is bounded within $\frac{R_{\max}}{1 - \gamma}$. % because the number of samples is $m$ and error per phrase is $\epsilon \over T$.
In the same way, $\hat{r}_t$ is bounded by $\frac{M}{\mathrm{det}(\rmC)} \cdot R_{\max}$ from Lemma~\ref{lemma:upper_bound_reward}. We then have,
\begin{align*}
    P\left(\left| \sum_{\substack{s_{t+1} \in \mathcal{S}\\\hat{r}_t \in \mathcal{\hat{R}}}} \mathbb{P}_a(s_t, s_{t+1}, \hat{r}_t) \hat{r}_t - \sum_{\substack{s_{t+1} \in \mathcal{S}\\\hat{r}_t \in \mathcal{\hat{R}}}} \hat{\mathbb{P}}_a(s_t, s_{t+1}, \hat{r}_t) \hat{r}_t \right| \geq \epsilon\right) \leq 2\exp{\left(\frac{-2 m\epsilon^2 \mathrm{det}(\rmC)^2}{M^2 R_{\max}^2} \right)}.
\end{align*}

Further, due to the unbiasedness of surrogate rewards, we have 
\begin{align*}
    \sum_{s_{t+1} \in \mathcal{S}} \mathbb{P}_a(s_t, s_{t+1}) r_t = \sum_{s_{t+1} \in \mathcal{S}; \hat{r}_t \in \mathcal{\hat{R}}} \mathbb{P}_a(s_t, s_{t+1}, \hat{r}_t) \hat{r}_t.
\end{align*}

As a result,
\begin{align*}
\left|V_{t}^\ast(s) - \hat{V}_{t}(s) \right| &= \max_{a \in \mathcal{A}} \sum_{s_{t+1} \in \mathcal{S}} \mathbb{P}_a(s_t, s_{t+1})\left[ r_t + \gamma V_{t+1}^\ast(s_{t+1}) \right] - \max_{a \in \mathcal{A}} \sum_{s_{t+1} \in \mathcal{S}} \hat{\mathbb{P}}_a(s_{t}, s_{t+1})\left[ \hat{r}_t + \gamma V_{t+1}^\ast(s_{t+1}) \right] \\
& \leq 
\epsilon_1 + \gamma \max_{a \in \mathcal{A}} \left| \sum_{s_{t+1} \in \mathcal{S}} \mathbb{P}_a(s_t, s_{t+1}) V_{t+1}^\ast(s_{t+1}) - \sum_{s_{t+1} \in \mathcal{S}} \hat{\mathbb{P}}_a(s_t, s_{t+1}) V_{t+1}^\ast(s_{t+1}) \right| \\ &\quad +  \max_{a \in \mathcal{A}} \left| \sum_{s_{t+1} \in \mathcal{S}} \mathbb{P}_a(s_t, s_{t+1}) r_t - \sum_{s_{t+1} \in \mathcal{S}; \hat{r}_t \in \mathcal{\hat{R}}} \mathbb{P}_a(s_t, s_{t+1}, \hat{r}_t) \hat{r}_t \right|\\
& \leq \gamma \max_{s \in \mathcal{S}} \left|V_{t+1}^\ast(s) - \hat{V}_{t+1}(s) \right| + \epsilon_1 + \gamma \epsilon_2
\end{align*}
In the same way,
\begin{align*}
\left|V_{t}(s) - \hat{V}_{t}(s) \right| \leq \gamma \max_{s \in \mathcal{S}} \left|V_{t+1}^\ast(s) - \hat{V}_{t+1}(s) \right| + \epsilon_1 + \gamma \epsilon_2
\end{align*}

Recursing the two equations in two directions ($0 \to T$), we get
\begin{align*}
\max_{s \in \mathcal{S}} \left|V^{\ast}(s) - \hat{V}(s) \right| & \leq (\epsilon_1 + \gamma \epsilon_2) + \gamma (\epsilon_1 + \gamma \epsilon_2) + \cdots + \gamma^{T - 1} (\epsilon_1 + \gamma \epsilon_2) \\
    &=  \frac{(\epsilon_1 + \gamma \epsilon_2)(1 - \gamma^{T})}{1 - \gamma} \\
\max_{s \in \mathcal{S}} \left|V(s) - \hat{V}(s) \right| & \leq \frac{(\epsilon_1 + \gamma \epsilon_2)(1 - \gamma^{T})}{1 - \gamma}
\end{align*}

Combining these two inequalities above we have:
$$
\max_{s \in \mathcal{S}} \left|V^{\ast}(s) - V(s) \right| \leq 2 \frac{(\epsilon_1 + \gamma \epsilon_2)(1 - \gamma^{T})}{1 - \gamma} \leq 2 \frac{(\epsilon_1 + \gamma \epsilon_2)}{1 - \gamma}.
$$

Let $\epsilon_1 = \epsilon_2$, so
 $\max_{s \in \mathcal{S}} \left|V^{\ast}(s) - V(s) \right| \leq \epsilon$ as long as
$$\epsilon_1 = \epsilon_2 \leq \frac{(1 - \gamma)\epsilon}{2(1 + \gamma)}.$$

For arbitrarily small $\epsilon$, by choosing $m$ appropriately, there always exists $\epsilon_1 = \epsilon_2 = \frac{(1 - \gamma)\epsilon}{2(1 + \gamma)}$ such that the policy error is bounded within $\epsilon$. That is to say, the \textit{Phased Q-Learning} algorithm can converge to the near optimal policy within finite steps using our proposed surrogate rewards.

Finally, there are $|\mathcal{S}||\mathcal{A}|T$ transitions under which these conditions must hold, where $|\cdot|$ represent the number of elements in a specific set. Using a union bound, the probability of failure in any condition is smaller than 
$$2|\mathcal{S}||\mathcal{A}|T \cdot \exp \left({- m \frac{\epsilon^2(1 - \gamma)^2}{2(1 + \gamma)^2}} \cdot \min\{(1 - \gamma)^2, \frac{\mathrm{det}(\rmC)^2}{M^2}\} \right).$$ 

We set the error rate less than $\delta$, and $m$ should satisfy that
$$
m = O\left(\frac{1}{\epsilon^2(1 - \gamma)^2\mathrm{det}(\rmC)^2}\log\frac{|\mathcal{S}||\mathcal{A}|T}{\delta}\right).
$$
In consequence, after $m|\mathcal{S}||\mathcal{A}|T$ calls, which is, $O\left(\frac{|\mathcal{S}||\mathcal{A}|T}{\epsilon^2(1 - \gamma)^2\mathrm{det}(\rmC)^2}\log\frac{|\mathcal{S}||\mathcal{A}|T}{\delta}\right)$, the value function converges to the optimal one for every state $s$, with probability greater than $1 - \delta$.
\end{proof}
The above bound is for discounted MDP setting with $0 \leq \gamma < 1$. For undiscounted setting $\gamma = 1$, since the total error (for entire trajectory of $T$ time-steps) has to be bounded by $\epsilon$, therefore, the error for each time step has to be bounded by $\frac{\epsilon}{T}$. Repeating our anayslis, we obtain the following upper bound:
$$
O\left(\frac{|\mathcal{S}||\mathcal{A}|T^3}{\epsilon^2\mathrm{det}(\rmC)^2}\log\frac{|\mathcal{S}||\mathcal{A}|T}{\delta}\right).
$$

\begin{proof}[Proof of Theorem \ref{thm:variance}]
\begin{align*}
\mathbf{Var}(\hat{r}) - \mathbf{Var}(r) &= \mathbb{E}\left[(\hat{r} - \mathbb{E}[\hat{r}])^2\right] - \mathbb{E}\left[\left({r} - \mathbb{E}[{r}] \right)^2\right] \\
&= \mathbb{E}[\hat{r}^2] - \mathbb{E}[\hat{r}]^2 + \mathbb{E}[{r}^2] - \mathbb{E}[{r}]^2 \\
&= \sum_{j}\hat{p}_{j}\hat{R_j}^2 - \left(\sum_j \hat{p}_j\hat{R}_j\right)^2 - \left[ \sum_{j}{p}_{j}{R_j}^2 - \left(\sum_j {p}_j{R}_j\right)^2 \right]\\
&= \sum_{j}\hat{p}_{j}\hat{R_j}^2 - \sum_{j}{p}_{j}{R_j}^2 \\
&= \sum_j\sum_i p_i c_{i,j}\hat{R_j}^2 - \sum_j p_j \left(\sum_i c_{j,i} \hat{R_i}\right)^2 \\
&= \sum_j p_j \left( \sum_i c_{j,i} \hat{R_i}^2 - \left(\sum_i c_{j,i}\hat{R_i}\right)^2\right).
\end{align*}
Using the Cauchy–Schwarz inequality,
\begin{align*}
\sum_i c_{j,i} \hat{R_i}^2 = \sum_i \sqrt{c_{j,i}}^2 \cdot \sum_i \left(\sqrt{c_{j,i}}\hat{R_i}\right)^2  \geq \left( \sum_ic_{j,i}\hat{R_i}\right)^2.
\end{align*}
So we get,
\begin{align*}
\mathbf{Var}(\hat{r}) - \mathbf{Var}(r) \geq 0.    
\end{align*}
In addition, 
\begin{align*}
    \mathbf{Var}(\hat{r}) &= \sum_{j}\hat{p}_{j}\hat{R_j}^2 - \left(\sum_j \hat{p}_j\hat{R}_j\right)^2 \leq \sum_{j}\hat{p}_{j}\hat{R_j}^2 \\
    & \leq \sum_{j}\hat{p}_{j} \frac{M^2}{\mathrm{det}(\rmC)^2} \cdot R_{\max}^2 = \frac{M^2}{\mathrm{det}(\rmC)^2} \cdot R_{\max}^2.
\end{align*}

% \begin{align*}
%     \mathbf{Var}(\hat{R}) &= \sum_{j}\hat{p}_{j}\hat{r_j}^2 - \left(\sum_j \hat{p}_j\hat{r}_j\right)^2 \leq \sum_{j}\hat{p}_{j}\hat{r_j}^2 - \sum_j \hat{p}_j^2\hat{r}_j^2 \\
%     & \leq \sum_{j}\hat{p}_{j} (1 - \hat{p}_{j}) \hat{r}_j^2 = \sum_{j}\hat{p}_{j} \hat{r}_j^2 - \sum_{j} \hat{p}_j^2 \hat{r}_j^2
%     \\
%     & \leq \sum_{j}\hat{p}_{j} \hat{r}_j^2 - \frac{1}{M} \left(\sum_{j}\hat{p}_{j}\hat{r}_j\right)^2
%     \\
%     & \leq \sum_{j}\hat{p}_{j} \frac{M^2}{\mathrm{det}(\rmC)^2} \cdot R_{\max}^2 = \frac{M^2}{\mathrm{det}(\rmC)^2} \cdot R_{\max}^2.
% \end{align*}
\end{proof}

\section{Experimental Setup}
\label{appendix:exp_setup}

We set up our experiments within the popular OpenAI baselines~\cite{baselines} and keras-rl~\cite{plappert2016kerasrl} framework. Specifically, we integrate the algorithms and interact with OpenAI Gym~\cite{gym} environments (Table~\ref{tab:alg}).

\begin{table}
\centering
\caption{RL algorithms utilized in the robustness evaluation.}
\label{tab:alg}
\begin{tabular}{@{}cc@{}}
\toprule
Environment               & RL Algorithm    \\ \midrule
\multirow{5}{*}{CartPole} & $Q$-Learning~\cite{Watkins:1989} \\
                          & CEM~\cite{cem} \\
                          & SARSA~\cite{sarsa}   \\
                          & DQN~\cite{dqn1,mnih2015human}         \\
                          & DDQN~\cite{dueling-dqn}  \\ \hline
\multirow{2}{*}{Pendulum} & DDPG~\cite{ddpg}\\
                          & NAF~\cite{naf}  \\ \hline
Atari Games               & PPO~\cite{ppo}  \\ \bottomrule
\end{tabular}
\end{table}

\subsection{RL Algorithms}
% \paragraph{RL Algorithms.}
A set of state-of-the-art reinforcement learning algorithms are experimented with while training under different amounts of noise, including $Q$-Learning~\cite{Watkins:1989,Watkins92q-learning}, Cross-Entropy Method (CEM)~\cite{cem}, Deep SARSA~\cite{sarsa}, Deep $Q$-Network (DQN)~\cite{dqn1,mnih2015human,double-dqn}, Dueling DQN (DDQN)~\cite{dueling-dqn}, Deep Deterministic Policy Gradient (DDPG)~\cite{ddpg}, Continuous DQN (NAF)~\cite{naf} and Proximal Policy Optimization (PPO)~\cite{ppo} algorithms. For each game and algorithm, three policies are trained based on different random initialization to decrease the variance in experiments.

\subsection{Post-Processing Rewards}
\label{sec:post-processing}
We explore both symmetric and asymmetric noise of different noise levels. For symmetric noise, the confusion matrices are symmetric, which means the probabilities of corruption for each reward choice are equivalent. For instance, a confusion matrix $$\mathbf{C} = \begin{bmatrix} 0.8 & 0.2 \\ 0.2 & 0.8 \end{bmatrix}$$ says that $r_1$ could be corrupted into $r_2$ with a probability of 0.2 and so does $r_2$ (weight~=~0.2). 

As for asymmetric noise, two types of random noise are tested: 1) \textit{rand-one}, each reward level can only be perturbed into another reward; 2) \textit{rand-all}, each reward could be perturbed to any other reward. To measure the amount of noise \textit{w.r.t} confusion matrices, we define the weight of noise as follows:
$$
\rmC = (1 - \omega) \cdot \rmI + \omega \cdot \rmN,~\omega \in [0,1],
$$
where $\omega$ controls the weight of noise; $\rmI$ and $\rmN$ denote the identity and noise matrix respectively. Suppose there are $M$ outcomes for true rewards, $\rmN$ writes as:
$$
\rmN = 
\begin{bmatrix} 
n_{0,0} & n_{0,1} & \cdots & n_{0,M-1} \\
% n_{1,0} & n_{1,1} & \cdots & n_{1,M-1} \\
\cdots & \cdots & \cdots & \cdots \\
n_{M-1,0} & n_{M-1,1} & \cdots & n_{M-1,M-1} \\
\end{bmatrix},
$$
where for each row $i$, 1) rand-one: randomly choose $j$, \textit{s.t} $n_{i,j} = 1$ and $n_{i,k} \neq 0$ if $k \neq j$; 2) rand-all: generate $M$ random numbers that sum to 1, \textit{i.e.}, $\sum_{j}n_{i,j} = 1$.
For the simplicity, for symmetric noise, we choose $\rmN$ as an anti-identity matrix. As a result, $c_{i,j} = 0$, if $i \neq j$ or $ i + j \neq M$.

\subsection{Perturbed-Reward MDP Example}
To obtain an intuitive view of the reward perturbation model, where the observed rewards are generated based on a reward confusion matrix, and meanwhile evaluate our estimation algorithm's robustness to time-variant noise, we constructed a simple MDP and evaluated the performance of \textit{robust reward $Q$-Learning} (Algorithm~\ref{alg:robust_qlearn_short}) on different noise ratios (both symmetric and asymmetric). The finite MDP is formulated as Figure~\ref{fig:mdp-example}: when the agent reaches state 5, it gets an instant reward of $r_{+} = 1$, otherwise a zero reward $r_{-} = 0$. During the explorations, the rewards are perturbed according to the confusion matrix $\rmC_{2\times2} = \left[1 - e_{-}, e_{-} ; e_{+}, 1 - e_{+} \right]$.

\begin{figure}[H]
\centering
\begin{subfigure}[b]{.30\textwidth}
    \centering
    \includegraphics[width=\textwidth]{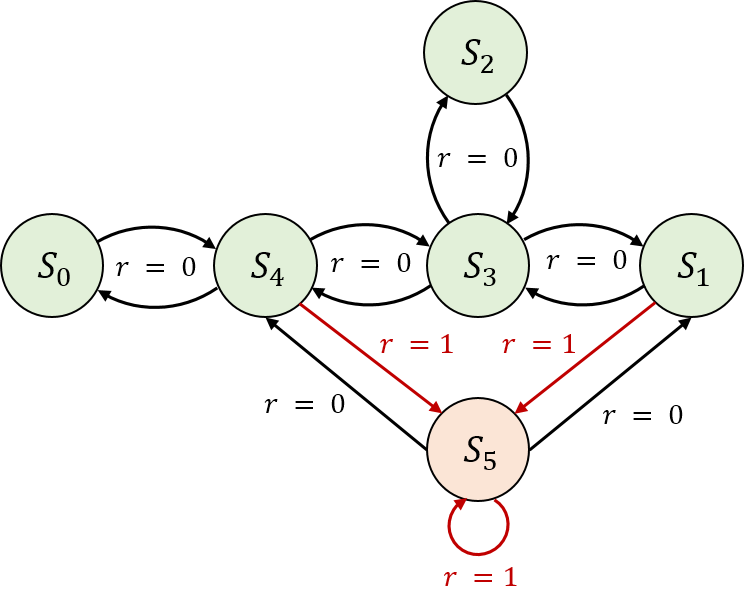}
    \caption{Finite MDP (six-state)}
    \label{fig:mdp-example}
\end{subfigure}
\hspace{0.7cm}
\begin{subfigure}[b]{.35\textwidth}
    \centering
    \includegraphics[width=\textwidth]{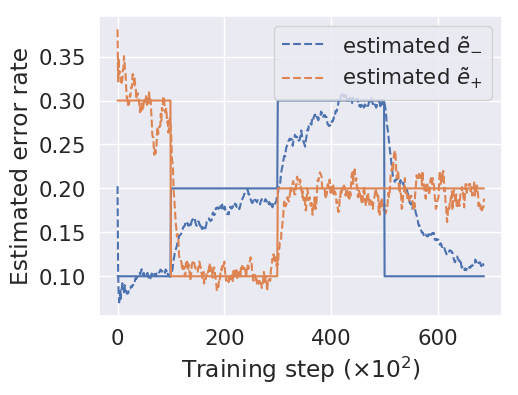}
    \caption{Estimation process in time-variant noise}
    \label{fig:mdp-time-variant}
\end{subfigure}
\caption{Perturbed-Reward MDP Example}
\end{figure}

% There are two experiments conducted in this setting: 1) performance of $Q$-Learning under different noise rates (Table~\ref{tab:mdp_acc}); 2) robustness of estimation module \textit{in time-variant noise} (Figure~\ref{fig:mdp-time-variant}). As shown in Table~\ref{tab:mdp_acc}, $Q$-Learning achieved better results consistently with the guidance of surrogate rewards and the confusion matrix estimation algorithm.
For time-variant noise, we generated varying amount of noise at different training stages: 1) $e_{-} = 0.1, e_{+} = 0.3$ ($0$ to $1e^4$ steps); 2) $e_{-} = 0.2, e_{+} = 0.1$ ($1e^4$ to $3e^4$ steps); 3) $e_{-} = 0.3, e_{+} = 0.2$ ($3e^4$ to $5e^4$ steps); 4) $e_{-} = 0.1, e_{+} = 0.2$ ($5e^4$ to $7e^4$ steps). In Figure~\ref{fig:mdp-time-variant}, we show that Algorithm~\ref{alg:robust_qlearn_short} is robust against time-variant noise, which dynamically adjusts the estimated $\tilde{\rmC}$ after the noise distribution changes. Note that we set a maximum memory size for collected noisy rewards to let the agents only learn with recent observations.

\subsection{Training Details}
\paragraph{CartPole and Pendulum}
The policies use the default network from keras-rl framework. which is a five-layer fully connected network\footnote{\url{https://github.com/keras-rl/keras-rl/examples}}. There are three hidden layers, each of which has 16 units and followed by a rectified nonlinearity. The last output layer is activated by the linear function. For CartPole, We trained the models using Adam optimizer with the learning rate of $1e^{-3}$ for 10,000 steps. The exploration strategy is Boltzmann policy. For DQN and Dueling-DQN, the update rate of target model and the memory size are $1e^{-2}$ and $50,000$. For Pendulum, We trained DDPG and NAF using Adam optimizer with the learning rate of $5e^{-4}$ for $150,000$ steps. the update rate of target model and the memory size are $1e^{-3}$ and $100,000$. 

\paragraph{Atari Games}
We adopt the pre-processing steps as well as the network architecture from ~\cite{mnih2015human}. Specifically, the input to the network is $84 \times 84 \times 4$, which is a concatenation of the last 4 frames and converted into $84 \times 84$ gray-scale. The network comprises three convolutional layers and two fully connected layers\footnote{\url{https://github.com/openai/baselines/tree/master/baselines/common}}. The kernel size of three convolutional layer are $8 \times 8$ with stride 4 (32 filters), $4 \times 4$ with stride 2 (64 filters) and $3 \times 3$ with stride 1 (64 filters), respectively. Each hidden layer is followed by a rectified nonlinearity. Except for Pong where we train the policies for $3e^7$ steps, all the games are trained for $5e^{7}$ steps with the learning rate of $3e^{-4}$. Note that the rewards in the Atari games are discrete and clipped into $\{-1, 0, 1\}$. Except for Pong game, in which $r = -1$ means missing the ball hit by the adversary, the agents in other games attempt to get higher scores in the episode with binary rewards $0$ and $1$.

\subsection{Discretization for Continuous States}
\label{sec:discretization}
To apply proposed estimation algorithm to continuous-state MDPs, we adopt a discretization procedure similar to the pre-processing of continuous rewards
%given in Section~\ref{sec:method}. 
As stated before, there is a also trade-off between the quantization error as well as the estimation complexity. However, in practice, we found that the estimation step is highly robust to the quantization level. % Consequently, we didn't fine-grained tune the parameters but set them empirically in experiments.

For \textit{Cartpole}, the observations (states) are speed and velocity of the cart, and we discretized them into $8~\text{(speeds)} \times 10~\text{(velocity)} = 80$ independent states for collecting noisy rewards for each state-action pair. In inverted \textit{Pendulum} swingup problem, the states ($\cos\theta \in [-1.0, 1.0]; \sin\theta \in [-1.0, 1.0]; d\theta/dt \in [-8.0, +8.0]$, $\theta$ denotes the rotation degree of pendulum) are discretized into $20~(\cos\theta) \times 20~(sin\theta) \times 40~(d\theta/dt) = 16,000$ states. When the state-space is high-dimensional (e.g., the image inputs for \textit{Atari} games), we propose a batch-based adjacency embedding policy. In particular, we embedded a batch (32) of adjacent image observations as one single state. For the consideration of time dependency and efficiency, we set a ``state queue'' which only records the noisy rewards for the latest 1,000 states. The confusion matrices are re-estimated based on current collections of observed noisy rewards every 100 steps.

\section{Estimation of Confusion Matrices}
\label{appendix:est_c}
\subsection{Reward Robust RL Algorithms} %subsection{Estimation Framework}
\label{sec:robust}
%In many circumstances, the noisy rates are often unknown to the agents, which raises trouble for the implementation of proposed algorithm. 
As stated in proposed reward robust RL framework,
%~\ref{sec:estimate_c}
the confusion matrix can be estimated dynamically based on the aggregated answers, similar to previous literature in supervised learning~\cite{DBLP:journals/corr/abs-1712-04577}. To get a concrete view, we take $Q$-Learning for an example, and the algorithm is called \textit{Reward Robust $Q$-Learning} (Algorithm~\ref{alg:robust_qlearn}). Note that is can be extended to other RL algorithms by plugging confusion matrix estimation steps and the computed surrogate rewards, as shown in the experiments (Figure~\ref{fig:cartpole_est}).

\begin{algorithm*}
\caption{Reward Robust $Q$-Learning}\label{alg:robust_qlearn}
\begin{algorithmic}[]
\STATE {\bfseries Input:} \\
$\mathcal{\tilde{M}} = (\mathcal{S}, \mathcal{A}, \mathcal{\tilde{R}}, \mathcal{P}, \gamma)$: MDP with corrupted reward channel \\
$T$: transition function $T: \mathcal{S} \times \mathcal{A} \rightarrow \mathcal{S}$ \\
$D_{\min} \in \mathbb{N}$: lower bound of collected noisy rewards (to collect enough noisy copies)\\
$\alpha \in (0, 1)$: learning rate in the update rule \\
$\eta \in (0,1)$: weight of unbiased surrogate reward \\
$\tilde{R}(s,a)$: set of observed rewards with a maximum size $D_{\max}$ when the state-action pair is $(s, a)$. 
\STATE {\bfseries Output:} $Q(s, a)$: value function; $\pi (s)$: policy function
\STATE Initialize $Q$: $\mathcal{S} \times \mathcal{A} \rightarrow \mathbb{R}$ arbitrarily
\STATE Set confusion matrix $\tilde{\rmC}$ as identity matrix $\rmI$
\WHILE{$Q$ is not converged}
\STATE Start in state $s \in \mathcal{S}$
    \WHILE{$s$ is not terminal}
\STATE  Calculate $\pi$ according to $Q$ and exploration strategy
\STATE  $a \leftarrow \pi(s)$; $s' \leftarrow T(s, a)$ 
\STATE Observe noisy reward $\tilde{r}(s, a)$ and add it to $\tilde{R}(s, a)$
\IF{$\sum_{(s,a)}|\tilde{R}(s, a)| \geq D_{\min} $}
\STATE Get predicted true reward $\bar{r}(s, a)$ using majority voting in every $\tilde{R}(s, a)$ (using Eqn.~\ref{eq:majority_voting})
\STATE Re-estimate confusion matrix $\tilde{\rmC}$ based on $\tilde{r}(s, a)$ and $\bar{r}(s, a)$ (using Eqn.~\ref{eq:estimate_c})
% \STATE Empty all the sets of observed rewards $\tilde{R}(s, a)$
\ENDIF
\STATE Obtain surrogate reward $\dot{r}(s, a)$ using $\rmR_{proxy} = (1 - \eta) \cdot {\rmR} + \eta 
\cdot \rmC^{-1} {\rmR}$
\STATE $Q(s, a) \leftarrow (1 - \alpha) \cdot Q(s, a) + \alpha \cdot \left(\hat{r}(s, a) + \gamma \cdot \max_{a'} Q(s', a') \right)$ 
\STATE  $s \leftarrow s'$
    \ENDWHILE
\ENDWHILE
\STATE {\bfseries return} $Q(s, a)$ and $\pi(s)$
\end{algorithmic}
\end{algorithm*}

\subsection{State-Dependent Perturbed Reward}
\label{sec:state-dependent}
In previous sections, to let our presentation stay focused, we consider the state-independent perturbed reward environments, which share the same confusion matrix for all states. In other words, the noise for different states is generated within the same distribution. More generally, the generation of $\tilde{r}$ follows a certain function $C : \mathcal{S} \times \mathcal{R} \to \tilde{R}$, where different states may correspond to varied noise distributions (also varied confusion matrices). However, our algorithm is still applicable. One intuitive solution is to maintain different confusion matrices $\rmC_s$ for different states.
%- which is not necessary but easier to start from. 
It is worthy to notice that Theorem~\ref{thm:convergence} holds because the surrogate rewards produce an unbiased estimation of true rewards for each state, \textit{i.e.}, 
\begin{align*}
\mathbb E_{\tilde{r}|r, s_t}[\hat{r}(s_t,a_t,s_{t+1})] = r(s_t,a_t,s_{t+1}).
\label{eq:state_dependent}
\end{align*}
Then we have, 
\begin{align*}
\mathbb E_{\tilde{r}|r}[\hat{r}(s_t,a_t,s_{t+1})] &= \sum_{s \in \mathcal{S}}\mathbb P_a(s_t, s_{t+1}) r(s_t,a_t,s_{t+1}) = r(s_t,a_t,s_{t+1})
\end{align*}

Furthermore, Theorem~\ref{thm:upper_bound} and~\ref{thm:variance} can be revised as:

\begin{theorem}
\label{thm:upper_bound__s_depend}
(Upper bound) Let $r \in [0, R_{\max}]$ be bounded reward, $\rmC_s$ be invertible reward confusion matrices with $\mathrm{det}(\rmC_s)$ denoting its determinant. For an appropriate choice of $m$, the Phased $Q$-Learning algorithm calls the generative model $G(\mathcal{\hat{M}})$
$$
O\left(\frac{|\mathcal{S}||\mathcal{A}|T}{\epsilon^2(1 - \gamma)^2\min_{s \in \mathcal{S}}\{\mathrm{det}(\rmC_s)\}^2}\log\frac{|\mathcal{S}||\mathcal{A}|T}{\delta}\right)
$$
times in $T$ epochs, and returns a policy such that for all state $s \in \mathcal{S}$, 
$
\left|V_{\pi}(s) - V^{\ast}(s) \right| \leq \epsilon, \epsilon > 0,$ w.p. $\geq 1 - \delta,~0<\delta<1$.
\end{theorem}

\begin{theorem}
\label{thm:variance_s_depend}
Let $r \in [0, R_{\max}]$ be bounded reward and all confusion matrices $\rmC_s$ are invertible. Then, the variance of surrogate reward $\hat{r}$ is bounded as follows:
$$
    \mathbf{Var}(r) \leq \mathbf{Var}(\hat{r}) \leq \frac{M^2}{\min_{s \in \mathcal{S}}\{\mathrm{det}(\rmC_s)\}^2} \cdot R_{\max}^2 .
$$
\end{theorem}

Let $\tilde{c}_{i,j|s}$ represents the entry of confusion matrix $\rmC_s$, indicating the flipping probability for generating a perturbed outcome for state $s$, \textit{i.e.}, $\tilde{c}_{i,j|s} = \mathbb{P}(\tilde{r}_t = R_k | r_t = R_j, s)$. Then the estimation step (see Eqn~(\ref{eq:estimate_c})) should be replaced by

\begin{align*}
\tilde{c}_{i,j|s} = \frac{\sum_{a \in \mathcal{A}} \#\left[\tilde{r}(s, a) = R_j | \bar{r}(s, a) = R_i\right]}{\sum_{a \in \mathcal{A}} \#[\bar{r}(s, a) = R_i]}.
\end{align*}

\subsection{Experimental Results}
To validate the effectiveness of \textit{robust reward} algorithms (like Algorithm~\ref{alg:robust_qlearn}), where the noise rates are unknown to the agents, we conduct extensive experiments in \textit{CartPole}. It is worthwhile to notice that the noisy rates are unknown in the explorations of RL agents. Besides, we discretize the observation (velocity, angle, etc.) to construct a set of states and implement like Algorithm~\ref{alg:robust_qlearn}. The $\eta$ is set $1.0$ in the experiments.

Figure~\ref{fig:cartpole_est} provides learning curves from five algorithms with different kinds of rewards. %The trends are very similar to Figure~\ref{fig:cartpole}, which means, 
The proposed estimation algorithms successfully obtain the approximate confusion matrices, and are robust in the unknown noise environments. From Figure~\ref{fig:cartpole_err}, we can observe that the estimation of confusion matrices converges very fast. The results are inspiring because we don't assume any additional knowledge about noise or true reward distribution in the implementation.

\begin{figure*}[htbp!]
\centering
\makebox[7pt]{\raisebox{40pt}{\rotatebox[origin=c]{90}{\scriptsize{$\omega=0.1$}}}}%
\begin{subfigure}[b]{.19\textwidth}
    \centering
    \includegraphics[width=\textwidth]{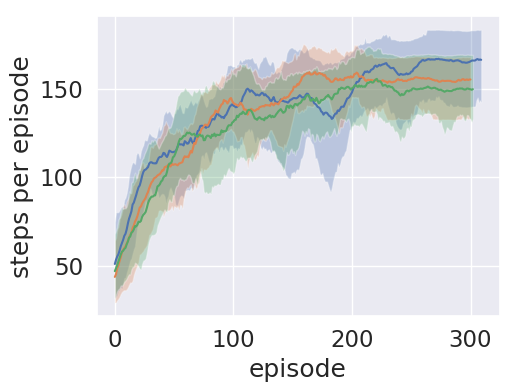}
\end{subfigure}
\begin{subfigure}[b]{.19\textwidth}
    \centering
    \includegraphics[width=\textwidth]{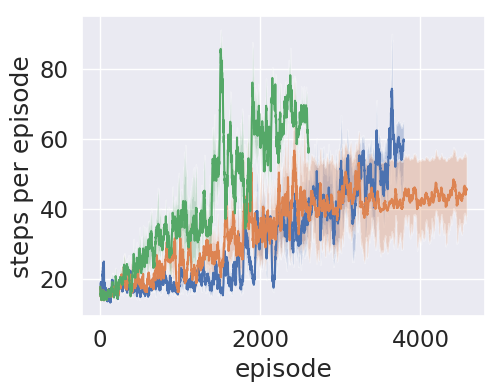}
\end{subfigure}
\begin{subfigure}[b]{.19\textwidth}
    \centering
    \includegraphics[width=\textwidth]{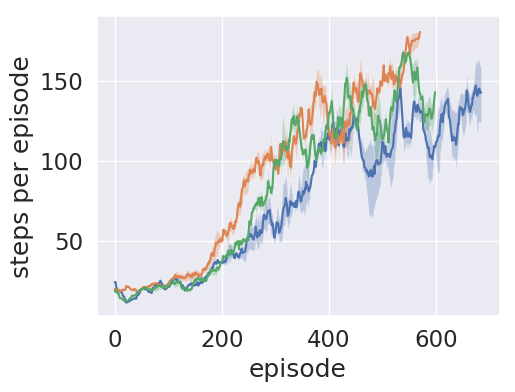}
\end{subfigure}
\begin{subfigure}[b]{.19\textwidth}
    \centering
    \includegraphics[width=\textwidth]{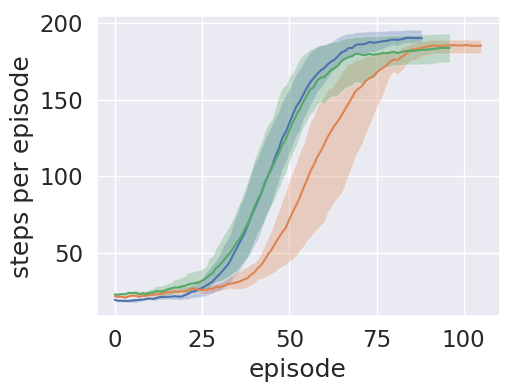}
\end{subfigure}
\begin{subfigure}[b]{.19\textwidth}
    \centering
    \includegraphics[width=\textwidth]{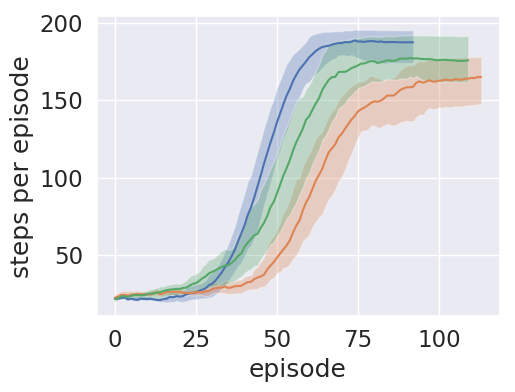}
\end{subfigure}

\makebox[7pt]{\raisebox{40pt}{\rotatebox[origin=c]{90}{\scriptsize{$\omega=0.3$}}}}%
\begin{subfigure}[b]{.19\textwidth}
    \centering
    \includegraphics[width=\textwidth]{cartpole_est2/qlearn-cartpole-0-3.png}
\end{subfigure}
\begin{subfigure}[b]{.19\textwidth}
    \centering
    \includegraphics[width=\textwidth]{cartpole_est2/cem-cartpole-0-3.png}
\end{subfigure}
\begin{subfigure}[b]{.19\textwidth}
    \centering
    \includegraphics[width=\textwidth]{cartpole_est2/sarsa-cartpole-0-3.png}
\end{subfigure}
\begin{subfigure}[b]{.19\textwidth}
    \centering
    \includegraphics[width=\textwidth]{cartpole_est2/dqn-cartpole-0-3.png}
\end{subfigure}
\begin{subfigure}[b]{.19\textwidth}
    \centering
    \includegraphics[width=\textwidth]{cartpole_est2/duel-dqn-cartpole-0-3.png}
\end{subfigure}

\makebox[7pt]{\raisebox{40pt}{\rotatebox[origin=c]{90}{\scriptsize{$\omega=0.7$}}}}%
\begin{subfigure}[b]{.19\textwidth}
    \centering
    \includegraphics[width=\textwidth]{cartpole_est2/qlearn-cartpole-0-7.png}
\end{subfigure}
\begin{subfigure}[b]{.19\textwidth}
    \centering
    \includegraphics[width=\textwidth]{cartpole_est2/cem-cartpole-0-7.png}
\end{subfigure}
\begin{subfigure}[b]{.19\textwidth}
    \centering
    \includegraphics[width=\textwidth]{cartpole_est2/sarsa-cartpole-0-7.png}
\end{subfigure}
\begin{subfigure}[b]{.19\textwidth}
    \centering
    \includegraphics[width=\textwidth]{cartpole_est2/dqn-cartpole-0-7.png}
\end{subfigure}
\begin{subfigure}[b]{.19\textwidth}
    \centering
    \includegraphics[width=\textwidth]{cartpole_est2/duel-dqn-cartpole-0-7.png}
\end{subfigure}

\makebox[7pt]{\raisebox{55pt}{\rotatebox[origin=c]{90}{\scriptsize{$\omega=0.9$}}}}%
\centering
\begin{subfigure}[b]{.19\textwidth}
    \centering
    \includegraphics[width=\textwidth]{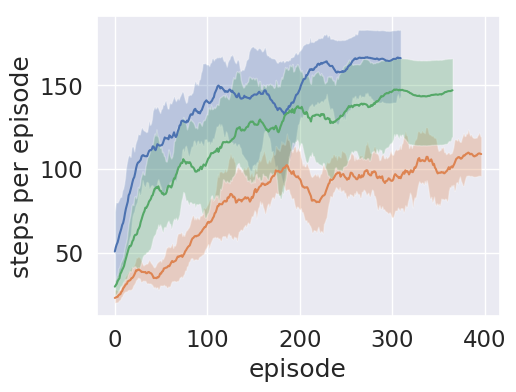}
    \caption{$Q$-Learning}
    \label{fig:cartpole_est_qlearn}
\end{subfigure}
\begin{subfigure}[b]{.19\textwidth}
    \centering
    \includegraphics[width=\textwidth]{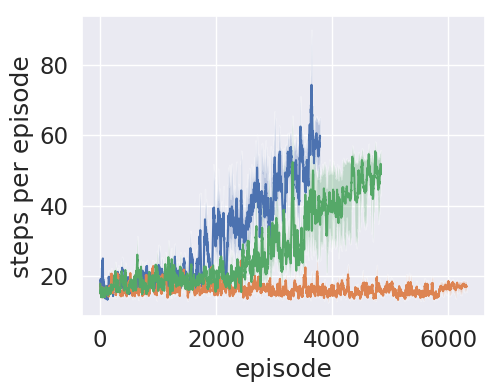}
    \caption{CEM}
    \label{fig:cartpole_est_cem}
\end{subfigure}
\begin{subfigure}[b]{.19\textwidth}
    \centering
    \includegraphics[width=\textwidth]{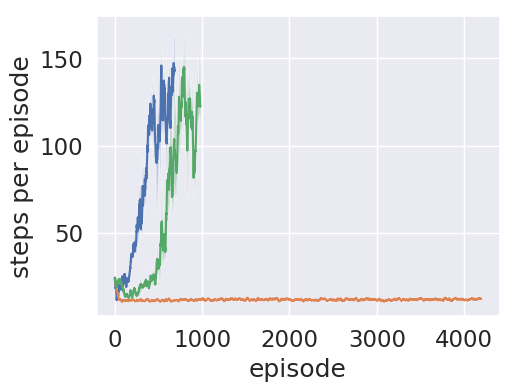}
    \caption{SARSA}
    \label{fig:cartpole_est_sarsa}
\end{subfigure}
\begin{subfigure}[b]{.19\textwidth}
    \centering
    \includegraphics[width=\textwidth]{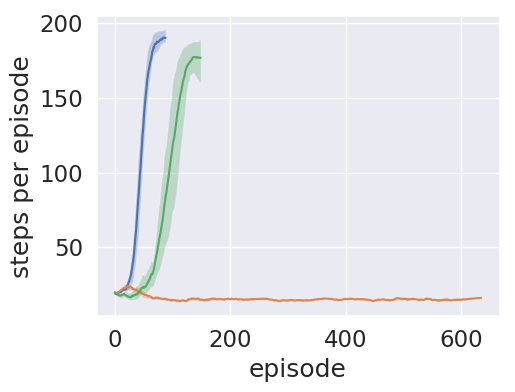}
    \caption{DQN}
    \label{fig:cartpole_est_dqn}
\end{subfigure}
\begin{subfigure}[b]{.19\textwidth}
    \centering
    \includegraphics[width=\textwidth]{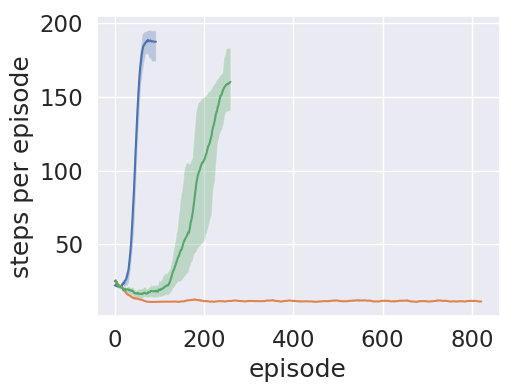}
    \caption{Dueling-DQN}
\end{subfigure}
\caption{Complete learning curves from five \textit{reward robust} RL algorithms (see Algorithm~\ref{alg:robust_qlearn}) on CartPole game with true rewards ($r$)~\crule[blue]{0.30cm}{0.30cm}, noisy rewards ($\tilde{r}$) ($\eta=1$)~\crule[orange]{0.30cm}{0.30cm} and estimated surrogate rewards ($\dot{r}$)~\crule[green]{0.30cm}{0.30cm}. Note that confusion matrices $\rmC$ are unknown to the agents here. From top to the bottom, the noise rates are 0.1, 0.3, 0.7 and 0.9. Here we repeated each experiment 10 times with different random seeds and plotted 10\% to 90\% percentile area with its mean highlighted.}
\label{fig:cartpole_est}
\end{figure*}

\begin{figure*}[htbp!]
\centering
\begin{subfigure}[b]{.19\textwidth}
    \centering
    \includegraphics[width=\textwidth]{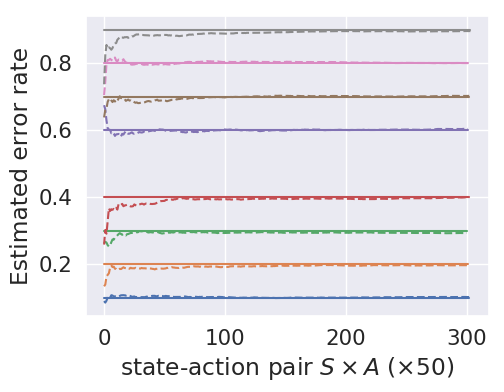}
\end{subfigure}
\begin{subfigure}[b]{.19\textwidth}
    \centering
    \includegraphics[width=\textwidth]{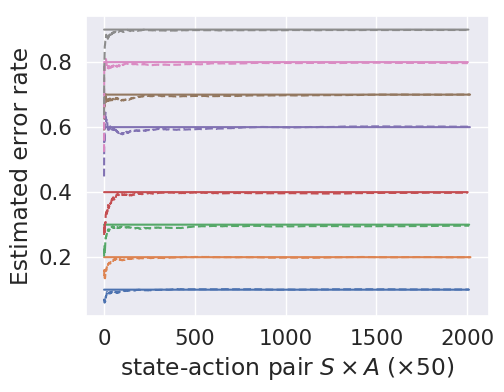}
\end{subfigure}
\begin{subfigure}[b]{.186\textwidth}
    \centering
    \includegraphics[width=\textwidth]{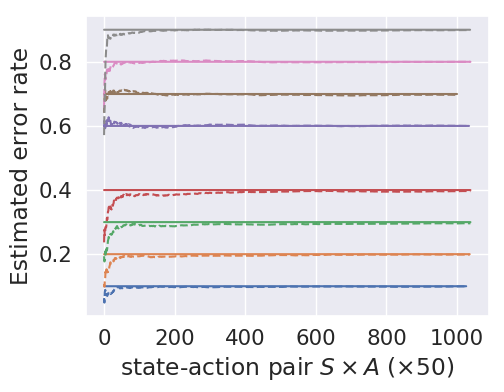}
\end{subfigure}
\begin{subfigure}[b]{.19\textwidth}
    \centering
    \includegraphics[width=\textwidth]{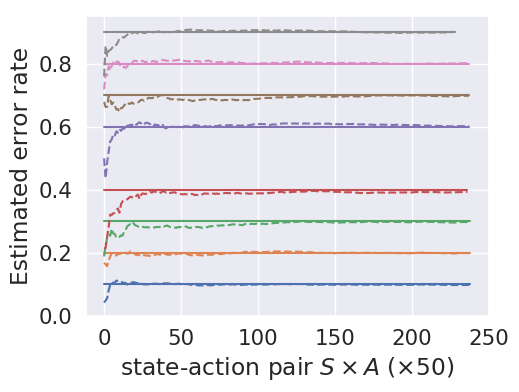}
\end{subfigure}
\begin{subfigure}[b]{.19\textwidth}
    \centering
    \includegraphics[width=\textwidth]{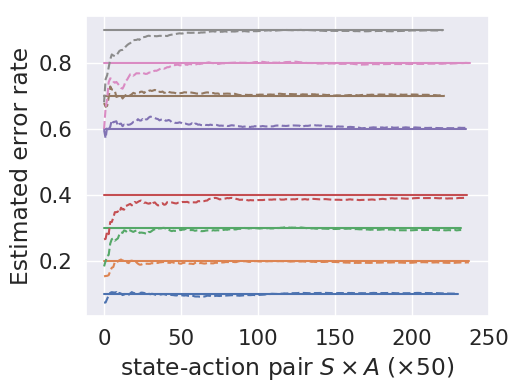}
\end{subfigure}

\centering
\begin{subfigure}[b]{.19\textwidth}
    \centering
    \includegraphics[width=\textwidth]{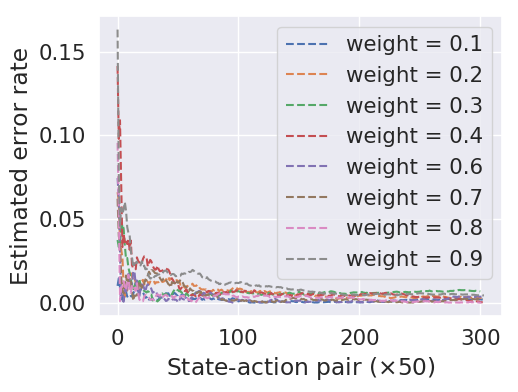}
    \caption{$Q$-Learning}
\end{subfigure}
\begin{subfigure}[b]{.19\textwidth}
    \centering
    \includegraphics[width=\textwidth]{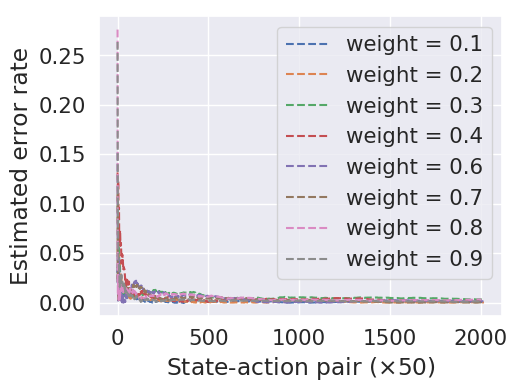}
    \caption{CEM}
\end{subfigure}
\begin{subfigure}[b]{.186\textwidth}
    \centering
    \includegraphics[width=\textwidth]{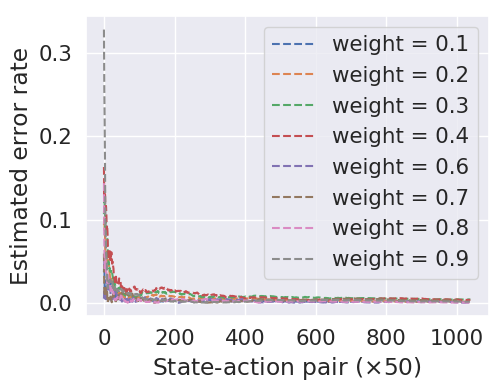}
    \caption{SARSA}
\end{subfigure}
\begin{subfigure}[b]{.19\textwidth}
    \centering
    \includegraphics[width=\textwidth]{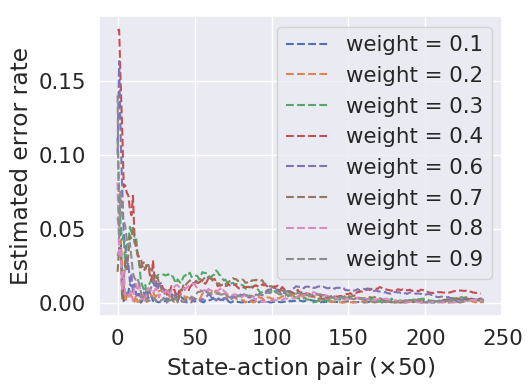}
    \caption{DQN}
\end{subfigure}
\begin{subfigure}[b]{.19\textwidth}
    \centering
    \includegraphics[width=\textwidth]{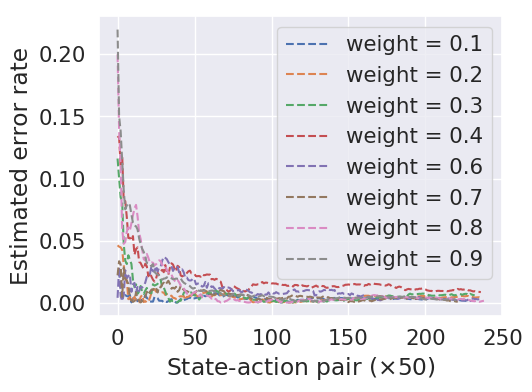}
    \caption{Dueling-DQN}
\end{subfigure}
\caption{Estimation analysis from five \textit{reward robust} RL algorithms (see Algorithm~\ref{alg:robust_qlearn}) on CartPole game. The upper figures are the convergence curves of estimated error rates (from 0.1 to 0.9), where the solid and dashed lines are ground truth and estimation, respectively; The lower figures are the absolute difference between the estimation and ground truth of confusion matrix $\rmC$ (normalized matrix norm).}
\label{fig:cartpole_err}
\end{figure*}

\begin{figure*}[htbp!]

\makebox[7pt]{\raisebox{40pt}{\rotatebox[origin=c]{90}{\scriptsize{$\omega=0.1$}}}}%
\begin{subfigure}[b]{.19\textwidth}
    \centering
    \includegraphics[width=\textwidth]{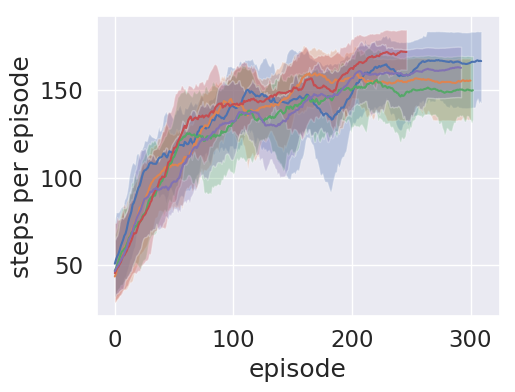}
\end{subfigure}
\begin{subfigure}[b]{.19\textwidth}
    \centering
    \includegraphics[width=\textwidth]{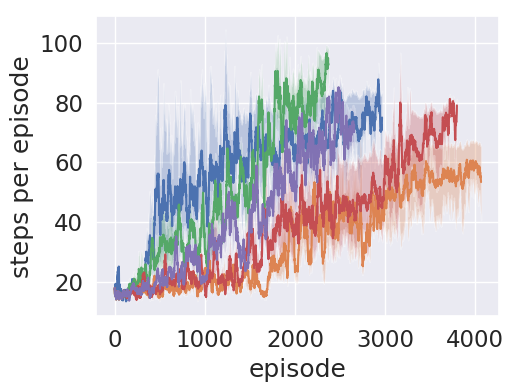}
\end{subfigure}
\begin{subfigure}[b]{.19\textwidth}
    \centering
    \includegraphics[width=\textwidth]{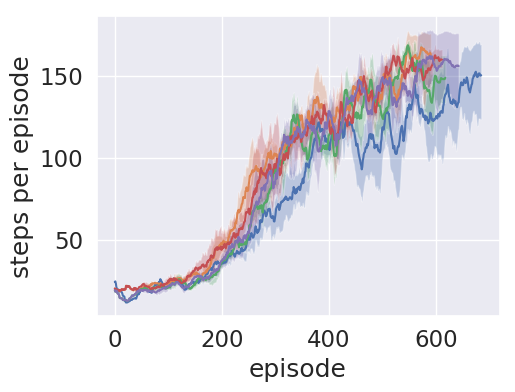}
\end{subfigure}
\begin{subfigure}[b]{.19\textwidth}
    \centering
    \includegraphics[width=\textwidth]{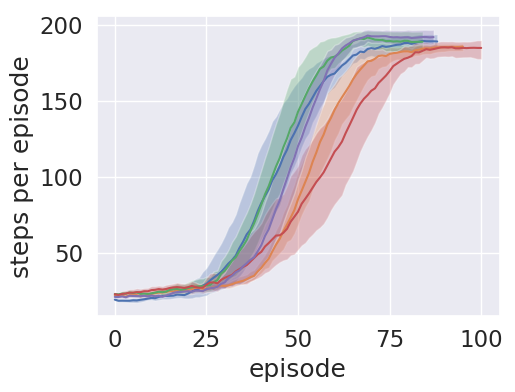}
\end{subfigure}
\begin{subfigure}[b]{.19\textwidth}
    \centering
    \includegraphics[width=\textwidth]{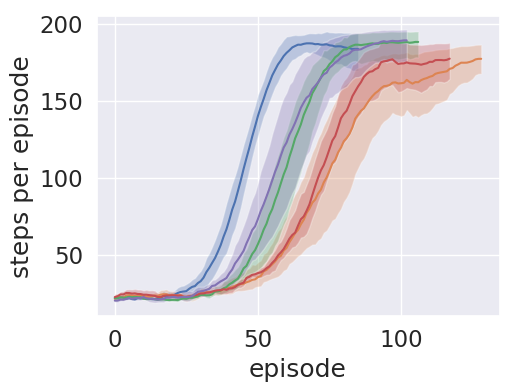}
\end{subfigure}

\makebox[7pt]{\raisebox{40pt}{\rotatebox[origin=c]{90}{\scriptsize{$\omega=0.3$}}}}%
\begin{subfigure}[b]{.19\textwidth}
    \centering
    \includegraphics[width=\textwidth]{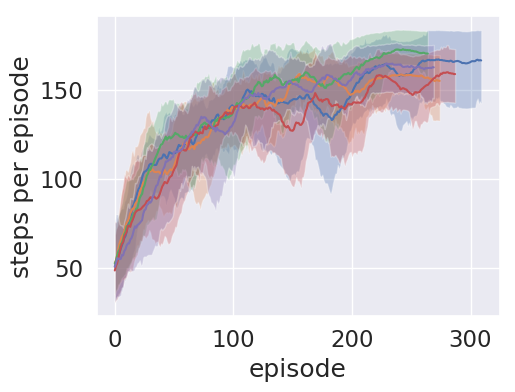}
\end{subfigure}
\begin{subfigure}[b]{.19\textwidth}
    \centering
    \includegraphics[width=\textwidth]{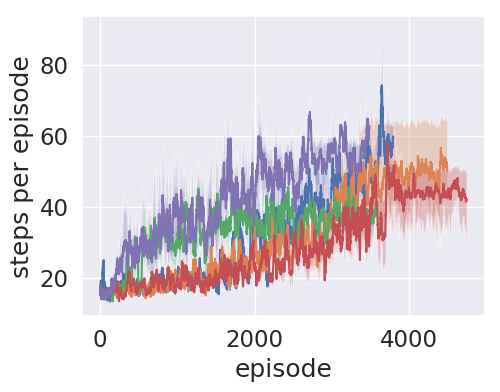}
\end{subfigure}
\begin{subfigure}[b]{.19\textwidth}
    \centering
    \includegraphics[width=\textwidth]{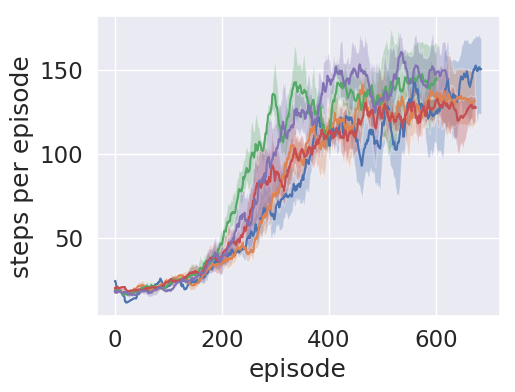}
\end{subfigure}
\begin{subfigure}[b]{.19\textwidth}
    \centering
    \includegraphics[width=\textwidth]{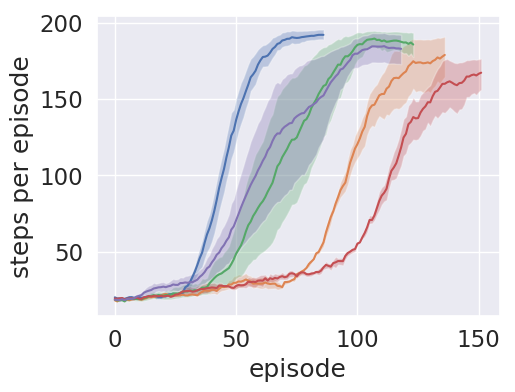}
\end{subfigure}
\begin{subfigure}[b]{.19\textwidth}
    \centering
    \includegraphics[width=\textwidth]{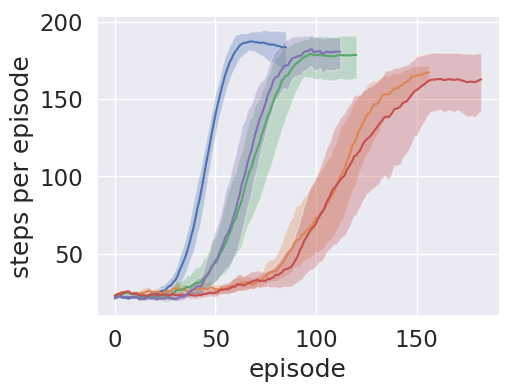}
\end{subfigure}

\makebox[7pt]{\raisebox{40pt}{\rotatebox[origin=c]{90}{\scriptsize{$\omega=0.7$}}}}%
\begin{subfigure}[b]{.19\textwidth}
    \centering
    \includegraphics[width=\textwidth]{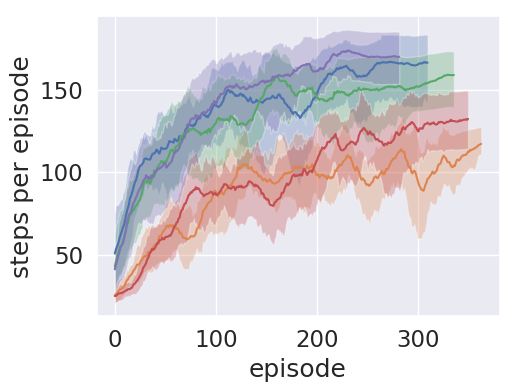}
\end{subfigure}
\begin{subfigure}[b]{.19\textwidth}
    \centering
    \includegraphics[width=\textwidth]{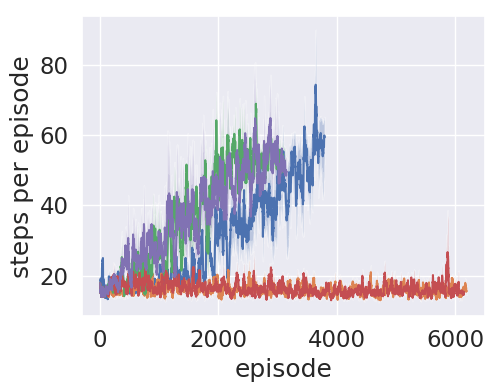}
\end{subfigure}
\begin{subfigure}[b]{.19\textwidth}
    \centering
    \includegraphics[width=\textwidth]{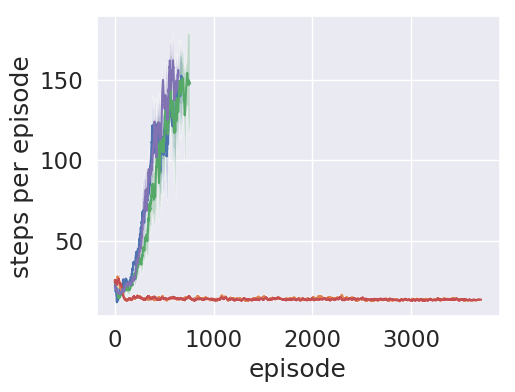}
\end{subfigure}
\begin{subfigure}[b]{.19\textwidth}
    \centering
    \includegraphics[width=\textwidth]{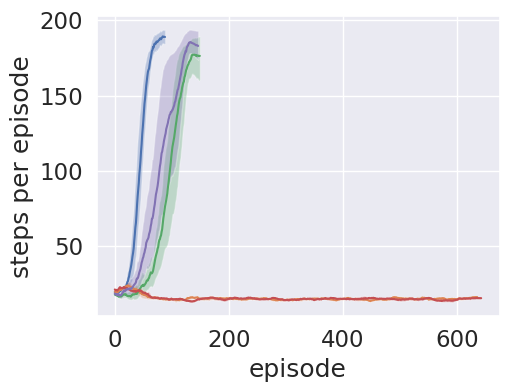}
\end{subfigure}
\begin{subfigure}[b]{.19\textwidth}
    \centering
    \includegraphics[width=\textwidth]{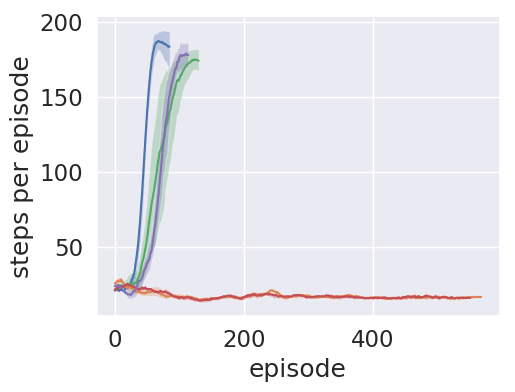}
\end{subfigure}

\makebox[7pt]{\raisebox{55pt}{\rotatebox[origin=c]{90}{\scriptsize{$\omega=0.9$}}}}%
\begin{subfigure}[b]{.19\textwidth}
    \centering
    \includegraphics[width=\textwidth]{cartpole_var/qlearn-cartpole-0-9.png}
    \caption{$Q$-Learning}
    \label{fig:cartpole_qlearn_var}
\end{subfigure}
\begin{subfigure}[b]{.19\textwidth}
    \centering
    \includegraphics[width=\textwidth]{cartpole_var/cem-cartpole-0-9.png}
    \caption{CEM}
    \label{fig:cartpole_cem_var}
\end{subfigure}
\begin{subfigure}[b]{.19\textwidth}
    \centering
    \includegraphics[width=\textwidth]{cartpole_var/sarsa-cartpole-0-9.png}
    \caption{SARSA}
    \label{fig:cartpole_sarsa_var}
\end{subfigure}
\begin{subfigure}[b]{.19\textwidth}
    \centering
    \includegraphics[width=\textwidth]{cartpole_var/dqn-cartpole-0-9.png}
    \caption{DQN}    
\end{subfigure}
\begin{subfigure}[b]{.19\textwidth}
    \centering
    \includegraphics[width=\textwidth]{cartpole_var/duel-dqn-cartpole-0-9.png}
    \caption{DDQN}
\end{subfigure}

\caption{Learning curves from five \textit{reward robust} RL algorithms (see Algorithm~\ref{alg:robust_qlearn}) on CartPole game with true rewards ($r$)~\crule[blue]{0.30cm}{0.30cm}, noisy rewards ($\tilde{r}$) ($\eta=1$)~\crule[orange]{0.30cm}{0.30cm}, sample-mean noisy rewards~($\eta=1$)~\crule[red]{0.30cm}{0.30cm}, estimated surrogate rewards ($\dot{r}$)~\crule[green]{0.30cm}{0.30cm} and sample-mean estimated surrogate rewards~\crule[blue!40!white]{0.30cm}{0.30cm}. Note that confusion matrices $\rmC$ are unknown to the agents here. From top to the bottom, the noise rates are 0.1, 0.3, 0.7 and 0.9. Here we repeated each experiment 10 times with different random seeds and plotted 10\% to 90\% percentile area with its mean highlighted. 
% Full results are in Appendix~\ref{appendix:control_figures} (Figure~\ref{fig:cartpole_est}).
}
\label{fig:cartpole_var}
\end{figure*}

\begin{table*}[htbp!]
\centering
\caption{Average scores of various RL algorithms on CartPole with sample-mean reward using variance reduction technique (VRT), surrogate rewards (ours) and the combination of them (ours + VRT). Note that the reward confusion matrices are unknown to the agents and each experiment is repeated three times with different random seeds.}
%The experiments are repeated three times with different random seeds.
%$Q$-Learn and DDQN in the table denote $Q$-Learning and Dueling DQN algorithms, respectively.
\label{tab:scores-control-var}
\vskip 0.1in
%\resizebox{0.75\columnwidth}{!}{	
\begin{tabular}{@{}c|c|ccccc@{}}
\toprule
\ Noise Rate\   & Reward & $Q$-Learn & CEM & SARSA & DQN & DDQN $\ $ \\ \midrule
\multicolumn{1}{c|}{\multirow{3}{*}{$\omega = 0.1$}} & VRT & 173.5 & \textbf{99.7}  & 167.3  & 181.9 & \textbf{187.4} \\ %\cmidrule(l){2-8} 
\multicolumn{1}{c|}{}                     &  ours ($\dot{r}$)  & 181.9 & 99.3  & 171.5 & \textbf{200.0} &  185.6             \\ 
\multicolumn{1}{c|}{}                     &  \cellcolor{Gray} ours + VRT   & \cellcolor{Gray}\textbf{184.5} & \cellcolor{Gray}98.2  & \cellcolor{Gray}\textbf{174.2} & \cellcolor{Gray}199.3 &  \cellcolor{Gray}186.5  \\ \midrule
\multicolumn{1}{c|}{\multirow{3}{*}{$\omega = 0.3$}} & VRT & 140.4 & 43.9 & 149.8 & 182.7 & 177.6   \\ %\cmidrule(l){2-8} 
\multicolumn{1}{c|}{}                     & ours ($\dot{r}$) & 161.1 & 81.8 &  159.6  & 186.7 & \textbf{200.0}   \\ 
\multicolumn{1}{c|}{}                     &  \cellcolor{Gray}ours + VRT   & \cellcolor{Gray}\textbf{161.6} & \cellcolor{Gray}\textbf{82.2}  & \cellcolor{Gray}\textbf{159.8} & \cellcolor{Gray}\textbf{188.4} & \cellcolor{Gray}198.2    \\ \midrule
\multicolumn{1}{c|}{\multirow{3}{*}{$\omega = 0.7$}} & VRT & 71.1 & 16.1 & 13.2 & 15.6 & 14.7 \\ %\cmidrule(l){2-8}
\multicolumn{1}{c|}{}                     & ours ($\dot{r}$) & 172.1 & \textbf{83.0} & 174.4  & 189.3 & 191.3  \\ 
\multicolumn{1}{c|}{}                     &  \cellcolor{Gray}ours + VRT  & \cellcolor{Gray}\textbf{182.3} & \cellcolor{Gray}79.5  & \cellcolor{Gray}\textbf{178.9} & \cellcolor{Gray}\textbf{195.9} & \cellcolor{Gray}\textbf{194.2}    \\
 \bottomrule
\end{tabular}
%}
\end{table*}

\section{Supplementary Experimental Results}
\label{appendix:exp_results}

\subsection{Visualizations on Control Games}
\label{appendix:control_figures}

\begin{figure*}[htbp!]
\centering
\makebox[7pt]{\raisebox{46pt}{\rotatebox[origin=c]{90}{\scriptsize{$\omega=0.1$}}}}%
\begin{subfigure}[b]{.24\textwidth}
    \centering
    \includegraphics[width=\textwidth]{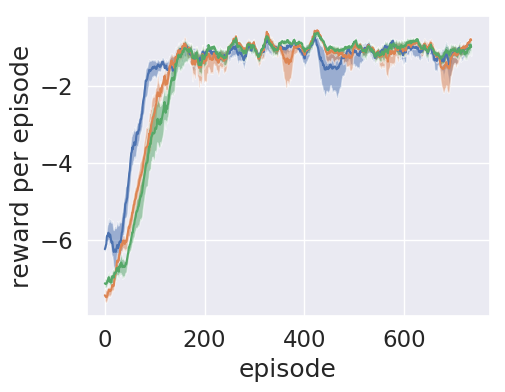}
\end{subfigure}
\begin{subfigure}[b]{.24\textwidth}
    \centering
    \includegraphics[width=\textwidth]{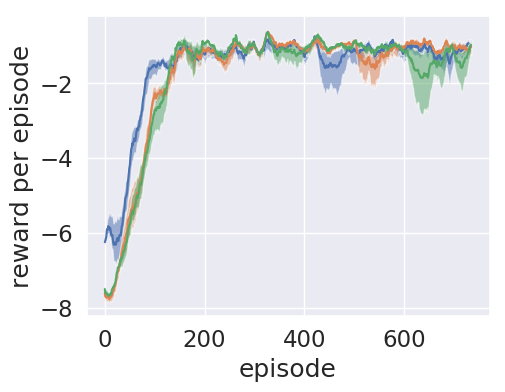}
\end{subfigure}
\begin{subfigure}[b]{.24\textwidth}
    \centering
    \includegraphics[width=\textwidth]{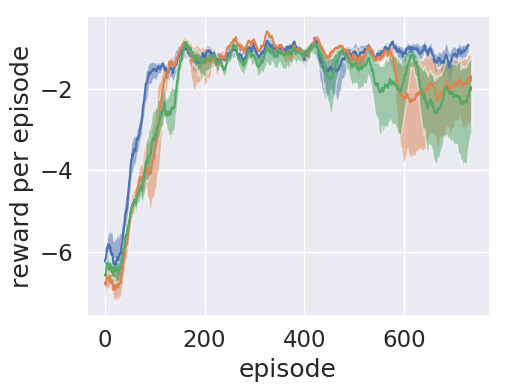}
\end{subfigure}
\begin{subfigure}[b]{.24\textwidth}
    \centering
    \includegraphics[width=\textwidth]{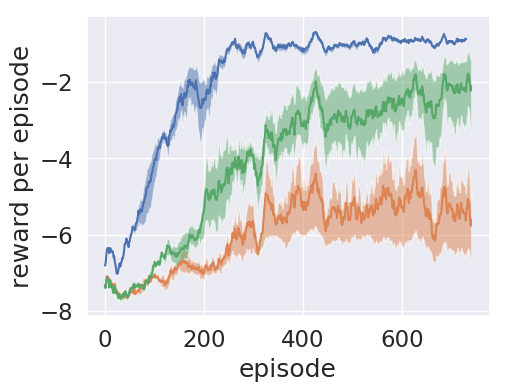}
\end{subfigure}
\makebox[7pt]{\raisebox{46pt}{\rotatebox[origin=c]{90}{\scriptsize{$\omega=0.3$}}}}%
\begin{subfigure}[b]{.24\textwidth}
    \centering
    \includegraphics[width=\textwidth]{pendulum3/ddpg-pendulum-0-3-anti_iden.png}
\end{subfigure}
\begin{subfigure}[b]{.24\textwidth}
    \centering
    \includegraphics[width=\textwidth]{pendulum3/ddpg-pendulum-0-3-norm_one.png}
\end{subfigure}
\begin{subfigure}[b]{.24\textwidth}
    \centering
    \includegraphics[width=\textwidth]{pendulum3/ddpg-pendulum-0-3-norm_all.png}
\end{subfigure}
\begin{subfigure}[b]{.24\textwidth}
    \centering
    \includegraphics[width=\textwidth]{pendulum3/naf-pendulum-0-3-norm_all.png}
\end{subfigure}
\makebox[7pt]{\raisebox{46pt}{\rotatebox[origin=c]{90}{\scriptsize{$\omega=0.7$}}}}%
\begin{subfigure}[b]{.24\textwidth}
    \centering
    \includegraphics[width=\textwidth]{pendulum3/ddpg-pendulum-0-7-anti_iden.png}
\end{subfigure}
\begin{subfigure}[b]{.24\textwidth}
    \centering
    \includegraphics[width=\textwidth]{pendulum3/ddpg-pendulum-0-7-norm_one.png}
\end{subfigure}
\begin{subfigure}[b]{.24\textwidth}
    \centering
    \includegraphics[width=\textwidth]{pendulum3/ddpg-pendulum-0-7-norm_all.png}
\end{subfigure}
\begin{subfigure}[b]{.24\textwidth}
    \centering
    \includegraphics[width=\textwidth]{pendulum3/naf-pendulum-0-7-norm_all.png}
\end{subfigure}
\makebox[7pt]{\raisebox{62pt}{\rotatebox[origin=c]{90}{\scriptsize{$\omega=0.9$}}}}%
\begin{subfigure}[b]{.24\textwidth}
    \centering
    \includegraphics[width=\textwidth]{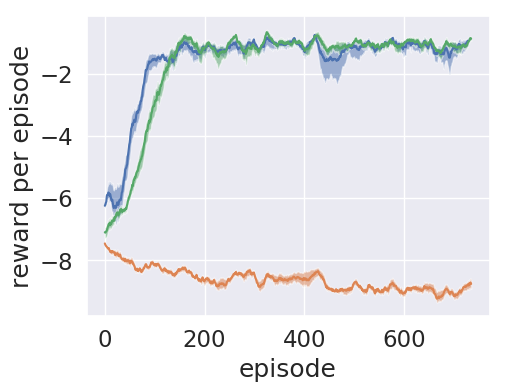}
    \caption{DDPG (symmetric)}
    \label{fig:pendulum_a_full}
\end{subfigure}
\begin{subfigure}[b]{.24\textwidth}
    \centering
    \includegraphics[width=\textwidth]{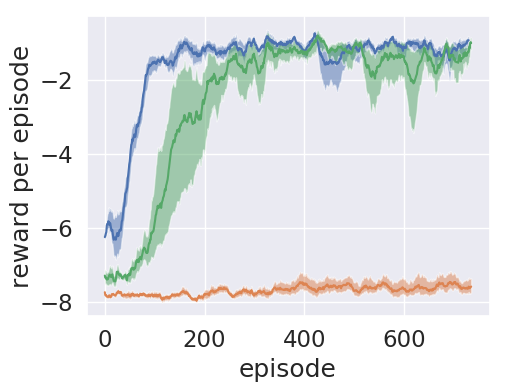}
    \caption{DDPG (rand-one)}
    \label{fig:pendulum_b_full}
\end{subfigure}
\begin{subfigure}[b]{.24\textwidth}
    \centering
    \includegraphics[width=\textwidth]{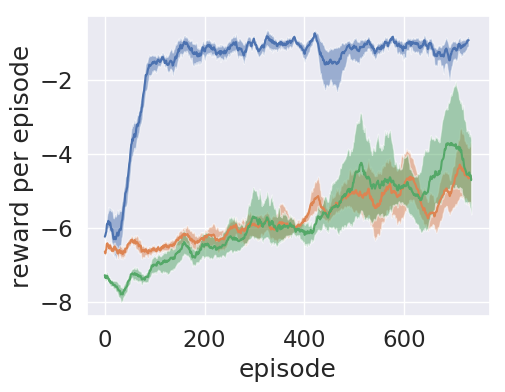}
    \caption{DDPG (rand-all)}
    \label{fig:pendulum_c_full}
\end{subfigure}
\begin{subfigure}[b]{.24\textwidth}
    \centering
    \includegraphics[width=\textwidth]{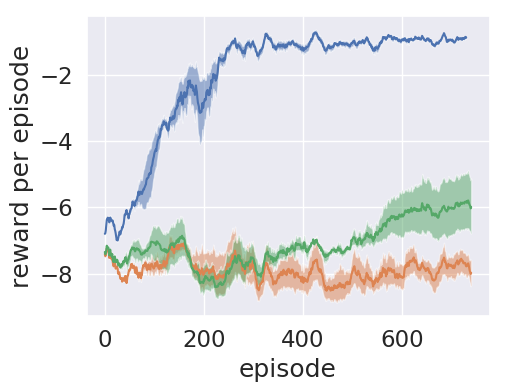}
    \caption{NAF (rand-all)}
    \label{fig:pendulum_d_full}
\end{subfigure}
\caption{Complete learning curves from DDPG and NAF on Pendulum game with true rewards ($r$) ~\crule[blue]{0.30cm}{0.30cm}, noisy rewards ($\tilde{r}$)~\crule[orange]{0.30cm}{0.30cm} and surrogate rewards ($\hat{r}$) ($\eta=1$)~\crule[green]{0.30cm}{0.30cm}. Both symmetric and asymmetric noise are conduced in the experiments. From top to the bottom, the noise rates are 0.1, 0.3, 0.7 and 0.9, respectively. Here we repeated each experiment 6 times with different random seeds and plotted 10\% to 90\% percentile area with its mean highlighted.}
\label{fig:pendulum_full}
\end{figure*}

\subsection{Visualizations on Atari Games\protect\footnote{For the clarity purpose, we remove the learning curves (blue ones in previous figures) with true rewards except for Pong-v4 game.}}
\label{appendix:atari_results}
\begin{figure*}[htbp!]
\centering
\begin{subfigure}[b]{0.98\textwidth}
    \includegraphics[width=\textwidth]{imgs/Pong_anti_iden_.png}
    \caption{Pong (\textit{sysmetric})}    
\end{subfigure}
\begin{subfigure}[b]{0.98\textwidth}
    \includegraphics[width=\textwidth]{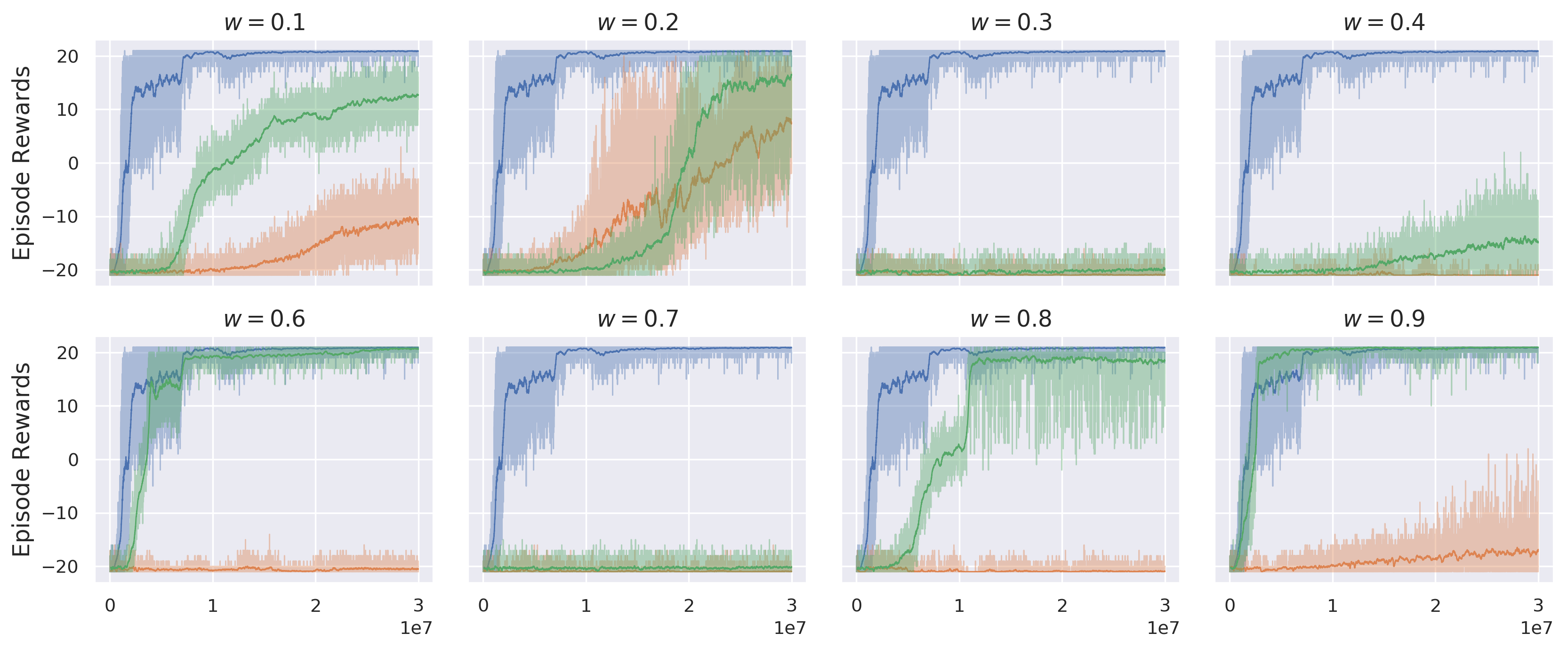}
    \caption{Pong (\textit{rand-one})}    
\end{subfigure}
\begin{subfigure}[b]{0.98\textwidth}
    \includegraphics[width=\textwidth]{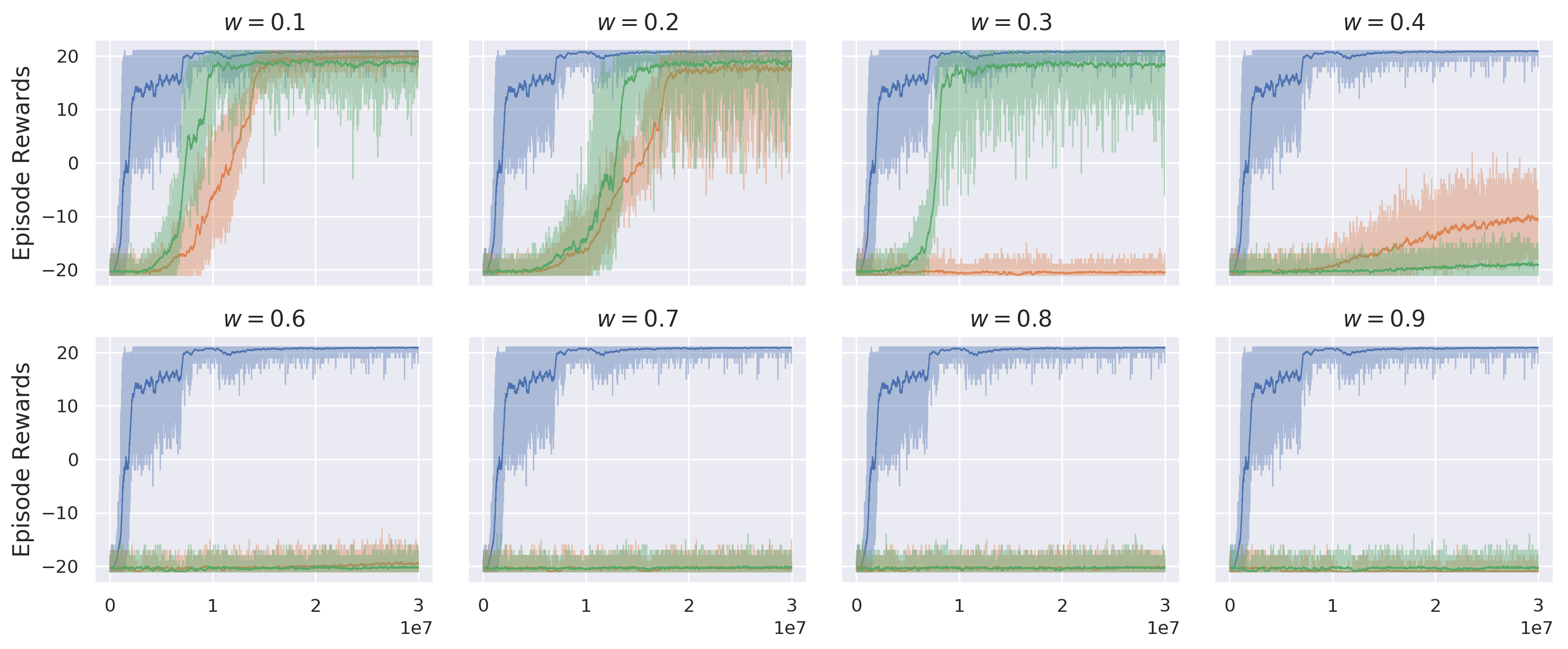}
    \caption{Pong (\textit{rand-all})}  
\end{subfigure}
\end{figure*}

\begin{figure*}[htbp!]\ContinuedFloat
\centering
\begin{subfigure}[b]{0.98\textwidth}
    \includegraphics[width=\textwidth]{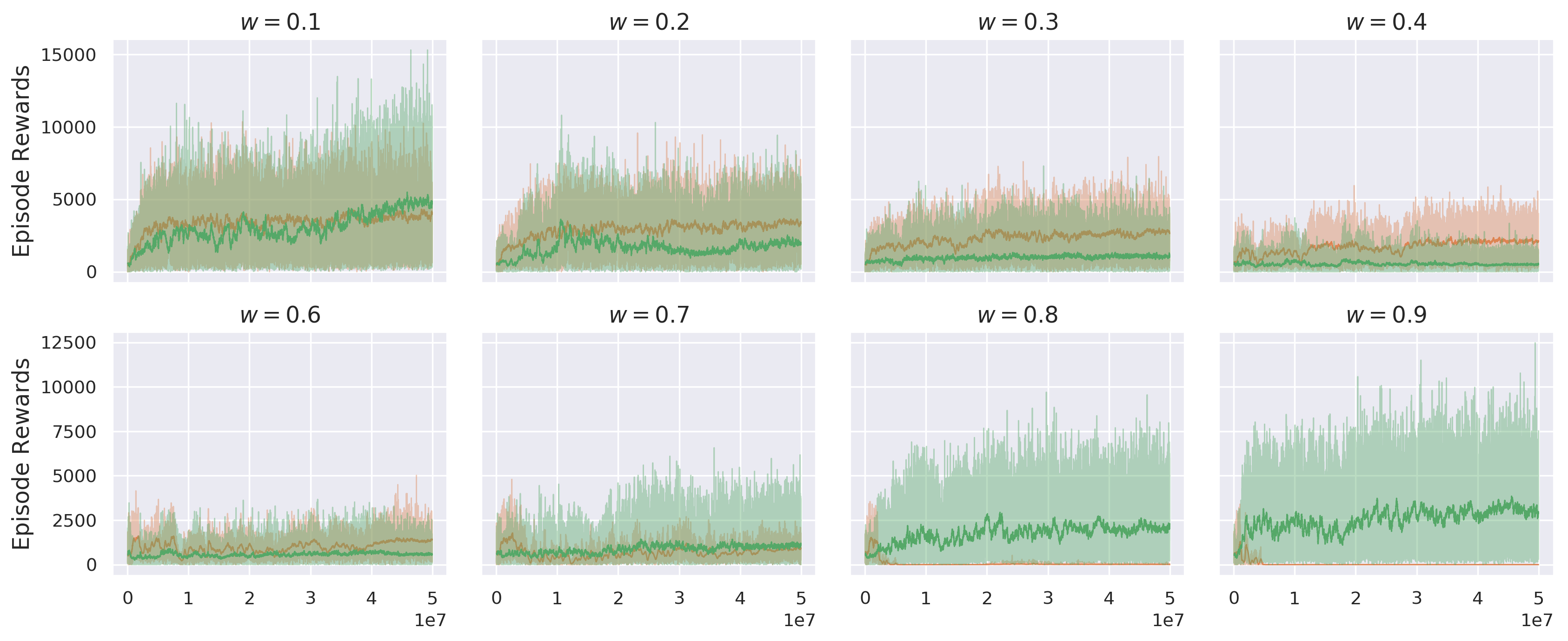}
    \caption{AirRaid (sysmetric)}  
\end{subfigure}
\begin{subfigure}[b]{0.98\textwidth}
    \includegraphics[width=\textwidth]{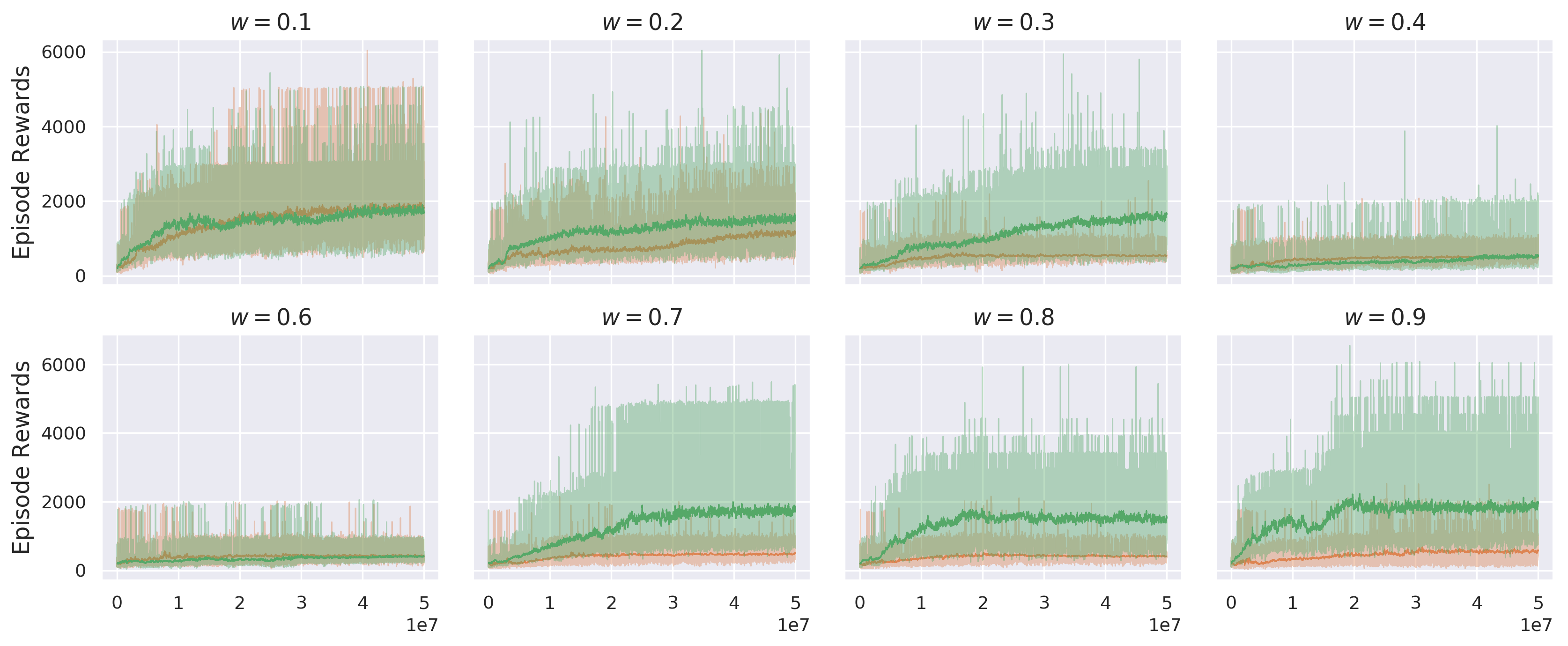}
    \caption{Alien (sysmetric)}  
\end{subfigure}
\begin{subfigure}[b]{0.98\textwidth}
    \includegraphics[width=\textwidth]{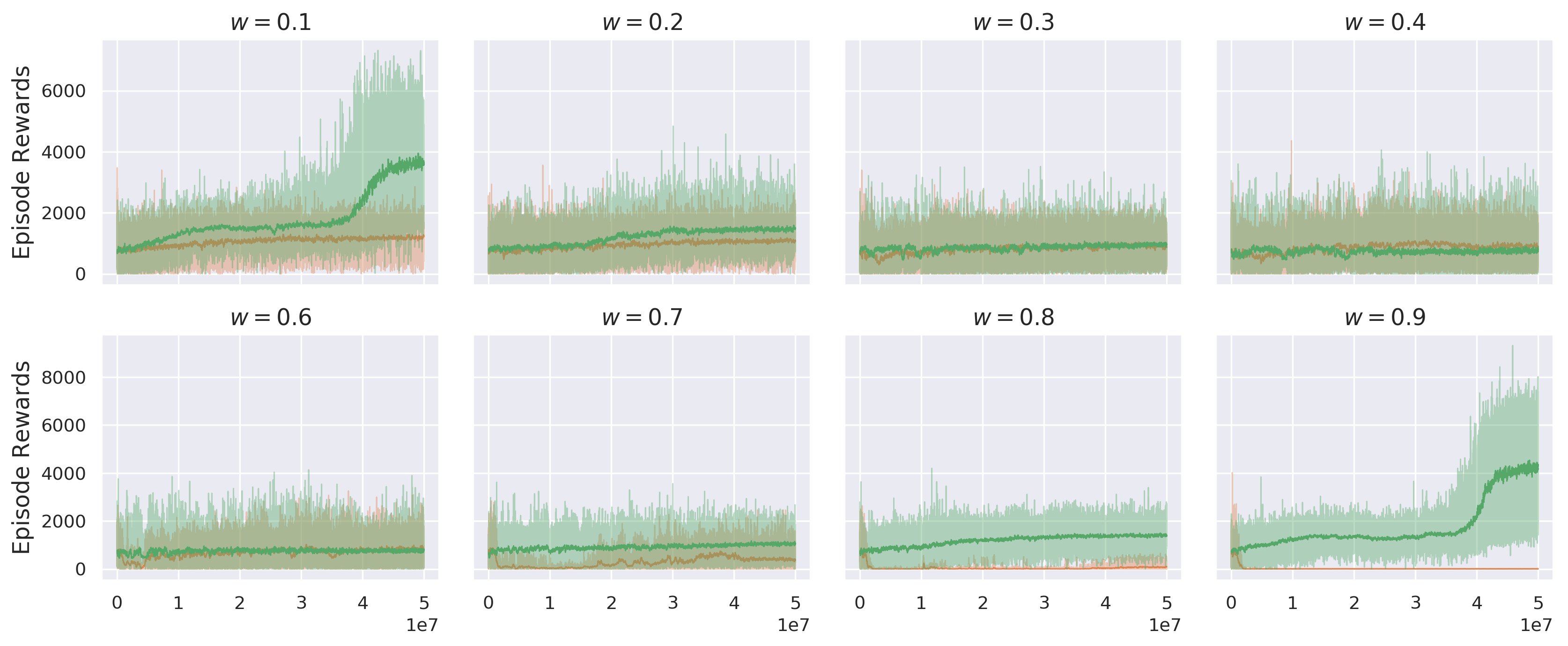}
    \caption{Carnival (sysmetric)}  
\end{subfigure}
\end{figure*}

\begin{figure*}[htb]\ContinuedFloat
\centering
\begin{subfigure}[b]{0.98\textwidth}
    \includegraphics[width=\textwidth]{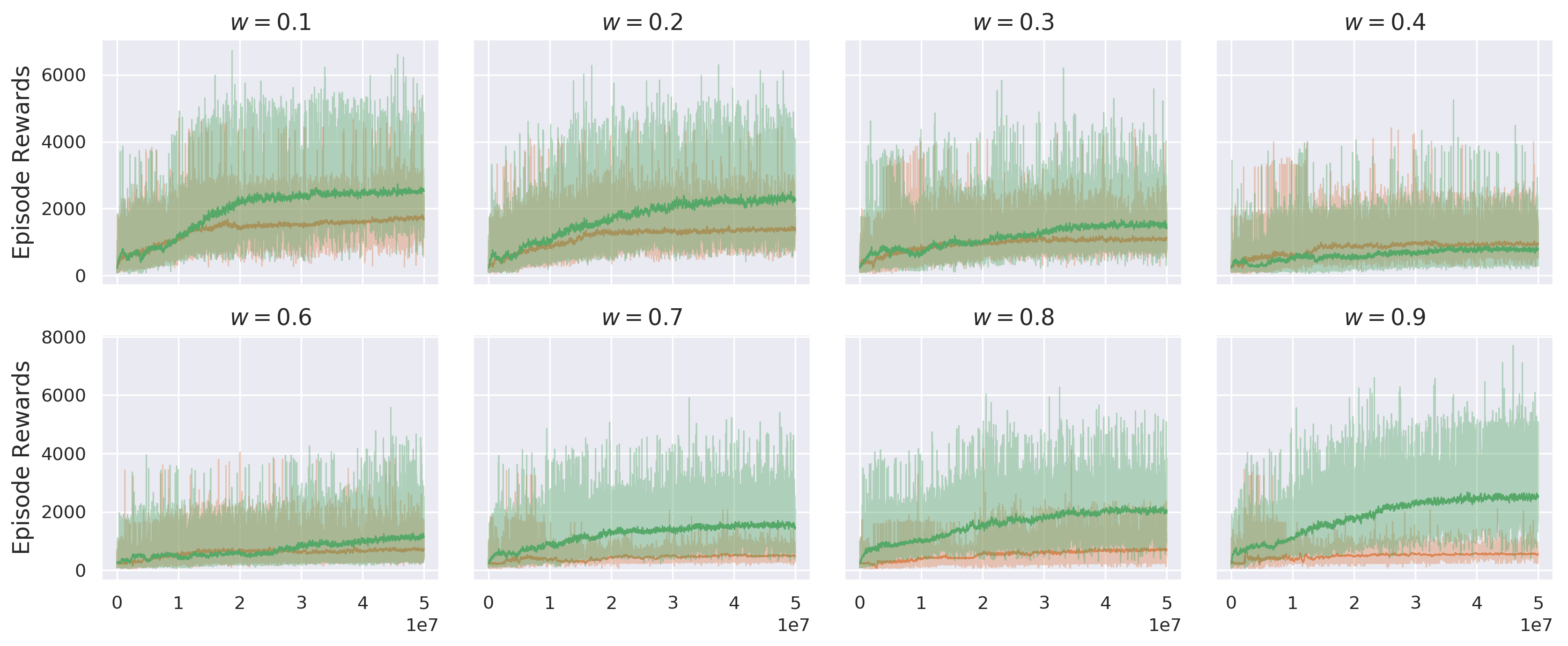}
    \caption{MsPacman (sysmetric)}  
\end{subfigure}
\begin{subfigure}[b]{0.98\textwidth}
    \includegraphics[width=\textwidth]{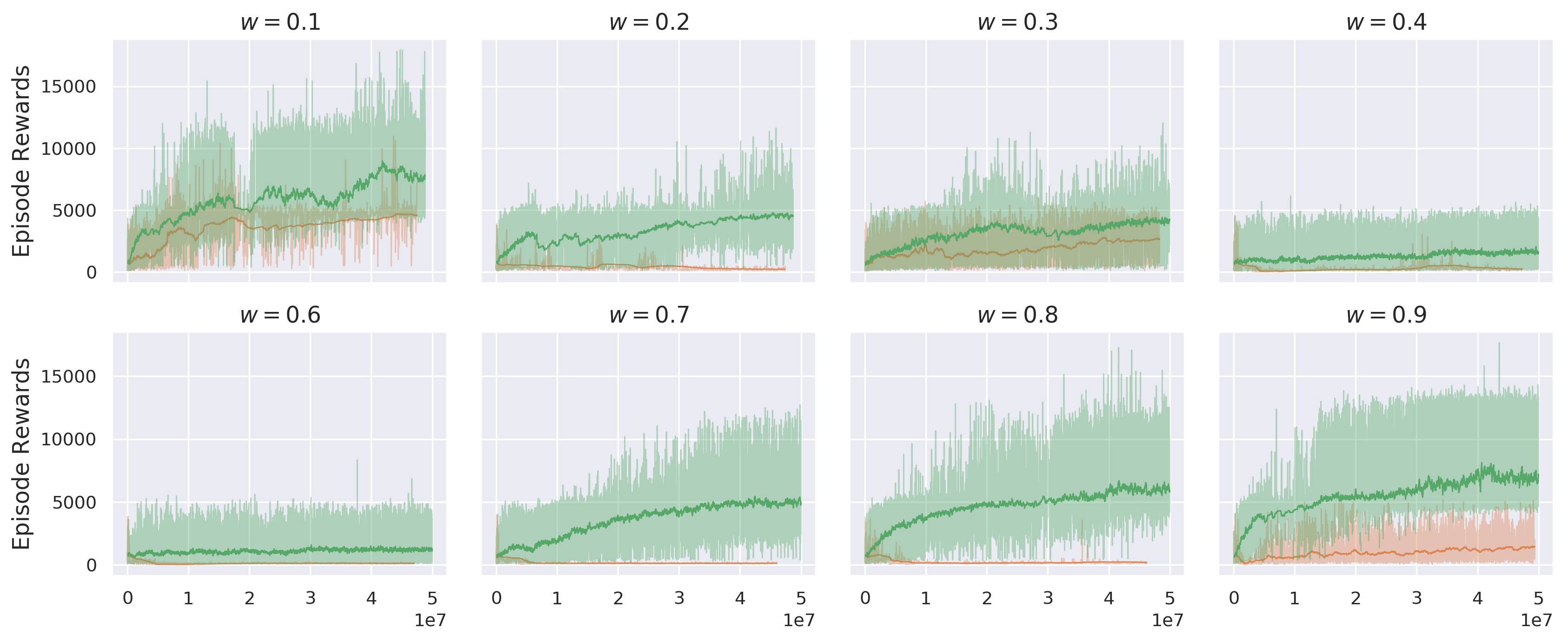}
    \caption{Phoenix (sysmetric)}  
\end{subfigure}
\begin{subfigure}[b]{0.98\textwidth}
    \includegraphics[width=\textwidth]{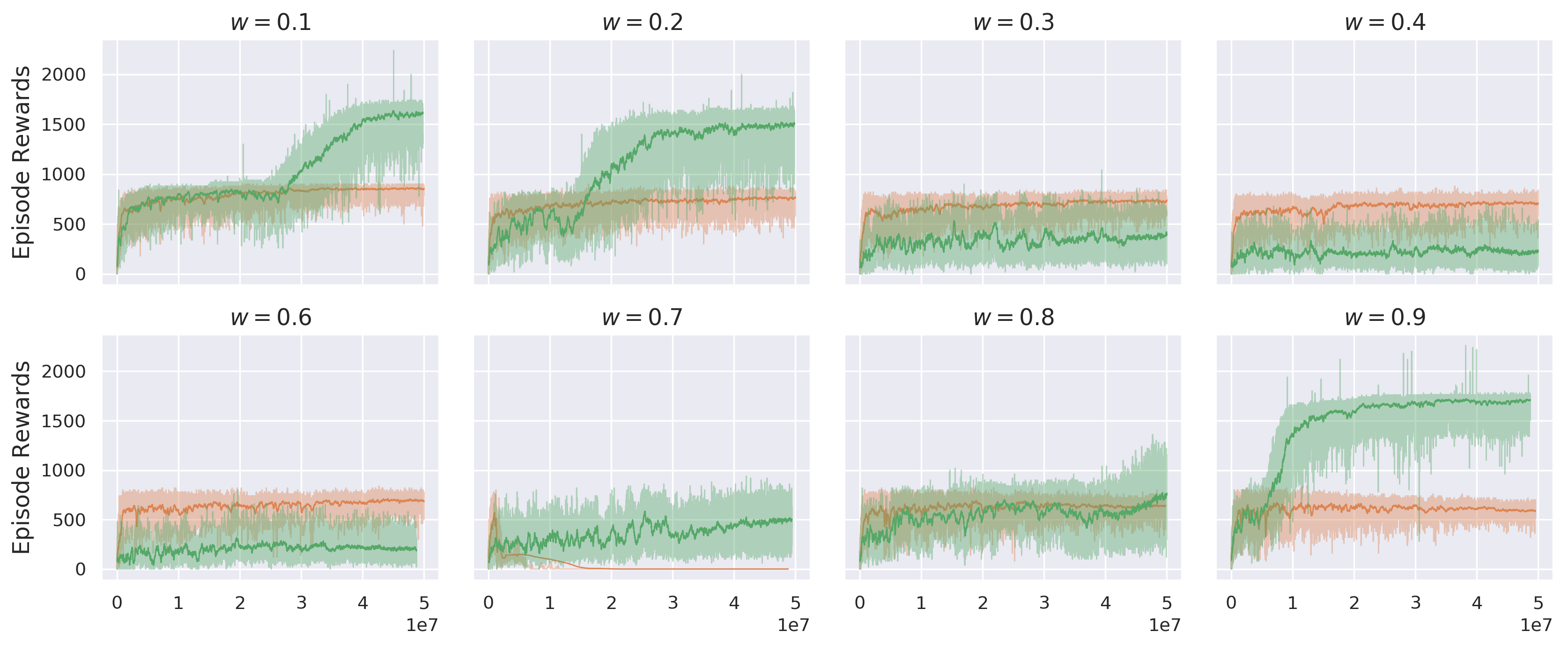}
    \caption{Seaquest (sysmetric)}  
\end{subfigure}
\caption{Complete learning curves from PPO on seven Atari game with true rewards ($r$)~\crule[blue]{0.30cm}{0.30cm}, noisy rewards ($\tilde{r}$)~\crule[orange]{0.30cm}{0.30cm} and surrogate rewards~($\eta = 1$) ($\hat{r}$)~\crule[green]{0.30cm}{0.30cm}. The noise rates increase from 0.1 to 0.9, with a step of 0.1.}
\label{fig:atari_full}
\end{figure*}

\end{document}